\documentclass{article}

\usepackage[mathlines]{lineno}
\usepackage[final]{neurips_2022}

\usepackage[utf8]{inputenc} 
\usepackage[T1]{fontenc}   
\usepackage{hyperref}   
\usepackage{url}          
\usepackage{booktabs}      
\usepackage{amsfonts}       
\usepackage{nicefrac}      
\usepackage{microtype}

\usepackage{xcolor}        
\usepackage{xspace}
\usepackage{amsmath,amsthm,bbm}
\newcommand{\new}[1]{\emph{#1}}

\newcommand{\cO}{\ensuremath{{\mathcal O}}\xspace}

\newcommand{\bbR}{\ensuremath{\mathbb{R}}}

\newcommand{\bbN}{\ensuremath{\mathbb{N}}}

\newcommand{\RR}{\mathbb{R}}

\newcommand{\NN}{\mathbb{N}}

\renewcommand{\vec}[1]{\mathbf{#1}}

\newcommand{\oms}{\{\!\!\{}
\newcommand{\cms}{\}\!\!\}}

\newcommand{\hb}{\mathbf{h}}
\newcommand{\pb}{\mathbf{g}}
\newcommand{\qb}{\mathbf{q}}

\newcommand{\st}{\mathsf{t}}

\newcommand{\Nbb}{\mathbb{N}}

\newcommand{\tup}[1]{{(#1)}}

\newcommand{\UPD}{\mathsf{UPD}}
\newcommand{\AGG}{\mathsf{AGG}}
\newcommand{\RO}{\mathsf{READOUT}}

\newcommand{\REL}{\mathsf{RELABEL}}

\newcommand{\pAGG}{\mathsf{SAGG}}

\newcommand{\WLk}[1]{#1\text{-}\mathsf{WL}}

\newcommand{\PWLk}[1]{#1\text{-}\mathsf{OSWL}}
\newcommand{\GNN}{$\mathsf{GNN}$\xspace}
\newcommand{\GNNs}{$\mathsf{GNN}\text{s}$\xspace}
\newcommand{\MPNN}{$\mathsf{MPNN}$\xspace}
\newcommand{\MPNNs}{$\mathsf{MPNN}\text{s}$\xspace}
\newcommand{\PMPNN}[1]{$#1$-$\mathsf{OSAN}$\xspace}
\newcommand{\PMPNNs}[1]{$#1$-$\mathsf{OSAN}\text{s}$\xspace}

\newcommand{\mgnns}[1]{$#1$-$\mathsf{mGNN}\text{s}$\xspace}
\newcommand{\recon}[1]{$#1$-$\mathsf{recGNN}\text{s}$\xspace}

\newcommand{\idgnn}[1]{$\mathsf{idGNN}\text{s}$\xspace}
\newcommand{\idgnnk}[1]{$#1$-$\mathsf{idGNN}\text{s}$\xspace}

\newcommand{\kernel}{$\mathsf{kernelGNN}\text{s}$\xspace}
\newcommand{\dsgnn}{$\mathsf{DS\text{-}GNN}\text{s}$\xspace}
\newcommand{\dssgnn}{$\mathsf{DSS\text{-}GNN}\text{s}$\xspace}

\newcommand{\nested}{$\mathsf{nestedGNN}\text{s}$\xspace}

\newcommand{\ones}{\mathbbm{1}}

\newtheorem{claim}{Claim}

\hypersetup{
    colorlinks,
    linkcolor={black},
    citecolor={black},
    urlcolor={black}
}

\newcommand{\cm}[1]{{{\textcolor{purple}{\textbf{[CM:} {#1}\textbf{]}}}}}
\newcommand{\fg}[1]{{{\textcolor{blue}{\textbf{[FG:} {#1}\textbf{]}}}}}

\newcommand{\florcol}{blue}
\newcommand{\floris}[1]{\textcolor{\florcol}{/* #1 */}}

\newcommand{\RSet}{\mathbb{R}}

\newcommand{\bz}{\bm{z}}

\newcommand{\by}{\bm{y}}

\newcommand{\bg}{\bm{g}}

\newcommand{\bomega}{\bm{\omega}}
\newcommand{\bepsilon}{\bm{\epsilon}}

\newcommand{\btheta}{\bm{\theta}}

\newcommand{\outputspace}{\mathcal{Y}}

\newcommand{\bnoisedist}{\rho(\bepsilon)}

\newcommand{\exy}{\hat{\by}}

\newcommand{\imle}{\textsc{I-MLE}\@\xspace}

\newcommand{\grad}[2]{\nabla_{#1}#2}

\usepackage{pifont}
\newcommand{\cmark}{\ding{51}}
\newcommand{\xmark}{\ding{55}}

\newcommand{\gin}{\textsf{GIN}\xspace}
\newcommand{\gcn}{\textsf{GCN}\xspace}
\newcommand{\esan}{\textsf{ESAN}\xspace}
\newcommand{\dswl}{\textsf{DS\text{-}WL}\xspace}

\newcommand{\xhdr}[1]{{\noindent\bfseries #1}}
\usepackage[inline]{enumitem}

\usepackage{adjustbox}
\usepackage{subcaption}
\usepackage{booktabs}
\usepackage{url}
\usepackage{multirow}
\usepackage{tikz}
\usetikzlibrary{calc}
\usetikzlibrary{decorations.pathmorphing}
\usetikzlibrary{positioning}
\usetikzlibrary{shapes}
\usetikzlibrary{arrows.meta,backgrounds,automata}
\usetikzlibrary{positioning,shapes,matrix,calc}
\tikzstyle{vertex}=[circle, draw, fill=gray!80!white,thick,scale=1.2]
\tikzstyle{edge}=[draw=black, thick,-]
\usetikzlibrary{hobby,arrows,decorations.pathmorphing,backgrounds,positioning,fit,petri,calc}
\pgfdeclarelayer{background}
\pgfsetlayers{background,main}
\usepackage{csquotes}

\usepackage[framemethod=tikz]{mdframed}

\definecolor{mycolor}{rgb}{0.122, 0.435, 0.698}

\newmdenv[innerlinewidth=0.5pt, roundcorner=4pt,linecolor=mycolor,innerleftmargin=6pt,
innerrightmargin=6pt,innertopmargin=6pt,innerbottommargin=6pt]{mybox}

\usepackage{bm}
\usepackage{amsmath}
\usepackage{amssymb}
\usepackage{amsthm}
\usepackage{amsfonts}
\usepackage{thmtools}		
\usepackage{mleftright}
\usepackage{stmaryrd}
\usepackage{nicefrac}
\usepackage{algorithm}
\usepackage{algorithmicx}
\usepackage[noend]{algpseudocode}

\theoremstyle{definition}
\newtheorem{theorem}{Theorem}

\newtheorem{proposition}[theorem]{Proposition}

\newtheorem{lemma}[theorem]{Lemma}

\usepackage{thm-restate}
\usepackage[mathic=true]{mathtools}
\usepackage{fixmath}
\usepackage{siunitx}

\usepackage{setspace}
\usepackage{ellipsis}
\usepackage{xspace}

\usepackage{tcolorbox}
\newcommand{\comm}[1]{}

\makeatletter
\def\thmt@refnamewithcomma #1#2#3,#4,#5\@nil{%
	\@xa\def\csname\thmt@envname #1utorefname\endcsname{#3}%
	\ifcsname #2refname\endcsname
	\csname #2refname\expandafter\endcsname\expandafter{\thmt@envname}{#3}{#4}%
	\fi
}
\makeatother
\usepackage[capitalise,noabbrev]{cleveref}  

\title{Ordered Subgraph Aggregation Networks}
\author{%
  Chendi Qian\thanks{These authors contributed equally.}\\
  Department of Computer Science\\
  TU Munich\\
  \And
  Gaurav Rattan$^*$\\
  Department of Computer Science\\
  RWTH Aachen University\\
  \And 
  Floris Geerts\\
  Department of Computer Science\\
  University of Antwerp\\
  \And 
  Christopher Morris\\
  Department of Computer Science\\
  RWTH Aachen University\\
  \And 
  Mathias Niepert\\
  Department of Computer Science\\
  University of Stuttgart\\
}

\hypersetup{draft}
\begin{document}

\maketitle

\begin{abstract}
Numerous subgraph-enhanced graph neural networks (GNNs) have emerged recently, provably boosting the expressive power of standard (message-passing) GNNs. However, there is a limited understanding of how these approaches relate to each other and to the Weisfeiler--Leman hierarchy. Moreover, current approaches either use all subgraphs of a given size, sample them uniformly at random, or use hand-crafted heuristics instead of learning to select subgraphs in a data-driven manner. Here, we offer a unified way to study such architectures by introducing a theoretical framework and extending the known expressivity results of subgraph-enhanced GNNs. Concretely, we show that increasing subgraph size always increases the expressive power and develop a better understanding of their limitations by relating them to the established $\WLk{k}$ hierarchy. In addition, we explore different approaches for learning to sample subgraphs using recent methods for backpropagating through complex discrete probability distributions. Empirically, we study the predictive performance of different subgraph-enhanced GNNs, showing that our data-driven architectures increase prediction accuracy on standard benchmark datasets compared to non-data-driven subgraph-enhanced graph neural networks while reducing computation time. 
\end{abstract}

\section{Introduction}
Graph-structured data are ubiquitous across application domains ranging from chemo- and bioinformatics~\citep{Barabasi2004,Jum+2021,Sto+2020} to image~\citep{Sim+2017} and social-network analysis~\citep{Eas+2010}. Numerous approaches for graph--based machine learning have been proposed, most notably those based on \new{graph kernels}~\citep{Borg+2020,Kri+2019} or using \new{graph neural networks} (GNNs)~\citep{Cha+2020,Gil+2017,Mor+2022}. Here, graph kernels based on the \new{$1$-dimensional Weisfeiler--Leman algorithm} ($\WLk{1}$)~\citep{Wei+1968}, a simple heuristic for the graph isomorphism problem, and corresponding GNNs~\citep{Mor+2019,Xu+2018b}, have recently advanced the state-of-the-art in supervised vertex- and graph-level learning. However, the $\WLk{1}$ and GNNs operate via local neighborhood aggregation, missing crucial patterns in the given data while more expressive architectures based on the \emph{$k$-dimensional Weisfeiler--Leman algorithm} ($\WLk{k}$)~\citep{Azi+2020,Mar+2019,Morris2020b,Mor+2022,Mor+2022b} may not scale to larger graphs.

Hence, several approaches such as~\citet{Bev+2021,Cot+2021,li2020distance,Pap+2021,Thi+2021,You+2021} and~\citet{Zha+2021b} have enhanced the expressive power of GNNs, by removing, extracting, or marking (small) subgraphs, so as to allow GNNs to leverage more structural patterns within the given graph, essentially breaking symmetries induced by the GNNs' local aggregation function. We henceforth refer to these approaches as \textit{subgraph-enhanced GNNs}.

\xhdr{Present work} 
First, to bring some order to the multitude of recently proposed subgraph-enhanced GNNs, we introduce a theoretical framework to study these approaches' expressive power in a unified setting. Concretely, 
\begin{itemize}
	\item we introduce \new{$k$-ordered subgraph aggregation networks} (\PMPNNs{k}) and show that they capture most of the recently proposed subgraph-enhanced GNNs.
	\item We show that any \PMPNN{k} is upper bounded by $\WLk{(k+1)}$ in terms of expressive power and show that \PMPNNs{k} and $\WLk{k}$ are incomparable in terms of expressive power. Consequently, we obtain new upper bounds on the expressive power of recently proposed subgraph-enhanced GNNs.
	\item We show that increasing $k$, i.e., using larger subgraphs, always leads to an increase in expressive power, effectively showing that \PMPNNs{k} form a hierarchy.
\end{itemize}

Second, most approaches consider all subgraphs or use hand-crafted heuristics to select them, e.g., by deleting vertices or edges. Instead, we leverage recent progress in back-propagating through discrete structures using perturbation-based differentiation~\citep{Dom+2010,Nie+2021} to sample subgraphs in a \textit{data-driven} fashion, automatically adapting to the given data distribution. 
Concretely,   
\begin{itemize}
	\item we explore different strategies to sample subgraphs leveraging the \imle framework~\citep{Nie+2021}, resulting in the data-driven \PMPNN{k} architecture.
	\item We show, empirically, that data-driven \PMPNNs{k} increase prediction accuracy on standard benchmark datasets compared to non-data-driven subgraph-enhanced GNNs while vastly reducing computation time.
\end{itemize}

\subsection{Related work}
In the following, we discuss related work relevant to the present work; see~\cref{exp_rel} for an extended discussion.

\xhdr{GNNs} Recently, GNNs~\citep{Gil+2017,Sca+2009} emerged as the most prominent graph representation learning architecture. Notable instances of this architecture include, e.g.,~\citet{Duv+2015,Ham+2017} and~\citet{Vel+2018}, which can be subsumed under the message-passing framework introduced in~\citet{Gil+2017}. In parallel, approaches based on spectral information were introduced in, e.g.,~\citet{Defferrard2016,Bru+2014,Kip+2017} and~\citet{Mon+2017}---all of which descend from early work in~\citet{bas+1997,Kir+1995,mic+2005,Mer+2005,mic+2009,Sca+2009} and~\citet{Spe+1997}. 

\xhdr{Limits of GNNs and more expressive architectures} 
Recently, connections between GNNs and Weisfeiler--Leman type algorithms have been shown~\citep{Azi+2020,Bar+2020,Che+2019,Gee+2020a,Gee+2020b,geerts2022,Mae+2019,Mar+2019,Mor+2019,Mor+2022b,Xu+2018b}. Specifically,~\citet{Mor+2019} and~\citet{Xu+2018b} showed that the expressive power of any possible GNN architecture is limited by the $\WLk{1}$ in terms of distinguishing non-isomorphic graphs. 

Recent works have extended the expressive power of GNNs, e.g., by encoding vertex identifiers~\citep{Mur+2019b, Vig+2020}, using random features~\citep{Abb+2020,Das+2020,Sat+2020}, homomorphism and subgraph counts~\citep{Bar+2021,botsas2020improving,Hoa+2020}, spectral information~\citep{Bal+2021}, simplicial and cellular complexes~\citep{Bod+2021,Bod+2021b}, persistent homology~\citep{Hor+2021}, random walks~\citep{Toe+2021}, graph decompositions~\citep{Tal+2021}, or distance~\citep{li2020distance} and directional information~\citep{beaini2020directional}. See~\citet{Mor+2022} for an in-depth survey on this topic. 

\xhdr{Subgraph-enhanced GNNs} Most relevant to the present work are \emph{subgraph-enhanced GNNs}. \citet{Cot+2021} and~\citet{Pap+2021} showed how to make GNNs more expressive by removing one or more vertices from a given graph and using standard GNN architectures to learn vectorial representations of the resulting subgraphs. The approaches either consider all possible subgraphs or utilize random sampling to arrive at more scalable architectures.  \citet{Cot+2021} showed that by removing one or two vertices, such architectures can distinguish graphs the $\WLk{1}$ and $\WLk{2}$, respectively, are not able to distinguish. 
Extensions and refinements of the above were proposed in~\citet{Bev+2021,Pap+2022,Thi+2021,You+2021,Zha+2021} and~\citet{Zha+2021b}, see~\citet{Pap+2022} for an overview. For example, \citet{Bev+2021} generalized several ideas discussed above and proposed the \esan framework in which each graph is represented as a multiset of its subgraphs and processed them using an equivariant architecture based on the DSS architecture \citep{maron2020learning} and GNNs. The authors proposed several simple subgraph selection policies, e.g., edge removal, ego networks, or vertex removal, and showed that the architecture surpasses the expressive power of the $\WLk{1}$. Moreover, \citet{frasca2022understanding} presented a novel symmetry analysis unifying a series of subgraph-enhanced GNNs, allowing them to upper-bound their expressive power and to define a systematic framework to conceive novel architectures in this family. We note here that the above works, unlike the present one, mostly do not study the approaches' expressive power beyond (folklore or non-oblivious) $\WLk{2}$ and do not compare at all to the (folklore or non-oblivious) $\WLk{3}$, while our analysis works for the whole $\WLk{k}$ hierarchy.

See~\cref{exp_rel} for a detailed overview of recent progress in differentiating through discrete structures.

\section{Preliminaries}\label{sec:prelim}
As usual, for $n \geq 1$, let $[n] \coloneqq \{ 1, \dotsc, n \} \subset \NN$. We use $\{\!\!\{ \dots\}\!\!\}$ to denote multisets, i.e., the generalization of sets allowing for multiple instances of each of its elements. 

A \new{graph} $G$ is a pair $(V(G),E(G))$ with \emph{finite} sets of \new{vertices} $V(G)$ and \new{edges} $E(G) \subseteq \{ \{u,v\}
\subseteq V(G) \mid u \neq v \}$. If not otherwise stated, we set $n \coloneqq |V(G)|$. For ease of
notation, we denote the edge $\{u,v\}$ in $E(G)$ by $(u,v)$ or
$(v,u)$. In the case of \emph{directed graphs}, $E \subseteq \{ (u,v)
\in V \times V \mid u \neq v \}$. A \new{labeled graph} $G$ is a triple
$(V,E,l)$ with \new{(vertex) coloring} or \new{label function} $l \colon V(G) \to \bbN$. Then $l(v)$ is a
\new{label} of $v$ for $v$ in $V(G)$. The \new{neighborhood} of $v$ in $G$ is denoted by $N_G(v)\coloneqq \{ u \in V(G) \mid \{ v, u \} \in E(G) \}$ and the \new{degree} of a vertex $v$ is $|N_G(v)|$. For $S \subseteq
V(G)$, the graph $G[S] = (S,E_S)$ is the \new{subgraph induced by $S$}, where $E_S\coloneqq \{ (u,v) \in E(G) \mid u,v \in S \}$.

Two graphs $G$ and $H$ are \new{isomorphic} and we write $G \simeq H$ if there exists a bijection $\varphi \colon V(G) \to V(H)$ that preserves the adjacency relation, i.e., $(u,v)$ is in $E(G)$ if and only if $(\varphi(u),\varphi(v))$ is in $E(H)$. Then $\varphi$ is an \new{isomorphism} between
$G$ and $H$. Moreover, we call the equivalence classes induced by the relation $\simeq$ \emph{isomorphism types}. 
In the case of labeled graphs, we additionally require that $l(v) = l(\varphi(v))$ for $v$ in $V(G)$. We further define the atomic type $\mathsf{atp}\colon V(G)^k \to \bbN$ such that $\mathsf{atp}(\vec{v}) = \mathsf{atp}(\vec{w})$ for $\vec{v},\vec{w} \in V(G)^k$ if and only if the mapping $\varphi\colon V(G)^k \to V(G)^k$ where $v_i \mapsto w_i$ induces a partial isomorphism, i.e., we have $v_i = v_j \iff w_i = w_j$ and $(v_i,v_j) \in E(G) \iff (\varphi(v_i),\varphi(v_j)) \in E(G)$. Let $\vec{v}$ be a \emph{tuple} in $V(G)^k$ for $k > 0$, then $G[\vec{v}]$ is the \emph{ordered $k$-vertex subgraph} induced by the multiset of elements of $\vec{v}$, where the vertices are labeled with integers from $[k]$ corresponding to their positions in $\vec{v}$.\comm{\footnote{More specifically, we treat labels of vertices $\vec{v}$ in $G[\vec{v}]$ as being separate from the labels assigned by the label function $l$, in order to keep these labels separate. Moreover, a vertex $v$ in $\vec{v}$ may obtain multiple labels in $G[\vec{v}]$ in case $v$ occurs multiple  times in $\vec{v}$. Having multiple labels can be avoided, e.g., by using a hot-one encoding.}}
Moreover, let $\st(G[\vec{v}]) \coloneqq \vec{v}$, i.e., the $k$-tuple $\vec{v}$ underlying the ordered $k$-vertex subgraph $G[\vec{v}]$. We denote the set of all ordered $k$-vertex subgraphs of a graph $G$ by $G_k$. Finally, let $\mathcal{G}$ be the set of all vertex-labeled graphs. 

\xhdr{The $\WLk{1}$ and the $\WLk{k}$} The $\WLk{1}$ or color refinement is a simple heuristic for the graph isomorphism problem, originally proposed by~\citet{Wei+1968}.\footnote{Strictly speaking, the $\WLk{1}$ and color refinement are two different algorithms. That is, the  $\WLk{1}$ considers neighbors and non-neighbors to update the coloring, resulting in a slightly higher expressive power when distinguishing vertices in a given graph, see~\citet{Gro+2021} for details. For brevity, we consider both algorithms to be equivalent.}
Intuitively, the algorithm determines if two graphs are non-isomorphic by iteratively coloring or labeling vertices. Given an initial coloring or labeling of the vertices of both
graphs, e.g., their degree or application-specific information, in
each iteration, two vertices with the same label get different labels if the number of identically labeled neighbors is not equal. If, after some iteration, the number of vertices annotated with a specific label is different in both graphs, the algorithm terminates and a stable coloring (partition) is obtained. We can then conclude that the two graphs are not isomorphic. It is easy to see that the algorithm cannot distinguish all non-isomorphic graphs~\citep{Cai+1992}. Nonetheless, it is a powerful heuristic that can successfully test isomorphism for a broad class of graphs~\citep{Bab+1979}.

Formally, let $G = (V,E,l)$ be a labeled graph. In each iteration, $i > 0$, the $\WLk{1}$ computes a vertex coloring $C^1_i \colon V(G) \to \bbN$,
which depends on the coloring of the neighbors. That is, in iteration $i>0$, we set
\begin{equation*}
	C^1_i(v) \coloneqq \REL\Big(\!\big(C^1_{i-1}(v),\oms C^1_{i-1}(u) \mid u \in N_G(v)  \cms \big)\! \Big),
\end{equation*}
where $\REL$ injectively maps the above pair to a unique natural number, which has not been used in previous iterations.  In iteration $0$, the coloring $C^1_{0}\coloneqq l$. To test if two graphs $G$ and $H$ are non-isomorphic, we run the above algorithm in ``parallel'' on both graphs. If the two graphs have a different number of vertices colored $c$ in $\bbN$ at some iteration, the $\WLk{1}$ \new{distinguishes} the graphs as non-isomorphic. Moreover, if the number of colors between two iterations, $i$ and $(i+1)$, does not change, i.e., the cardinalities of the images of $C^1_{i}$ and $C^1_{i+1}$ are equal,
\comm{or, equivalently,
\begin{equation*}
	C^1_{i}(v) = C^1_{i}(w) \iff C^1_{i+1}(v) = C^1_{i+1}(w),
\end{equation*}
for all vertices $v$ and $w$ in $V(G)$, }the algorithm terminates. For such $i$, we define the \new{stable coloring}
$C^1_{\infty}(v) = C^1_i(v)$ for $v$ in $V(G)$. The stable coloring is reached after at most $\max \{ |V(G)|,|V(H)| \}$ iterations~\citep{Gro2017}.

Due to the shortcomings of the $\WLk{1}$  or color refinement in distinguishing non-isomorphic
graphs, several researchers~\citep{Bab1979,Bab2016,Imm+1990},
devised a more powerful generalization of the former, today known
as the \new{$k$-dimensional Weisfeiler-Leman algorithm}
($\WLk{k}$); see~\cref{kwl} for a detailed description.

\xhdr{Graph Neural Networks}
\label{sec:gnn}
Intuitively, GNNs learn a vectorial representation, i.e., a $d$-dimensional vector, representing each vertex in a graph by aggregating information from neighboring vertices. Formally, let $G = (V,E,l)$ be a labeled graph with initial vertex features $\hb_{v}^\tup{0} \in \RR^{d}$ that are \emph{consistent} with $l$. That is, each vertex $v$ is annotated with a feature  $\hb_{v}^\tup{0} \in \bbR^{d}$ such that $\hb_{u}^\tup{0} = \hb_{v}^\tup{0}$ if $l(u) = l(v)$, e.g., a one-hot encoding of the labels $l(u)$ and $l(v)$. Alternatively,  $\hb_{v}^\tup{0}$ can be an arbitrary real-valued feature vector or attribute of the vertex $v$, e.g., physical measurements in the case of chemical molecules. A GNN architecture consists of a stack of neural network layers, i.e., a composition of parameterized functions. Similarly to $\WLk{1}$, each layer aggregates local neighborhood information, i.e., the neighbors' features, around each vertex and then passes this aggregated information on to the next layer.

Following, \citet{Gil+2017} and \citet{Sca+2009}, in each layer, $i > 0$,  we compute vertex features
\begin{equation*}\label{def:gnn}
	\hb_{v}^\tup{i+1} \coloneqq
	\UPD^\tup{i+1}\Bigl(\hb_{v}^\tup{i},\AGG^\tup{i+1} \bigl(\oms \hb_{u}^\tup{i}
	\mid u\in N_G(v) \cms \bigr)\Bigr) \in \RR^{d},   
\end{equation*}
where  $\UPD^\tup{i+1}$ and $\AGG^\tup{i+1}$ may be differentiable parameterized functions, e.g., neural networks.\footnote{Strictly speaking, \citet{Gil+2017} consider a slightly more general setting in which vertex features are computed by $\hb_{v}^\tup{i+1} \coloneqq
	\UPD^\tup{i+1}\Bigl(\hb_{v}^\tup{i},\AGG^\tup{i+1} \bigl(\oms (\hb_v^\tup{i},\hb_{u}^\tup{i},l(v,u))
	\mid u\in N_G(v) \cms \bigr)\Bigr)$.}
In the case of graph-level tasks, e.g., graph classification, one uses 
\begin{equation}\label{readout}
	\hb_G \coloneqq \RO\bigl( \oms \hb_{v}^{\tup{T}}\mid v\in V(G) \cms \bigr) \in \RR^{d},
\end{equation}
to compute a single vectorial representation based on learned vertex features after iteration $T$. Again, $\RO$  may be a differentiable parameterized function. To adapt the parameters of the above three functions, they are optimized end-to-end, usually through a variant of stochastic gradient descent, e.g.,~\citep{Kin+2015}, together with the parameters of a neural network used for classification or regression.

\xhdr{The Weisfeiler--Leman hierarchy and permutation-invariant function approximation}\label{connect}
The Weisfeiler--Leman hierarchy is a purely combinatorial algorithm for testing graph isomorphism. However,  the graph isomorphism function, mapping non-isomorphic graphs to different values, is the hardest to approximate permutation-invariant function. Hence, the Weisfeiler--Leman hierarchy has strong ties to GNNs' capabilities to approximate permutation-invariant or equivariant functions over graphs. For example,~\citet{Mor+2019,Xu+2018b} showed that the expressive power of any possible GNN architecture is limited by $\WLk{1}$ in terms of distinguishing non-isomorphic graphs. \citet{Azi+2020} refined these results by showing that if an architecture is capable of simulating $\WLk{k}$ and allows the application of universal neural networks on vertex features, it will be able to approximate any permutation-equivariant function below the expressive power of $\WLk{k}$; see also~\citet{Che+2019}. Hence, if one shows that one architecture distinguishes more graphs than another, it follows that the corresponding GNN can approximate more functions. These results were refined in \citet{geerts2022} for color refinement and taking into account the number of iterations of $\WLk{k}$.

\section{Ordered subgraph Weisfeiler--Leman and MPNNs}
In the following, we introduce a variant of $\WLk{1}$, denoted \new{$k$-ordered subgraph WL} ($\PWLk{k}$). Essentially, the $\PWLk{k}$ labels or marks ordered subgraphs and then executes $\WLk{1}$ on top of the marked graphs, followed by an aggregation phase. Although unordered subgraphs are also possible, ordered ones lead to more expressive architectures and also encompass the unordered case; see~\cref{unordered} for a discussion. To make the procedure permutation-invariant, we consider all possible ordered subgraphs. Based on the ideas of $\PWLk{k}$, we then introduce \PMPNNs{k}, which can be seen as a neural variant of the former, allowing us to analyze various subgraph-enhanced GNNs. 

\subsection{Ordered subgraph WL} We now describe the algorithm formally. Let $G$ be a graph, and let $\pb\in G_k$ be an ordered $k$-vertex subgraph. Then $\PWLk{k}$ computes a vertex coloring, similarly to $\WLk{1}$, with the main distinction that it can use structural graph information related to the ordered subgraph $G[(v, \st(\pb))]$, where $v$ is a vertex in $V(G)$.

More precisely, at each iteration $i \geq 0$, $\PWLk{k}$ computes a coloring $C_i \colon V(G) \times G_{k} \to\Nbb$ where we interpret elements $(v,\pb)\in V(G) \times G_{k} $ as a vertex $v$ along with an ordered $k$-vertex subgraph. 
Given an ordered $k$-vertex subgraph   $\pb\in G_k$, initially, for $i=0$, we set $C_{0}(v,\pb) \coloneqq \mathsf{atp}{(v,\st(\pb))}$, and for $i > 0$, we set
\begin{equation*}
	C_{i+1}(v,\pb) \coloneqq  \REL \Big(\!\big(C_i(v,\pb),\oms C_{i}(u,\pb)\mid u \in \square  \cms \big)\!\Big), 
\end{equation*}
where $\square$ is either $N_G(v)$ or $V(G)$. We compute the stable partition analogously to $\WLk{1}$. Finally, to compute a single color for a vertex $v$, we aggregate all ordered $k$-vertex subgraphs, i.e., we compute 
\begin{equation}\label{vp}
	C(v) \coloneqq  \REL \bigl(\oms C_{\infty}(v,\pb) \mid \pb\in G_k \cms \bigr).
\end{equation}
In other words, one can regard the $\PWLk{k}$ as running $\WLk{1}$ in parallel over $n^k$ graphs, one for each ordered $k$-vertex subgraph $\pb\in G_k$, followed by combining the colors of each vertex in all these graphs. Furthermore, by restricting the number of considered subgraphs, the algorithm allows for more fine-grained control over the trade-off between scalability and expressivity. Note that $\PWLk{0}$ is equal to $\WLk{1}$. We also define a variation of the $\PWLk{k}$, denoted \emph{vertex-subgraph $\PWLk{k}$}, which, unlike~\cref{vp}, first computes a color $C(\pb)$ for each ordered $k$-vertex subgraph $\pb$ by aggregating over vertices; see~\cref{sec:vertexpebble} for details.

We remark that in contrast to $\WLk{k}$, which has to update the coloring of all  $n^k$ ordered $k$-vertex subgraphs in a complicated manner, the computation of $\PWLk{k}$'s coloring relies on the simple and easy-to-implement $\WLk{1}$. Furthermore, $\PWLk{k}$'s computation can be either done in parallel or sequentially across all $n$ vertices and $n^k$ graphs. Despite its simplicity, in~\cref{expressivity}, we will show that the $\PWLk{k}$ has high expressivity.

\subsection{Ordered subgraph MPNNs}
In the following, to study the expressivity of subgraph-enhanced GNNs, we introduce $k$-ordered subgraph MPNNs (\PMPNNs{k}), which can be viewed as neural variants of the $\PWLk{k}$. At initialization, \PMPNNs{k} learn two features for each element in $G_k$ and each vertex $v$
\allowdisplaybreaks
\begin{linenomath}
	\postdisplaypenalty=0
	\begin{align*}
		\hb_{{v},\pb}^\tup{0} \coloneqq\UPD(\sf{atp}(\text{$v$},\st(\pb))) \in \RR^d, \quad \text{ and } \quad \pmb\pi_{\text{$v$},\pb}\coloneqq\UPD_{\pmb\pi}(\mathsf{atp}(\text{$v$},\st(\pb))),
	\end{align*}
\end{linenomath}
where $\UPD$ and $\UPD_{\pmb\pi}$ are differentiable, parameterized function, e.g., a neural network. Additional vertex features can be concatenated to the first feature. We use the second feature $\pmb\pi_{v,\pb}$ to select admissible ordered subgraphs for the vertex $v$; see below. Now in each layer $(i+1)$, we update the feature of a vertex $v$ with regard to the $k$-ordered subgraph $\vec{g}$  as
\allowdisplaybreaks
\begin{linenomath}
	\postdisplaypenalty=0
	\begin{align*}
		\hb_{v,\pb}^\tup{i+1}& \coloneqq\UPD^\tup{i+1}\Bigl(
		\hb_{v,\pb}^\tup{i},\AGG^\tup{i+1}\bigl(\oms \hb_{u,\pb}^\tup{i} \mid u\in \square \cms\bigr)\Bigr),
	\end{align*}
\end{linenomath}
where $\square$ is either $N_G(v)$ or $V(G)$. After $T$ such layers, for each vertex $v$, we then learn a joint feature over all $k$-ordered subgraphs, i.e., we apply subgraph aggregation
\allowdisplaybreaks
\begin{linenomath}
	\postdisplaypenalty=0
	\begin{align}\label{vpmpnn}
		\hb_{v}^\tup{T}&\coloneqq\pAGG\bigl(\oms \hb_{v,\pb}^\tup{T}\mid \pb\in G_k \text{ s.t. }\pmb\pi_{v,\pb}\neq \mathbf{0}\cms\bigr).
	\end{align}	
\end{linenomath}
Here, we leverage $\pmb\pi_{v,\pb}\neq \mathbf{0}$ to select a subset of the set of $k$-ordered subgraphs. Finally, analogous to GNNs, we use a $\RO$ layer to compute a single graph feature. Again, $\AGG^\tup{i+1}$, $\UPD^\tup{i+1}$, $\RO$, and $\pAGG$ are differentiable, parameterized functions, e.g., neural networks.

We also define a variation of \PMPNNs{k}, denoted \emph{vertex-subgraph \PMPNNs{k}}, which, unlike~\cref{vpmpnn}, first compute a color for each ordered $k$-vertex subgraph; see~\cref{sec:vertexpebble} for details.

\xhdr{Expressive power of \PMPNNs{k}}\label{expressivity}
In the following, we study the expressive power of \PMPNNs{k}. The first result shows that any possible \PMPNN{k} has at most the expressive power of $\PWLk{k}$ in terms of distinguishing non-isomorphic graphs. Further, \PMPNNs{k} are in principle capable of reaching $\PWLk{k}$'s expressive power. Hence, the  $\PWLk{k}$ upper bounds \PMPNNs{k} ability to represent permutation-invariant functions.

\begin{proposition}\label{equal} For all $k \geq 1$, it holds 
that \text{\PMPNNs{k}} are upper bounded by $\PWLk{k}$ in terms of distinguishing non-isomorphic graphs. Further, there exists a \text{\PMPNN{k}} instance that has exactly the same expressive power as the $\PWLk{k}$.
	\end{proposition}
\looseness=-1 The following result shows that any possible \PMPNN{k} is upper-bounded by the $\WLk{(k+1)}$ in terms of distinguishing non-isomorphic graphs while the expressive power of \PMPNNs{k} and the $\WLk{k}$ are incomparable. That is, there exist non-isomorphic graphs $\WLk{k}$   cannot distinguish while \PMPNNs{k} can and vice versa.
\begin{proposition}\label{prop:upperbound} For all $k \geq 1$, it holds that $\WLk{(k+1)}$ is \emph{stricly more} expressive than \text{\PMPNNs{k}} and there exist non-isomorphic graphs $\WLk{k}$  cannot distinguish while \PMPNNs{k} can and vice versa.
\end{proposition}
Finally, the following results shows that increasing the size of the subgraphs always leads to a strictly more expressive \PMPNNs{k}. 
\begin{theorem} 
For all $k \geq 1$, it holds that \PMPNNs{(k+1)} is \emph{strictly more}
expressive than \PMPNNs{k}.
\end{theorem}

\xhdr{Subgraph-enhanced GNNs captured by \PMPNNs{k}}
To exemplify the power and generality of \PMPNNs{k}, we show how \PMPNNs{k} cover most subgraph-enhanced GNNs; see~\cref{sec:sgnns} for a thorough overview. We say that \PMPNNs{k} \new{capture} a subgraph-enhanced GNN $G$ if there exists a \PMPNN{k} instance that is at least as expressive as $G$. 

The first results shows that \PMPNNs{k} capture $k$-marked GNNs (\mgnns{k})~\citep{Pap+2022}  and $k$-reconstruction GNNs (\recon{k})~\citep{Cot+2021}. For both approaches, the sets of $k$ vertices to be marked or deleted correspond to unordered $k$-vertex subgraphs. It then suffices to ensure that the update and aggregation functions in the \PMPNNs{k} treat the vertices in the selected $k$-vertex subgraphs as being marked or deleted.
\begin{proposition}\label{mark}
For $k \geq 1$, \PMPNNs{k} capture \mgnns{k} and vertex-subgraph \PMPNNs{k} capture \recon{k}. 
\end{proposition}
Further, \PMPNNs{1} capture identity-aware GNNs (\idgnn{1})~\citep{You+2021}, GNN As Kernel (\kernel)~\citep{Zha+2021b}, and nested GNNs (\nested)~\citep{Zha+2021}. Intuitively, in these approaches GNNs are used locally around each vertex. It thus suffices to ensure that the update and aggregation functions in the \PMPNNs{1} use the selected single vertex subgraph to only pass messages locally.
\begin{proposition}\label{identity}
\PMPNNs{1} capture \idgnn{1}, \kernel, and \nested.
\end{proposition}
Finally, \PMPNNs{k} capture the \dsgnn with the vertex-deleted policy~\citep{Bev+2021}.\footnote{We note here that it is an open question if vertex-subgraph \PMPNNs{k} also capture the more general \dssgnn~\citep{Bev+2021}.}
\begin{proposition}\label{esan}
Vertex-subgraph \PMPNNs{k} capture \dsgnn with the $k$-vertex-deleted policy.
\end{proposition}
We note that the above result can be further extended to accommodate the edge-deleted and ego-networks policy from~\citet{Bev+2021}; see~\cref{sec:sgnns}.

Importantly, by viewing existing subgraph-enhanced GNNs as \PMPNNs{k} we immediately gain insights into their expressive power. Previous results primarily focused on showing more expressivity than $\WLk{1}$.

\section{Data-driven Subgraph-enhanced GNNs}
In the above section, we thoroughly investigated the expressive power of subgraph-enhanced GNNs. Specifically, we showed that they are strictly limited by the $\WLk{(k+1)}$ and that they can distinguish graphs which are not distinguishable by $\WLk{k}$. As indicated by~\cref{equal} and to reach maximal expressive power, however, we need to consider all possible ordered subgraphs, resulting in an exponential running time. Hence, in this section, we leverage the \imle framework~\citep{Nie+2021}, to sample ordered subgraphs in a data-driven fashion. We first address the problem of learning the parameters of a probability distribution over ordered subgraphs using a GNN. Secondly, we show how to approximately sample from this intractable distribution. Subsequently, these subgraphs are used within a \PMPNN{k} to compute a graph representation. Finally, we propose a gradient estimation scheme that allows us to use backpropagation in the resulting discrete-continuous architecture.

\xhdr{Parameterizing probability distributions over subgraphs}
Contrary to existing approaches, which often consider all possible subgraphs or sample a fraction of subgraphs uniformly at random,  our method maintains a probability distribution over (ordered) subgraphs. Let $G$ be a graph where each vertex $v$ has an initial feature $\hb_{v}^\tup{0}$, which we stack row-wise over all vertices into the feature matrix $\mathbf{H} \in \bbR^{n \times d}$. Further, let $h_{\vec{W}_1} \colon \mathcal{G} \times \bbR^{n \times d} \to \mathbb{R}^{m \times n}$ be a permutation-equivariant function, e.g., a message-passing GNN, parameterized by $\vec{W}_1$, called \new{upstream model}, mapping a graph $G$ and its initial features $\mathbf{H}$ to a parameter matrix 
\begin{equation*}
	\btheta  \coloneqq h_{\vec{W}_1}(G, \mathbf{H}) \in  \mathbb{R}^{m \times n}.
\end{equation*}
Intuitively, each parameter $\theta_{ij}$ is an unnormalized prior probability of vertex $j$ being part of the $i$th sampled subgraph of $G$. Let $\btheta_i\coloneqq (\theta_{i1},\ldots,\theta_{in})$ for $i\in [m]$. We use these to parameterize $m$ probability distributions $p(\bz; \btheta_{i})$, for $i \in [m]$, over vector encodings of ordered $k$-vertex subgraphs of $G$, i.e.,
\begin{equation} \label{def-constrained-exp-family}
p(\bz; \btheta_{i})  \coloneqq \left\lbrace
\begin{array}{ll}
     \exp\left(\langle\bz,\btheta_{i}\rangle - A(\btheta_{i}) \right) & \text{if } \bz \in \mathcal{Z}, \\
     0 & \text{otherwise,}
\end{array}
\right.
\end{equation}
where $\langle\cdot, \cdot\rangle$ is the standard inner product and $A(\btheta_{i})$ is the \new{log-partition function} defined as $A(\btheta_{i}) \coloneqq \log\left(\sum_{\bz \in \mathcal{Z}} \exp \left(\langle\bz,\btheta_{i}\rangle  \right)\right)$. Furthermore, for a distribution over \emph{unordered} $k$-vertex subgraphs, $\mathcal{Z}$ is the set of all binary $n$-component vectors with exactly $k$ non-zero entries indicating which vertices are part of a subgraph of $G$. Hence, there is a bijection between $\mathcal{Z}$ and the set of unordered $k$-vertex subgraphs of $G$.
For a distribution over \emph{ordered} $k$-vertex subgraphs, the set $\mathcal{Z}$ is the set of all vectors in $[k]^n$ with $k$ non-zero entries. For each non-zero entry $\bz_i$ for $\bz \in \mathcal{Z}$ it holds that $\bz_i = j$ if and only if vertex $i$ has \new{rank} $k + 1 - j$ in the ordered subgraph, encoding the position in the ordered graph. This encoding is required for the gradient computation we perform later. For instance, in an ordered $5$-vertex subgraph, if a node has rank $1$ but should have rank $5$ to obtain a lower loss, then the gradient of the downstream loss is proportional to $5 - 1 = 4$. Similarly, if a node $i$ is not part of the ordered subgraph, that is, $\bz_i = 0$ but should be in position $1$ of the ordered subgraph, then the gradient of a downstream loss is proportional to $0 - 5$. For any $i,j \in [k]$ with $i \neq j$ we have that $\bz_i \neq \bz_j$. Hence, again, there is a bijection between $\mathcal{Z}$ and the set of ordered $k$-vertex subgraphs of $G$.

\xhdr{Efficient approximate sampling of subgraphs}
Computing the log-partition function and sampling exactly from the probability distribution in \cref{def-constrained-exp-family} is intractable for both ordered and unordered graphs. Since it is tractable, however, to compute a configuration with a highest probability, a maximum a posteriori (MAP) configuration $\bz^{*}(\btheta_{i})$, we can use perturb-and-MAP~\citep{Papandreou:2011,Nie+2021} to sample approximately. For unordered graphs, determining the top-$k$ values in $\btheta_i$ suffices, while for ordered graphs, we additionally require their rank.
Therefore, the worst-case running time of computing $\bz^{*}(\btheta_{i})$ for ordered graphs of size $k$ is in $O(n + k\log k)$. That is, we first use a selection algorithm to find the $k$th largest element $E$ in the list of weights, taking time $\cO(n)$, e.g., using the Quickselect algorithm. Now, we go through the list again and select all entries larger to $E$, taking time $\cO(n)$. Finally, we sort the $k$ values.

Now, to use perturb-and-MAP to \emph{approximately} sample the $i$-th ordered $k$-vertex subgraph $\pb_i$ from the above  probability distributions, we compute
\begin{equation*} \label{eq:hybridm_2}
 \pb_i := \mathtt{adj}\left(\bz^{*}(\btheta_{i} + \bepsilon_{i})\right) \quad \mbox{ with } \quad \bepsilon_i \sim \rho(\bepsilon),
\end{equation*}
where $\rho(\bepsilon)$ is a noise distribution such as the Gumbel distribution and $\mathtt{adj}$ converts the above vector encoding of the (ordered) subgraph to an $n \times n$ adjacency matrix 
as follows. The $j$th row or column encodes the vertex of the ordered subgraph with rank $j$ and its incident edges within the ordered subgraph. All other entries are set to $0$, i.e., they are masked out.
We therefore sample a multiset of (ordered) subgraphs $S \coloneqq \oms \bg_1, ..., \bg_m \cms \subseteq G_k$, which act 
  as the input to a \PMPNN{k} instance $f_{\vec{W}_2}$,  called \new{downstream model}, where $\pmb\pi_{v,\pb}\neq \mathbf{0}$ for $v \in V(G)$ if $\pb \in S$, to compute the target outputs
\begin{equation*} 
  f_{\vec{W}_2}(G, \mathbf{H}, \oms \bg_1, ..., \bg_m \cms).
\end{equation*}
    
\xhdr{Backpropagating through the subgraph distribution}
Now that we have outlined a way to approximately sample subgraphs, we still need to learn the parameters $\bomega = (\vec{W}_1, \vec{W}_2)$ of the upstream and downstream model. Hence, given a set of examples $\{(G_j, \mathbf{H}_j, \exy_j)\}_{j=1}^N$, we are concerned with finding approximate solutions to 
$\min_{\bomega} \nicefrac{1}{N} \sum_j L(G, \mathbf{H}, \exy_j; \bomega)$, where $L$ is the expected training error
\begin{equation} \label{eq:loss_cs}
	L(G, \mathbf{H}, \exy_j; \bomega)  \coloneqq \mathbb{E}_{\bg_i\sim p(\bz; \btheta_i)}\left[\ell\left(f_{\vec{W}_2}\left(G, \mathbf{H}, \oms \bg_1, ..., \bg_m \cms \right), \exy\right)\right],
\end{equation}
with $\btheta  \coloneqq h_{\vec{W}_1}(G, \mathbf{H})$ and $\ell \colon \outputspace\times\outputspace\to\RSet^+$ is a point-wise loss function. The challenge of training a model as defined in~\cref{eq:loss_cs} is to compute $\nabla_{\btheta} L(G, \mathbf{H}, \exy_j; \bomega)$, i.e., the gradient with respect to the parameters $\btheta$ of the probability distribution for the expected loss. In this work, we utilize implicit maximum likelihood learning, a recent framework that allows us to use algorithmic solvers of combinatorial optimization problems as black-box components~\citep{Rolinek:2020,Nie+2021}. 
A particular instance of the framework uses implicit differentation via perturbation~\citep{Dom+2010}. We derive the gradient computation for a single ordered subgraph $\pb_i$ to simplify the notation.
We compute the gradients of the downstream loss with respect to  parameters $\btheta_i$ as
\begin{equation*}
    \nabla_{\btheta_i}L(G, \mathbf{H}, \exy; \bomega) \approx \mathbb{E}_{\bepsilon_i \sim \bnoisedist} \left[\nicefrac{1}{\lambda} \left( \bz^{*}\left(\btheta_i + \bepsilon_i\right) - \bz^{*}\left(\btheta_i + \bepsilon_i - \lambda\widehat{\nabla}\right)\right)\right],
\end{equation*}
where $\lambda > 0$ and $\widehat{\nabla}_v$, the approximated gradient for a single vertex $v$, is defined as
\begin{equation*}
    \widehat{\nabla}_v \coloneqq \mathbf{agg} \bigl( \oms \left[\grad{\mathbf{\bg_i}}\ell\left(f_{\vec{W_2}}\left(G, \mathbf{H},\oms\bg_i \cms \right)\right)\right]_{v,w} \mid w \in N_G(v) \cms \bigr),
\end{equation*}
with $v \in V(G)$. Here, $\mathbf{agg}$ can be any aggregation function such as the element-wise sum or mean. Hence, to approximate the gradients with respect to the parameter $\btheta_{iv}$, which corresponds to vertex $v$ of the input graph, we aggregate the gradients of the downstream loss with respect to all edges $(v,w)$ incident to vertex $v$ in the \emph{original} graph.  

Hence, the above techniques allow us to efficiently learn to sample subgraphs, which are then fed into a \PMPNN{k} to learn a graph representation while optimizing the parameters of all components in an end-to-end fashion.

\section{Experimental evaluation}
Here, we aim to empirically investigate the learning performance and efficiency of data-driven subgraph-enhanced GNNs, instances of the \PMPNN{k} framework, compared to non-data-driven ones. Specifically, we aim to answer the following questions. 
\begin{description}
	\item[Q1] Do data-driven subgraph-enhanced GNNs 
		exhibit better predictive performance than non-data-driven ones?
	\item[Q2] Does the graph structure of the subgraphs sampled affect predictive performance?
	\item[Q3]  Does data-driven sampling have an advantage in efficiency and predictive performance when used within state-of-the-art subgraph-enhanced GNNs?
\end{description}

All experimental results are fully reproducible from the source code provided at \url{https://github.com/Spazierganger/OSAN}. 

\begin{table}[!t]
\centering
    \caption{Results on large-scale regression datasets, data-driven versus non-data-driven  subgraph sampling.\label{tab:results_1}}
    \begin{subtable}{.4\linewidth}
      \centering
      	\caption{Results for the \textsc{Ogbg-Molesol} dataset.}
	\label{t_mol}	
	\resizebox{.85\textwidth}{!}{ 	\renewcommand{\arraystretch}{1.0}

\begin{tabular}{@{}l <{\enspace}@{}ccccc@{}}	\toprule
			 \multicolumn{5}{c}{\textbf{Method}} &  \multirow{1}{*}{\textbf{RSMSE} $\downarrow$}
			 \\
			 \cmidrule{1-6}
			 \multicolumn{5}{c}{Baseline}   &   $1.193$  \scriptsize	$\pm 0.083$ \\
			\cmidrule{1-6}
			& \textsc{Operat.} & \textsc{Type} & {\textsc{\#}}         &  {\textsc{\# Subg.}}     &    \\	\toprule

			Random & \multirow{2}{*}{{Delete}}                     & \multirow{2}{*}{{Vertex}}     &  \multirow{2}{*}{{1}}     & \multirow{2}{*}{{3}}   &	$ 1.215$ \scriptsize $\pm  0.095$ \\
			\imle &   &    &      &  &	$ 1.053$ \scriptsize $\pm  0.080$ \\
        	\cmidrule{1-6}
        	Random & \multirow{2}{*}{{Delete}}                     & \multirow{2}{*}{{Vertex}}     &  \multirow{2}{*}{{1}}     & \multirow{2}{*}{{10}}   &	$ 1.128$ \scriptsize $\pm  0.055$ \\
			\imle &   &    &      &  &	$ 0.984$ \scriptsize $\pm  0.086$ \\
        	\cmidrule{1-6}
        	Random & \multirow{2}{*}{{Delete}}                     & \multirow{2}{*}{{Vertex}}     &  \multirow{2}{*}{{2}}     & \multirow{2}{*}{{1}}   &	$ 1.283$ \scriptsize $\pm  0.080$ \\
			\imle &   &    &      &  &	$ 0.968$ \scriptsize $\pm  0.102$ \\
        	\cmidrule{1-6}
        	Random & \multirow{2}{*}{{Delete}}                     & \multirow{2}{*}{{Vertex}}     &  \multirow{2}{*}{{2}}     & \multirow{2}{*}{{3}}   &	$ 1.132$ \scriptsize $\pm  0.020$ \\
			\imle &   &    &      &  &	$ 1.081$ \scriptsize $\pm  0.021$ \\
        	\cmidrule{1-6}
        	Random & \multirow{2}{*}{{Delete}}                     & \multirow{2}{*}{{Vertex}}     &  \multirow{2}{*}{{5}}     & \multirow{2}{*}{{3}}   &	$ 0.992$ \scriptsize $\pm  0.115$ \\
			\imle &   &    &      &  &	$ 1.115$ \scriptsize $\pm  0.076$ \\
        	\cmidrule{1-6}
        	Random & \multirow{2}{*}{{Delete}}                     & \multirow{2}{*}{{Vertex}}     &  \multirow{2}{*}{{5}}     & \multirow{2}{*}{{10}}   &	$ 1.186$ \scriptsize $\pm  0.154$ \\
			\imle &   &    &      &  &	$ 1.137$ \scriptsize $\pm  0.053$ \\
        	\cmidrule{1-6}
        	Random & \multirow{2}{*}{{Select}}                     & \multirow{2}{*}{{Vertex}}     &  \multirow{2}{*}{{10}}     & \multirow{2}{*}{{1}}   &	$ 1.128$ \scriptsize $\pm  0.022$ \\
			\imle &   &    &      &  &	$ 1.099$ \scriptsize $\pm  0.099$ \\
        	\cmidrule{1-6}
        	Random & \multirow{2}{*}{{Delete}}                     & \multirow{2}{*}{{Edge}}     &  \multirow{2}{*}{{1}}     & \multirow{2}{*}{{3}}   &	$ 1.240$ \scriptsize $\pm  0.029$ \\
			\imle &   &    &      &  &	$ 1.106$ \scriptsize $\pm  0.069$ \\
        	\cmidrule{1-6}
			Random & \multirow{2}{*}{{Delete}}                     & \multirow{2}{*}{{Edge}}     &  \multirow{2}{*}{{1}}     & \multirow{2}{*}{{10}}   &	$ 1.152$ \scriptsize $\pm  0.046$ \\
			\imle &   &    &      &  &	$ 1.056$ \scriptsize $\pm  0.071$ \\
        	\cmidrule{1-6}
        	Random & \multirow{2}{*}{{Delete}}                     & \multirow{2}{*}{{Edge}}     &  \multirow{2}{*}{{3}}     & \multirow{2}{*}{{3}}   &	$ 1.084$ \scriptsize $\pm  0.076$ \\
			\imle &   &    &      &  &	$ 1.052$ \scriptsize $\pm  0.049$ \\
        	\cmidrule{1-6}
        	Random & \multirow{2}{*}{{Delete}}                     & \multirow{2}{*}{{Edge}}     &  \multirow{2}{*}{{3}}     & \multirow{2}{*}{{10}}   &	$ 1.099$ \scriptsize $\pm  0.071$ \\
			\imle &   &    &      &  &	$ 1.077$ \scriptsize $\pm  0.079$ \\
        	\cmidrule{1-6}
			Random & \multirow{2}{*}{{Delete}}                     & \multirow{2}{*}{{$2$-Ego}}     &  \multirow{2}{*}{{--}}     & \multirow{2}{*}{{3}}   &	$1.071$ \scriptsize $\pm  0.062$ \\
			\imle &   &    &      &  &	$ 0.959$ \scriptsize $\pm  0.184$ \\
			\bottomrule
			
		\end{tabular}
	}
    \end{subtable}
    \begin{subtable}{.5\linewidth}
      \centering
\caption{Result for the \textsc{Alchemy} dataset.}
	\label{t_al}	
	\resizebox{.75\textwidth}{!}{ 	\renewcommand{\arraystretch}{1.0}
	
\begin{tabular}{@{}l <{\enspace}@{}ccccr@{}}	\toprule
			 \multicolumn{5}{c}{\textbf{Method}} &  \multirow{1}{*}{\textbf{MAE} $\downarrow$}
			 \\
			 \cmidrule{1-6}
			 \multicolumn{5}{c}{Baseline}   &   $11.12$  \scriptsize	$\pm 0.69$ \\
			\cmidrule{1-6}
			& \textsc{Operat.} & \textsc{Type} & {\textsc{\#}}         &  {\textsc{\# Subg.}}     &    \\	\toprule

			Random & \multirow{2}{*}{{Delete}}                     & \multirow{2}{*}{{Vertex}}     &  \multirow{2}{*}{{1}}     & \multirow{2}{*}{{3}}   &	$ 13.26$ \scriptsize $\pm  0.41$ \\
			\imle &   &    &      &  &	$ 8.78$ \scriptsize $\pm  0.28$ \\
        	\cmidrule{1-6}
        	Random & \multirow{2}{*}{{Delete}}                     & \multirow{2}{*}{{Vertex}}     &  \multirow{2}{*}{{1}}     & \multirow{2}{*}{{10}}   &	$ 12.11$ \scriptsize $\pm  0.21$ \\
			\imle &   &    &      &  &	$ 8.87$ \scriptsize $\pm  0.12$ \\
        	\cmidrule{1-6}
        	Random & \multirow{2}{*}{{Delete}}                     & \multirow{2}{*}{{Vertex}}     &  \multirow{2}{*}{{2}}     & \multirow{2}{*}{{3}}   &	$ 12.66$ \scriptsize $\pm  0.28$ \\
			\imle &   &    &      &  &	$ 9.01$ \scriptsize $\pm  0.27$ \\
        	\cmidrule{1-6}
        	Random & \multirow{2}{*}{{Delete}}                     & \multirow{2}{*}{{Vertex}}     &  \multirow{2}{*}{{5}}     & \multirow{2}{*}{{3}}   &	$ 10.29$ \scriptsize $\pm  0.30$ \\
			\imle &   &    &      &  &	$ 9.22$ \scriptsize $\pm  0.06$ \\
        	\cmidrule{1-6}
        	
			Random & \multirow{2}{*}{{Delete}}                     & \multirow{2}{*}{{Edge}}     &  \multirow{2}{*}{{1}}     & \multirow{2}{*}{{3}}   &	$ 11.66$ \scriptsize $\pm  0.63$ \\
			\imle &   &    &      &  &	$ 10.80$ \scriptsize $\pm  0.31$ \\
        	\cmidrule{1-6}
        	Random & \multirow{2}{*}{{Delete}}                     & \multirow{2}{*}{{Edge}}     &  \multirow{2}{*}{{2}}     & \multirow{2}{*}{{3}}   &	$ 10.79$ \scriptsize $\pm  0.64$ \\
			\imle &   &    &      &  &	$ 10.56$ \scriptsize $\pm  0.44$ \\
        	\cmidrule{1-6}
        	Random & \multirow{2}{*}{{Delete}}                     & \multirow{2}{*}{{Edge}}     &  \multirow{2}{*}{{5}}     & \multirow{2}{*}{{3}}   &	$ 9.15$ \scriptsize $\pm  0.12$ \\
			\imle &   &    &      &  &	$  9.08$ \scriptsize $\pm  0.28$ \\
        	\cmidrule{1-6}
        	
			Random & \multirow{2}{*}{{Select}}                     & \multirow{2}{*}{{Vertex}}     &  \multirow{2}{*}{{5}}     & \multirow{2}{*}{{3}}   &	$ 11.48$ \scriptsize $\pm  0.60$ \\
			\imle &   &    &      &  &	$ 9.22$ \scriptsize $\pm  0.14$ \\
			\cmidrule{1-6}
			Random & \multirow{2}{*}{{Select}}                     & \multirow{2}{*}{{Edge}}     &  \multirow{2}{*}{{5}}     & \multirow{2}{*}{{3}}   &	$ 8.99$ \scriptsize $\pm  0.24$ \\
			\imle &   &    &      &  &	$ 8.95$ \scriptsize $\pm  0.29$ \\
			\cmidrule{1-6}
			
			Random & \multirow{2}{*}{{Delete}}                     & \multirow{2}{*}{{$1$-Ego}}     &  \multirow{2}{*}{{--}}     & \multirow{2}{*}{{3}}   &	$ 14.98$ \scriptsize $\pm  0.49$ \\
			\imle &   &    &      &  &	$ 11.15$ \scriptsize $\pm  1.09$ \\
			\bottomrule
			
			Random & \multirow{2}{*}{{Select}}                     & \multirow{2}{*}{{$5$-Ego}}     &  \multirow{2}{*}{{--}}     & \multirow{2}{*}{{3}}   &	$ 14.97$ \scriptsize $\pm  0.23$ \\
			\imle &   &    &      &  &	$ 13.83$ \scriptsize $\pm  1.06$ \\
			\bottomrule
		\end{tabular}

	}
	   \vspace{-10pt}
    \end{subtable}
\end{table}

\xhdr{Datasets} 
To compare our data-driven, subgraph-enhanced GNNs to non-data-driven ones and standard GNN baselines, we used the \textsc{Alchemy}~\citep{Che+2019b}, the \textsc{Qm9}~\citep{Ram+2014,Wu+2018}, \textsc{Ogbg-Molesol}~\citep{hu2020ogb}, and the \textsc{Zinc}~\citep{Dwi+2020,Jin+2018} graph-level regression datasets; see~\cref{ds} in~\cref{app:exp} for dataset statistics and properties. In addition, we used the \textsc{Exp} dataset~\citep{Abb+2020} to investigate the additional expressive power of subgraph-enhanced GNNs over standard ones. Following~\cite{Morris2020b}, we opted not to use the 3D-coordinates of the \textsc{Alchemy} dataset to solely show the benefits of the data-driven subgraph-enhanced GNNs regarding graph structure. All datasets, excluding \textsc{Exp} and \textsc{Ogbg-Molesol}, are available from~\cite{Mor+2020}.\footnote{\url{https://chrsmrrs.github.io/datasets/}}

\xhdr{Neural architectures and experimental protocol} 
For all datasets and architectures, we used the competitive \gin layers~\citep{Xu+2018b} for the baselines and the downstream models. For data with (continuous) edge features, we used a $2$-layer MLP to map them to the same number of components as the vertex features and combined them using summation. We describe the upstream and downstream models' architecture used for each dataset in the following. We stress here that we always used the same hyperparameters for the downstream model and the baselines.

\xhdr{Sampling subgraphs}
Since the number of unordered $k$-vertex subgraphs is considerably smaller than the number of ordered $k$-vertex subgraphs, we opted to consider unordered $k$-vertex subgraphs; see also~\cref{unordered}. Further, since vertex-subgraph \PMPNNs{k}, see~\cref{sec:vertexpebble}, are easier to implement efficiently and are closer to \dsgnn variant of \esan~\citep{Bev+2021}, we opted to use them for the empirical evaluation. In addition, we used a simple GNN architecture for the upstream model to compute initial features for the subgraphs for ease of implementation. We experimented with selecting and deleting a various number of vertices, edges, and subgraphs induced by $k$-hop neighborhoods ($k$-Ego) for all datasets; see~\cref{app:exp} for details.

\xhdr{Upstream models} For all datasets and experiments, we used a \gcn model~\citep{Kip+2017} consisting of three  \gcn layers, with batch norm and ReLU activation after each layer. We set the hidden dimensions to that of the downstream model one. The model either outputs the vertex or edge embeddings according to the task. We computed edge embeddings based on the vertex features of the incident vertices after the last layers and the edge attributes provided by the dataset.

When sampling multiple subgraphs with \imle, they tend to have similar structures. In other words, \imle learns similar distributions in different channels of the neural network. 
This phenomenon is not in our favor, as we need to cover the original full graph as much as possible. To mitigate this issue, we propose an auxiliary loss for the diversity of subgraphs. We calculate the cosine similarity between the selected vertex or edge masks of every two subgraphs and try to minimize the average similarity value. We tune the weight for the auxiliary loss on the log scale, e.g., $0.1$, $1$, $10$, and so on.

\xhdr{Downstream and baseline models} See~\cref{exp_rel} for a detailed description of the architecture used for the downstream and baseline models, and how we processed subgraphs. For processing the subgraphs, we performed similar steps like \esan. We first applied intra-subgraph aggregation for the vertices within each subgraph and obtained graph embeddings for each subgraph. After that, we performed inter-subgraph mean pooling to obtain a single embedding vector for the original graph. It is worth noting that \esan does not exclude the vertices deleted during graph pooling but removes the adjacent edges of those nodes. In our experiments, we masked out the deleted or unselected nodes.  

See~\cref{app:exp} for further details on the experiments. 

\subsection{Results and discussion} In the following, we answer the research questions \textbf{Q1} to \textbf{Q3}.

\xhdr{A1} See~\cref{tab:results_1,tab:qm9,t_zc,tab:exp} (in the appendix). On all five datasets, the subgraph-enhanced GNN models based on \imle beat the random baseline, excluding edge sampling configurations on the \textsc{Alchemy} and the \textsc{Qm9} dataset; see~\cref{t_al,tab:qm9}. For example, on the \textsc{Ogbg-Molesol} the average gain over the random baseline is over 11\%. Similar improvements can be observed over the other four datasets. Moreover, the results on the \textsc{Exp} dataset, see~\cref{app:exp}, clearly indicate that the added expressivity of the (data-driven) subgraph-enhanced GNNs translates into improved predictive performance. The data-driven subgraph-enhanced GNNs improve the accuracy of the non-subgraph-enhanced GNN by almost 50\% in all configurations while improving over the random subgraph-enhanced GNN baseline by almost 6\%. The data-driven subgraph-enhanced GNNs also clearly improve over the (non-subgraph-enhanced) GNN baseline on four out of five datasets. 

\xhdr{A2}  See~\cref{tab:results_1,tab:qm9} (in the appendix). Deleting or selecting subgraphs leads to a clear boost in predictive performance across datasets over the random baseline while also improving over the non-subgraph-enhanced GNN baseline. Further, on all datasets, learning to delete or select $k$-hop neighborhood subgraphs for $k \in \{ 2,3 \}$ leads to a clear boost over the random as well as non-subgraph-enhanced baselines. However, the number of deleted vertices seems to affect the predictive performance. For example, on the \textsc{Ogbg-Molesol} dataset, going from deleting one 2-vertex subgraph to one 10-vertex subgraph  leads to a drop in performance. Hence, the drop in performance of the latter is in contrast to our theoretical findings, i.e., larger subgraphs lead to improved expressivity, indicating that more work should be done to understand subgraph-enhanced GNNs' generalization ability. Interestingly, deleting edges did not perform as well as deleting vertices or other subgraphs. We speculate that a more powerful edge embedding method is needed here, which computes edge features directly instead of learning them from vertex features. 

\xhdr{A3} See~\cref{tab:z_esan} (in the appendix). The \imle-based \esan severely speeds up the computation time. That is, across all configurations, we achieve a significant speed-up. For some configurations, e.g., sampling three vertices, the \imle based \esan is more than 3.5 times faster than the non-data-driven \esan while taking about the same time as the simple random baseline. We stress here that the \esan implementation provided by~\citet{Bev+2021} precomputes subgraphs in a preprocessing step, which is not possible when learning to sample subgraphs using \esan. Regarding predictive performance, the \imle based \esan is slightly behind the non-data-driven one, although always better than the non-subgraph enhanced GNN baseline; see~\cref{t_zc} in~\cref{app:exp}.

\section{Conclusion}
We introduced the \PMPNN{k} framework to study the expressive power of recently introduced subgraph-enhanced GNNs. We showed that any such architecture is strictly less powerful than the $\WLk{(k+1)}$ while being incomparable to the $\WLk{k}$ in representing permutation-invariant functions over graphs. Further, to circumvent random or heuristic subgraph selection, we devised a data-driven variant of \PMPNNs{k} which learn to select subgraph for a given data distribution. Empirically, we verified that such data-driven subgraph selection is superior to previously used random sampling in predictive performance. Further,  when compared to state-of-the-art models, we showed promising performance in terms of computation time while still providing good predictive performance. We believe that our paper provides a first step in unifying combinatorial insights on the expressive power of GNNs with data-driven insights. 

\begin{ack}
CM is partially funded a DFG Emmy Noether grant (468502433) and  RWTH Junior Principal Investigator Fellowship under the Excellence Strategy of the Federal Government and the Länder. GR is funded by the DFG Research Grants Program–RA 3242/1-1–411032549. MN acknowledges funding by the German Research Foundation under Germany's Excellence Strategy--EXC 2075.
\end{ack}

\bibliographystyle{abbrvnat}
\bibliography{main_old}

\appendix

\newpage
\section{Related work (expanded)}\label{exp_rel}

\xhdr{GNNs} Recently, GNNs~\citep{Gil+2017,Sca+2009} emerged as the most prominent graph representation learning architecture. Notable instances of this architecture include, e.g.,~\citep{Duv+2015,Ham+2017,Vel+2018}, which can be subsumed under the message-passing framework introduced in~\citep{Gil+2017}. In parallel, approaches based on spectral information were introduced in, e.g.,~\citep{Defferrard2016,Bru+2014,Kip+2017,Mon+2017}---all of which descend from early work in~\citep{bas+1997,Kir+1995,mic+2005,Mer+2005,mic+2009,Sca+2009,Spe+1997}. Recent extensions and improvements to the GNN framework include approaches to incorporate different local structures (around subgraphs), e.g.,~\citep{Hai+2019,Fla+2020,Jin+2020,Nie+2016,Xu+2018}, novel techniques for pooling vertex representations to perform graph classification, e.g.,~\citep{Bia+2020,Can+2018,Gao+2019,Gra+2021,Yin+2018,Zha+2018}, incorporating distance information~\citep{You+2019}, non-Euclidean geometry approaches~\citep{Cha+2019}, and more efficient GNNs, e.g.,~\citep{Fey+2021,Li+2021}. Furthermore, recently, empirical studies on neighborhood aggregation functions for continuous vertex features~\citep{Cor+2020}, edge-based GNNs that leverage physical knowledge~\citep{And+2019,Kli+2020,Kli+2021}, studying over-smoothing and over-squashing phenomena~\citep{Alon2020,Bod+2022,Goo+2021}, and sparsification methods~\citep{Ron+2020} emerged. Surveys of recent advancements in GNN techniques can be found, e.g., in~\citet{Cha+2020,Wu+2019,Zho+2018}. 

\xhdr{Limits of GNNs and more expressive architectures} 
Recently, connections of GNNs to Weisfeiler--Leman type algorithms have been shown~\citep{Azi+2020,Bar+2020,Che+2019,Gee+2020a,Gee+2020b,geerts2022,Mae+2019,Mar+2019,Mor+2019,Mor+2022b,Xu+2018b}. Specifically,~\citep{Mor+2019,Xu+2018b} showed that the expressive power of any possible GNN architecture is limited by the $\WLk{1}$ in terms of distinguishing non-isomorphic graphs. 

Triggered by the above results, a large set of papers proposed architectures to overcome the expressivity limitations of $\WLk{1}$.
\citet{Mor+2019} introduced \emph{$k$-dimensional} GNNs which rely on a message-passing scheme between subgraphs of cardinality~$k$. Similar to~\citep{Mor+2017}, the paper employed a local, set-based (neural) variant of the $\WLk{1}$. Later, this was refined in~\citep{Azi+2020,Mar+2019} by introducing \emph{$k$-order folklore graph neural networks}, which are equivalent to the folklore or oblivious variant of the $\WLk{k}$~\citep{Gro+2021,Mor+2022} in terms of distinguishing non-isomorphic graphs. Subsequently,~\citet{Morris2020b} introduced neural architectures based on a local version of the $\WLk{k}$, which only considers a subset of the original neighborhood, taking sparsity of the underlying graph (to some extent) into account.  \citet{Che+2019} connected the theory of universal approximations of permutation-invariant functions and the graph isomorphism viewpoint and introduced a variation of the $\WLk{2}$. \citet{geerts2022} introduced a higher-order message-passing framework that allows us to obtain upper bounds of extension of GNNs in terms of $\WLk{k}$.

Recent works have extended the expressive power of GNNs, e.g., by encoding vertex identifiers~\citep{Mur+2019b, Vig+2020}, using random features~\citep{Abb+2020,Das+2020,Sat+2020}, homomorphism and subgraph counts~\citep{Bar+2021,botsas2020improving,Hoa+2020}, spectral information~\citep{Bal+2021}, simplicial and cellular complexes~\citep{Bod+2021,Bod+2021b}, persistent homology~\citep{Hor+2021}, random walks~\citep{Toe+2021}, graph decompositions~\citep{Tal+2021}, or distance~\citep{li2020distance} and directional information~\citep{beaini2020directional}. See~\citet{Mor+2022} for an in-depth survey on this topic. 

\paragraph{Differentiating through discrete structures} 
Recently, numerous papers have aimed to combine discrete random variables and (continuous) neural network components and addressed the resulting gradient estimation problem. Most existing approaches used various types of relaxation of discrete distributions.
For instance, \citet{maddison2016concrete} and \cite{jang2016categorical} proposed the Gumbel-softmax distribution to relax categorical random variables. REBAR~\citep{tucker2017rebar} combined the Gumbel-softmax trick with the score-function estimator but is tailored to categorical distributions. Recent work on relaxed gradient estimators derived several extensions of the softmax trick~\citep{paulus2020gradient}. 
However, the Gumbel-softmax distribution is only directly applicable to categorical variables. For more complex distributions, one has to come up with tailor-made relaxations or use the straight-through or score function estimators, see, e.g., \citet{grover2019stochastic,kim2016exact}. Further, \citet{grathwohl2017backpropagation,tucker2017rebar} developed parameterized control variates based on continuous relaxations for the score-function estimator. In this work, we have to sample and select sparse, discrete, and complex substructures of a given input graph. Due to the resulting exponential number of possible substructures, we cannot use the Gumbel-softmax trick for categorical distributions. On the other hand, the requirement to sample sparse and discrete substructures does not allow us to utilize relaxations. Therefore, we use \imle, a recently proposed general-purpose framework to combine neural and discrete components~\citep{Nie+2021}.

\section{Weisfeiler--Leman algorithm (expanded)}\label{kwl}

Due to the shortcomings of the $\WLk{1}$  or color refinement in distinguishing non-isomorphic
graphs, several researchers~\citep{Bab1979,Bab2016,Imm+1990},
devised a more powerful generalization of the former, today known
as the \new{$k$-dimensional Weisfeiler-Leman algorithm}
($\WLk{k}$).\footnote{In~\citep{Bab2016} László Babai mentions that he
first introduced the algorithm in 1979 together with Rudolf Mathon from the University of Toronto.}\textsuperscript{,}\footnote{In this paper $\WLk{k}$ corresponds to original version \citep{Bab1979,Bab2016,Imm+1990} which is sometimes referred to as the ``folklore'' version in the literature. It corresponds to the ``oblivious'' $\WLk{(k+1)}$ version often used in the graph learning community \citep{Gro+2021}.}

Intuitively, to surpass the limitations of the $\WLk{1}$, the algorithm colors ordered subgraphs instead of a single vertex. More precisely, given a graph $G$, it colors the tuples from $V(G)^k$ for $k \geq 1$ instead of the vertices. By defining a neighborhood between these tuples, we can define a coloring  similar to the $\WLk{1}$. Formally, let $G$ be a graph, and let
$k \geq 2$. In each iteration $i \geq 0$, the algorithm, similarly to the $\WLk{1}$, computes a
\new{coloring} $C^k_i \colon V(G)^k \to \bbN$. In the first iteration, $i=0$, the tuples $\vec{v}$ and $\vec{w}$ in $V(G)^k$ get the same
color if they have the same atomic type, i.e., 
$C^k_{0}(\vec{v}) \coloneqq \mathsf{atp}(\vec{v})$. Now, for $i > 0$, $C^k_{i+1}$ is defined
by
\begin{equation*}\label{ci}
	C^k_{i+1}(\vec{v}) \coloneqq \REL \big(C^k_{i}(\vec{v}), M_i(\vec{v}) \big),
\end{equation*}
with $M_i(\vec{v})$ the multiset 
\begin{equation*}\label{mi}
	M_i(\vec{v}) \coloneqq \oms    (C^k_{i}(\phi_1(\vec{v},w)), \dots,  C^k_{i}(\phi_k(\vec{v},w)))   \mid w \in V(G)   \cms
\end{equation*}
and where
\begin{equation*}
	\phi_j(\vec{v},w)\coloneqq (v_1, \dots, v_{j-1}, w, v_{j+1}, \dots, v_k).
\end{equation*}
That is, $\phi_j(\vec{v},w)$ replaces the $j$-th component of the tuple $\vec{v}$ with the vertex $w$. Hence, two tuples are \new{adjacent} or \new{$j$-neighbors}, with respect to a vertex $w$, if they are different in the $j$th component (or equal, in the case of self-loops). Again, we run the algorithm until convergence, i.e.,
\begin{equation*}
	C^k_{i}(\vec{v}) = C^k_{i}(\vec{w}) \iff C^k_{i+1}(\vec{v}) = C^k_{i+1}(\vec{w}),
\end{equation*}
for all $\vec{v}$ and $\vec{w}$ in $V(G)^k$ holds, and call the partition of $V(G)^k$
induced by $C^k_i$ the stable partition. For such $i$, we define
$C^k_{\infty}(\vec{v}) = C^k_i(\vec{v})$ for $\vec{v}$ in $V(G)^k$. Hence, two tuples $\vec{v}$ and $\vec{w}$ with the same color in iteration $(i-1)$ get different colors in iteration $i$ if there exists a $j$ in $[k]$ such that the number of $j$-neighbors of $\vec{v}$ and $\vec{w}$, respectively, colored with a certain color is different. We set $C_\infty^k(v)\coloneqq C_\infty^k(v,\ldots,v)$ and refer to this as the color of the vertex $v$.

To test whether two graphs $G$ and $H$ are non-isomorphic, we run the $\WLk{k}$ in ``parallel'' on both graphs. Then, if the two graphs have a different number of vertices colored $c$ in $\bbN$, the $\WLk{k}$ \textit{distinguishes} the graphs as non-isomorphic. By increasing $k$, the algorithm becomes more powerful in distinguishing non-isomorphic graphs, i.e., for each $k\geq 1$, there are non-isomorphic graphs distinguished by $\WLk{(k+1)}$ but not by $\WLk{k}$ ~\citep{Cai+1992}. 

\paragraph{The Weisfeiler--Leman hierarchy and permutation-invariant function approximation}\label{connect}
The Weisfeiler--Leman hierarchy is a purely combinatorial algorithm for testing graph isomorphism. However,  the graph isomorphism function, mapping non-isomorphic graphs to different values, is the hardest to approximate permutation-invariant function. Hence, the Weisfeiler--Leman hierarchy has strong ties to GNNs' capabilities to approximate permutation-invariant or equivariant functions over graphs. For example,~\citet{Mor+2019,Xu+2018b} showed that the expressive power of any possible GNN architecture is limited by $\WLk{1}$ in terms of distinguishing non-isomorphic graphs. \citet{Azi+2020} refined these results by showing that if an architecture is capable of simulating $\WLk{k}$ and allows the application of universal neural networks on vertex features, it will be able to approximate any permutation-equivariant function below the expressive power of $\WLk{k}$; see also~\citet{Che+2019}. Hence, if one shows that one architecture distinguishes more graphs than another, it follows that the corresponding GNN can approximate more functions. These results were refined in \citet{geerts2022} for color refinement and taking into account the number of iterations of $\WLk{k}$.

\section{Datasets, details on the experiments, and additional experimental results}\label{app:exp}

In the following, we outline details on the experiments. 

\paragraph{Additional details on the upstream model}
\label{par: add_upstream}
When sampling multiple subgraphs with \imle, they tend to have similar structures. In other words, \imle learns similar distributions in different channels of the neural network. 
This phenomenon is not in our favor as we need to cover the original full graph as much as possible. To mitigate this issue, we propose an auxiliary loss for the diversity of subgraphs. We calculate the KL-divergence between the selected vertex or edge masks and an all-one vector and try to minimize the value. We tune the weight for the auxiliary loss on the log scale, e.g., $0.1$, $1$, $10$, and so on.

\paragraph{Downstream and baseline models}
For the larger molecular regression tasks \textsc{Alchemy}, \textsc{Qm9}, and \textsc{Zinc}, we closely followed the hyperparameters found in~\citet{Che+2019b},~\citet{Gil+2017}, and~\citet{Dwi+2020} respectively, and used \gin layers. 
That is, for \textsc{Zinc}, we used four \gin layers with a hidden dimension of 256 followed by batch norm and a 4-layer MLP for the joint regression of the target after applying
\esan mean pooling. Moreover, we report results on \textsc{Zinc} dataset with \textsc{PNA} model architecture. We mainly follow the configurations of ~\citep{Corso+2020} and the official implementation of ~\citep{Fey+2019}. For the number of hidden dimensions, where we used 128 instead of 75. For \textsc{Alchemy} and \textsc{Qm9}, we used six layers with 64 (hidden) features and a set2seq layer~\citep{Vin+2016} for graph-level pooling, followed by a $2$-layer MLP for the joint regression of the twelve targets. Moreover, following~\citep{Che+2019b,Gil+2017}, we normalized the targets of the training split to zero mean and unit variance and report re-normalized testing scores. We used a single model to predict all targets and report (mean) MAE. For the GNN baseline for the \textsc{Qm9} dataset, we computed edge-wise $\ell_2$ distances based on the 3D coordinates and concatenated them to the edge features. We note here that our intent is not the beat state-of-the-art, physical knowledge-incorporating architectures, e.g., \textsf{DimeNet}~\citep{Kli+2020} or \textsf{Cormorant}~\citep{And+2019}, but to solely show the benefits of data-driven subgraph-enhanced GNNs. Further, to compare to \esan, we used the same architecture as~\citet{Bev+2021}. For the \textsc{Exp} dataset, we processed the raw dataset following~\citet{Abb+2020} and used six \gin layers, each with a hidden dimension of 32, followed by mean pooling and one linear layer immediately after mean pooling. For the \textsc{ogbg-molesol} and \textsc{ogbg-molbace} datasets, we followed \textsc{Ogb}'s~\citep{hu2020ogb} official \gin model architecture without virtual vertices, i.e., five \gin layers each with 300 hidden dimensions, and mean pooling as the final layer. For the \textsc{Proteins} dataset, we followed the DS-GNN setting of \esan paper, i.e., using 32 hidden dimensions, 4 hidden layers, and mean pooling as the last layer.

\paragraph{Sampling subgraphs}
Since the number of unordered $k$-vertex subgraphs is considerably smaller than the number of ordered $k$-vertex subgraphs, we opted to consider unordered $k$-vertex subgraphs; see also~\cref{unordered}. Further, since vertex-subgraph \PMPNNs{k}, see~\cref{sec:vertexpebble}, are easier to implement efficiently and are closer to \esan~\citep{Bev+2021}, we opted to use them for the empirical evaluation. In addition, we used a simple GNN architecture for the upstream model to compute initial features for the subgraphs for ease of implementation. We experimented with selecting and deleting a various number of vertices, edges, and subgraphs induced by $k$-hop neighborhoods ($k$-Ego) for all datasets. We outline the subgraph sampling methods for each dataset below.

For \textsc{Alchemy}, we opted for learning to delete three vertices or edges. We also looked at sampling three subgraphs on five vertices or edges. Finally, we looked at sampling three 3-hop neighborhood subgraphs. For \textsc{Qm9}, we opted for learning to delete one vertex or edge, learning to select three subgraphs with ten vertex or edges, and sampling ten 3-hop neighborhood subgraphs. For the \textsc{Ogbg-Molesol} dataset, we looked at learning to delete one vertex three and ten times, two vertices one time, two subgraphs on five vertices, and one subgraph on ten vertices. Further, we looked at deleting one edge ten times and three 2-hop neighborhood subgraphs. For \textsc{Zinc}, we opted to learn to delete vertices three and ten times. In addition, we investigated deleting two vertices three times. We also investigated learning to delete edges three and ten times. Further, we looked at selecting three subgraphs on 20 vertices or edges. Finally, we looked at sampling three 7-hop neighborhood subgraphs. For the \textsc{Exp} dataset, we learned to delete three vertices or edges. 

\xhdr{Comparison to \esan} To investigate how our 
data-driven sampling approaches compares to state-of-the-art architectures, we integrated \imle-based subgraph sampling into  \esan~\citep{Bev+2021} (\dsgnn)), and compared to \esan using all subgraphs of a given size on the \textsc{Zinc} dataset. In addition, we also compared to a simple random model sampling subgraphs uniformly and at random, using the same configurations as the data-driven ones. To compare computation time between our method and \esan, we measured the time on the test set. The timing consisted of data batch retrieval, subgraph sampling, downstream model forward propagation, upstream model forward propagation (in our method), and loss calculation.  Like \esan, we repeated the inference five times and voted for the majority. 

\paragraph{Additional experimental details }For \textsc{Zinc}, we used the subset of 12\,000 graphs provided in~\citep{Dwi+2020}. For \textsc{Alchemy} and \textsc{Qm9}, we used a subset of 12\,000 graphs sampled uniformly at random, we used the splits provided in~\citep{Morris2020b}. All of the above datasets consists of a training split of 10\,000 graphs, and a validation and test split of 1\,000 graphs, respectively. For the other datasets, we used the officially provided splits. 

We repeated all experiments five times and report average scores and standard deviations. All experiments were conducted on a workstation with one GPU card with 32GB of GPU memory. 

We used two separate instances of an Adam optimizers~\citep{Kin+2015} for the upstream and downstream models, both with default hyper-parameters and no weight decay. For the upstream model, we did not use learning rate decay. For the \textsc{Zinc}, \textsc{Alchemy} and \textsc{Qm9} datasets, we trained for at least 700 epochs and leveraged early stopping with a patience of 100 afterwards. The learning rate for the downstream model decays twice by 0.316 at the 400 and 600 epochs. For the \textsc{Exp} dataset, we trained for 350 epochs with a decay rate of 0.5 every 50 epochs, following the setup of \esan ~\citep{Bev+2021}. For the \textsc{ogbg-molesol} and \textsc{ogbg-molbace} datasets, we trained for 100 epochs, following the default setting of ~\citep{hu2020ogb}. For the \textsc{Proteins} dataset, we trained for 400 epochs, decaying the learning rate by 0.316 twice at 150 and 300 epoch. 

\paragraph{Additional experimental results} In addition to the results already shown in the main paper, we exhibit some additional results for different datasets and training settings. 
\begin{table}[!t]
\centering
    \caption{Additional experimental results on large-scale regression datasets and comparison to non-data-driven \esan~\citep{Bev+2021}.}
    \begin{subtable}{.26\linewidth}
      \centering
         \caption{Results for the \textsc{Qm9} dataset}
    
	\resizebox{1.1\textwidth}{!}{ 	\renewcommand{\arraystretch}{1.0}
		\begin{tabular}{@{}l <{\enspace}@{}ccccr@{}}	\toprule
			 \multicolumn{5}{c}{\textbf{Method}} &  \multirow{1}{*}{\textbf{MAE} $\downarrow$}
			 \\
			 \cmidrule{1-6}
			 \multicolumn{5}{c}{Baseline}   &   $21.92$  \scriptsize	$\pm 4.37$ \\
			\cmidrule{1-6}
			& \textsc{Operat.} & \textsc{Type} & {\textsc{\#}}         &  {\textsc{\# Subg.}}     &    \\	\toprule
			Random & \multirow{2}{*}{{Delete}}                     & \multirow{2}{*}{{Vertex}}     &  \multirow{2}{*}{{1}}     & \multirow{2}{*}{{3}}   &	$ 15.46$ \scriptsize $\pm  1.05$ \\
			\imle &   &    &      &  &	$ 9.14$ \scriptsize $\pm  0.60$ \\
        	\cmidrule{1-6}
        	Random & \multirow{2}{*}{{Delete}}                     & \multirow{2}{*}{{Vertex}}     &  \multirow{2}{*}{{3}}     & \multirow{2}{*}{{3}}   &	$ 22.29$ \scriptsize $\pm  4.07$ \\
			\imle &   &    &      &  &	$ 9.30$ \scriptsize $\pm  0.32$ \\
        	\cmidrule{1-6}
        	Random & \multirow{2}{*}{{Delete}}                     & \multirow{2}{*}{{Vertex}}     &  \multirow{2}{*}{{3}}     & \multirow{2}{*}{{10}}   &	$ 24.81$ \scriptsize $\pm  2.01$ \\
			\imle &   &    &      &  &	$ 12.43$ \scriptsize $\pm  0.12$ \\
        	\cmidrule{1-6}
        	Random & \multirow{2}{*}{{Delete}}                     & \multirow{2}{*}{{Vertex}}     &  \multirow{2}{*}{{5}}     & \multirow{2}{*}{{3}}   &	$ 30.12$ \scriptsize $\pm  1.27$ \\
			\imle &   &    &      &  &	$ 11.35$ \scriptsize $\pm  0.41$ \\
        	\cmidrule{1-6}
			Random & \multirow{2}{*}{{Select}}                     & \multirow{2}{*}{{Vertex}}     &  \multirow{2}{*}{{10}}     & \multirow{2}{*}{{3}}   &	$ 22.69$ \scriptsize $\pm  3.05$ \\
			\imle &   &    &      &  &	$ 11.88$ \scriptsize $\pm  0.52$ \\
        	\cmidrule{1-6}
        	Random & \multirow{2}{*}{{Delete}}                     & \multirow{2}{*}{{Edge}}     &  \multirow{2}{*}{{3}}     & \multirow{2}{*}{{3}}   &	$ 14.85$ \scriptsize $\pm  0.35$ \\
			\imle &   &    &      &  &	$ 9.72$ \scriptsize $\pm  0.23$ \\
			\cmidrule{1-6}
			Random & \multirow{2}{*}{{Delete}}                     & \multirow{2}{*}{{Edge}}     &  \multirow{2}{*}{{5}}     & \multirow{2}{*}{{3}}   &	$ 13.69$ \scriptsize $\pm  0.28$ \\
			\imle &   &    &      &  &	$ 10.08$ \scriptsize $\pm  0.36$ \\
			\cmidrule{1-6}
			Random & \multirow{2}{*}{{Select}}                     & \multirow{2}{*}{{Edge}}     &  \multirow{2}{*}{{10}}     & \multirow{2}{*}{{3}}   &	$ 14.02$ \scriptsize $\pm  0.99$ \\
			\imle &   &    &      &  &	$ 11.58$ \scriptsize $\pm  0.46$ \\
			\cmidrule{1-6}
			Random & \multirow{2}{*}{{Delete}}                     & \multirow{2}{*}{{1-Ego}}     &  \multirow{2}{*}{{--}}     & \multirow{2}{*}{{3}}   &	$ 22.20$ \scriptsize $\pm  3.01$ \\
			\imle &   &    &      &  &	$ 21.19$ \scriptsize $\pm  1.38$ \\
			\cmidrule{1-6}
			Random & \multirow{2}{*}{{Select}}                     & \multirow{2}{*}{{3-Ego}}     &  \multirow{2}{*}{{--}}     & \multirow{2}{*}{{5}}   &	$ 64.76$ \scriptsize $\pm  5.74$ \\
			\imle &   &    &      &  &	$ 27.28$ \scriptsize $\pm  5.30$ \\
			\cmidrule{1-6}
			Random & \multirow{2}{*}{{Select}}                     & \multirow{2}{*}{{3-Ego}}     &  \multirow{2}{*}{{--}}     & \multirow{2}{*}{{10}}   &	$ 19.64$ \scriptsize $\pm  1.38$ \\
			\imle &   &    &      &  &	$ 14.93$ \scriptsize $\pm  0.83$ \\
			\cmidrule{1-6}
			Random & \multirow{2}{*}{{Select}}                     & \multirow{2}{*}{{5-Ego}}     &  \multirow{2}{*}{{--}}     & \multirow{2}{*}{{3}}   &	$ 39.67$ \scriptsize $\pm  0.22$ \\
			\imle &   &    &      &  &	$ 34.98$ \scriptsize $\pm  1.52$ \\
			\bottomrule
		\end{tabular}
	}

    \label{tab:qm9}
    \end{subtable}%
    \begin{subtable}{.50\linewidth}
      \centering
        \caption{I-MLE with \esan on the \textsc{Zinc} dataset.}
    
	\resizebox{0.8\textwidth}{!}{ 	\renewcommand{\arraystretch}{1.0}
		\begin{tabular}{@{}l <{\enspace}@{}cccccr@{}}	\toprule
			 \multicolumn{5}{c}{\textbf{Method}} &  \multirow{2}{*}{\textbf{MAE} $\downarrow$} &\multirow{2}{*}{\textbf{Time in s}}
			 \\
			 \cmidrule{1-5}
			& \textsc{Operat.} & \textsc{Type} & {\textsc{\#}}         &  {\textsc{\# Subg.}}     &    \\	\toprule

			\esan &  \multirow{3}{*}{{Delete}}  &  \multirow{3}{*}{{Vertex}}  & \multicolumn{2}{c}{All Vertexs} &	$ 0.171$ \scriptsize $\pm  0.010$  &  $ 11.86$ \scriptsize $\pm 0.110$\\
	
			\imle &                   &   &  \multirow{2}{*}{{1}}     & \multirow{2}{*}{{2}}   &	$ 0.177$ \scriptsize $\pm  0.016$ & $ 3.449$ \scriptsize $\pm 0.082$\\
			Random & & & & & $0.214$ \scriptsize $\pm  0.007$ & $2.910$ \scriptsize $\pm 0.071$ \\
        	\cmidrule{1-7}

				\esan & \multirow{3}{*}{{Delete}}                   & \multirow{3}{*}{{Edge}}   & \multicolumn{2}{c}{All edges}  &	$ 0.172$ \scriptsize $\pm  0.008$ & $12.260$ \scriptsize $\pm 0.120$ \\
			\imle &   &      &  \multirow{2}{*}{{1}}     & \multirow{2}{*}{{3}} &	$0.222$ \scriptsize $\pm  0.003$ &$3.425$ \scriptsize $\pm  0.070$ \\
			Random & & & & & $0.214$ \scriptsize $\pm 0.008$ & $2.842$ \scriptsize $\pm 0.063$ \\
        	\cmidrule{1-7}
        	
			\esan & \multirow{3}{*}{{Delete}}                     & \multirow{3}{*}{{Edge}}       & \multicolumn{2}{c}{All edges}  &	-- & --\\
			\imle &   &    &   \multirow{2}{*}{{2}}     & \multirow{2}{*}{{3}} &	$0.171$ \scriptsize $\pm  0.009$ &	$4.538$ \scriptsize $\pm  0.091$ \\
			Random & & & & & -- & -- \\
			\cmidrule{1-7}

				\esan & \multirow{3}{*}{{Select}}                     & \multirow{3}{*}{{3-Ego}}     & \multicolumn{2}{c}{All 3-ego nets} &	$ 0.126$ \scriptsize $\pm  0.006$ & $6.825$ \scriptsize $\pm 0.021$\\
			\imle &   &    & \multirow{2}{*}{{--}}     & \multirow{2}{*}{{10}}  &	$ 0.181$ \scriptsize $\pm  0.010$ & $3.907$ \scriptsize $\pm 0.015$\\
			Random & & & & & $0.188$ \scriptsize $\pm 0.004$ & $4.502$ \scriptsize $\pm 0.043$ \\
			\bottomrule
		\end{tabular}
	}
    \label{tab:z_esan}
    \vspace{80pt}
    \end{subtable} 
\end{table}
See~\cref{tab:protein} for results on \textsc{Proteins}. We sampled by deleting one and three vertices three times, and deleting three edges three times. 
\begin{table}[!htb]
    \centering
    \caption{Results for the \textsc{Proteins} dataset.}
    	\resizebox{0.4\textwidth}{!}{ 
\begin{tabular}{@{}l <{\enspace}@{}ccccc@{}}	\toprule
			 \multicolumn{5}{c}{\textbf{Method}} &  \multirow{1}{*}{\textbf{ROCAUC} $\uparrow$}
			 \\
			 \cmidrule{1-6}
			 \multicolumn{5}{c}{Baseline}   &   $0.775$  \scriptsize	$\pm 0.034$ \\
			 \multicolumn{5}{c}{GIN ~\citep{Xu+2018b}}   &   $0.762$  \scriptsize	$\pm 0.028$ \\
			 \multicolumn{5}{c}{GIN + ID-GNN ~\citep{You+2021}}   &   $0.754$  \scriptsize	$\pm 0.027$ \\
			 \multicolumn{5}{c}{DropEdge ~\citep{Ron+2020}}   &   $0.735$  \scriptsize	$\pm 0.045$ \\
			 \multicolumn{5}{c}{PPGN ~\citep{Mar+2019}}   &   $0.772$  \scriptsize	$\pm 0.047$ \\
			 \multicolumn{5}{c}{CIN ~\citep{Bod+2021b}}   &   $0.770$  \scriptsize	$\pm 0.043$ \\
			\cmidrule{1-6}
			& \textsc{Operat.} & \textsc{Type} & {\textsc{\#}}         &  {\textsc{\# Subg.}}     &    \\	\toprule
       
			Random & \multirow{2}{*}{{Delete}}                     & \multirow{2}{*}{{Vertex}}     &  \multirow{2}{*}{{1}}     & \multirow{2}{*}{{3}}   &	$ 0.760$ \scriptsize $\pm  0.011$ \\
			\imle &   &    &      &  &	$ 0.775$ \scriptsize $\pm  0.014$ \\
        	\cmidrule{1-6}
        	Random & \multirow{2}{*}{{Delete}}                     & \multirow{2}{*}{{Vertex}}     &  \multirow{2}{*}{{3}}     & \multirow{2}{*}{{3}}   &	$ 0.769$ \scriptsize $\pm  0.019$ \\
			\imle &   &    &      &  &	$ 0.783$ \scriptsize $\pm  0.012$ \\
        	\cmidrule{1-6}
        	
        	Random & \multirow{2}{*}{{Delete}}                     & \multirow{2}{*}{{Edge}}     &  \multirow{2}{*}{{3}}     & \multirow{2}{*}{{3}}   &	$0.764$ \scriptsize $0.024$ \\
			\imle &   &    &      &  &	$0.780$ \scriptsize $0.013$ \\
        	\bottomrule

		\end{tabular}}
    \label{tab:protein}
\end{table}
See~\cref{tab:exp} for results on \textsc{EXP}. We examine the accuracy by selecting three subgraphs with one node or edge deletion.
\begin{table}[!htb]
    \centering
    \caption{Results for the \textsc{Exp} dataset.}
    
	\resizebox{0.4\textwidth}{!}{ 
\begin{tabular}{@{}l <{\enspace}@{}ccccr@{}}	\toprule
			 \multicolumn{5}{c}{\textbf{Method}} &  \multirow{1}{*}{\textbf{Accuracy} $\uparrow$}
			 \\
			 \cmidrule{1-6}
			 \multicolumn{5}{c}{Baseline}   &   $0.522$  \scriptsize	$\pm 0.003$ \\
			 \multicolumn{5}{c}{GIN ~\citep{Xu+2018b}}   &   $0.511$  \scriptsize	$\pm 0.021$ \\
			 \multicolumn{5}{c}{GIN + ID-GNN ~\citep{You+2021}}   &   $1.000$  \scriptsize	$\pm 0.000$ \\
			\cmidrule{1-6}
			& \textsc{Operat.} & \textsc{Type} & {\textsc{\#}}         &  {\textsc{\# Subg.}}     &    \\	\toprule

			Random & \multirow{2}{*}{{Delete}}                     & \multirow{2}{*}{{Vertex}}     &  \multirow{2}{*}{{1}}     & \multirow{2}{*}{{3}}   &	$0.943$ \scriptsize $\pm 0.002$ \\
			\imle &   &    &      &  &	$ 1.000$ \scriptsize $\pm  0.000$ \\
        	\cmidrule{1-6}
			Random & \multirow{2}{*}{{Delete}}                     & \multirow{2}{*}{{Edge}}     &  \multirow{2}{*}{{1}}     & \multirow{2}{*}{{3}}   &	$0.946$ \scriptsize $\pm 0.002$ \\
			\imle &   &    &      &  &	$0.999$ \scriptsize $\pm 0.001$ \\
			\bottomrule
		\end{tabular}}
    \label{tab:exp}
\end{table}
See~\cref{t_zc} for results on \textsc{Zinc} using a GIN model. See~\cref{tab:zinc_pna} for results using PNA.
\begin{table}[!htb]
    \centering
    
    \caption{Results for the \textsc{Zinc} dataset with GIN model.}
    	\resizebox{0.4\textwidth}{!}{ 
\begin{tabular}{@{}l <{\enspace}@{}ccccc@{}}	\toprule
			 \multicolumn{5}{c}{\textbf{Method}} &  \multirow{1}{*}{\textbf{MAE} $\downarrow$}
			 \\
			 \cmidrule{1-6}
			 \multicolumn{5}{c}{PNA ~\citep{Cor+2020}}   &   $0.188$  \scriptsize	$\pm 0.004$ \\
			 \cmidrule{1-6}
			 \multicolumn{5}{c}{Baseline}   &   $0.207$  \scriptsize	$\pm 0.006$ \\
			 \multicolumn{5}{c}{PNA~\citep{Corso+2020}}   &   $0.188$  \scriptsize	$\pm 0.004$ \\
			 \multicolumn{5}{c}{DGN~\citep{beaini2020directional}}   &   $0.168$  \scriptsize	$\pm 0.003$ \\
			 \multicolumn{5}{c}{GIN~\citep{Xu+2018b}}   &   $0.252$  \scriptsize	$\pm 0.017$ \\
			 \multicolumn{5}{c}{HIMP~\citep{fey2020hierarchical}}   &   $0.151$  \scriptsize	$\pm 0.006$ \\
			 \multicolumn{5}{c}{GNS~\citep{bouritsas2022improving}}   &   $0.108$  \scriptsize	$\pm 0.018$ \\
			 \multicolumn{5}{c}{CIN~\citep{Bod+2021b}}   &   $0.094$  \scriptsize	$\pm 0.004$ \\
			\cmidrule{1-6}
			& \textsc{Operat.} & \textsc{Type} & {\textsc{\#}}         &  {\textsc{\# Subg.}}     &    \\	\toprule

			Random & \multirow{2}{*}{{Delete}}                     & \multirow{2}{*}{{Vertex}}     &  \multirow{2}{*}{{1}}     & \multirow{2}{*}{{1}}   &	$ 0.378$ \scriptsize $\pm  0.004$ \\
			\imle &   &    &      &  &	$ 0.287$ \scriptsize $\pm  0.015$ \\
        	\cmidrule{1-6}
			Random & \multirow{2}{*}{{Delete}}                     & \multirow{2}{*}{{Vertex}}     &  \multirow{2}{*}{{1}}     & \multirow{2}{*}{{3}}   &	$ 0.283$ \scriptsize $\pm  0.003$ \\
			\imle &   &    &      &  &	$ 0.194$ \scriptsize $\pm  0.007$ \\
        	\cmidrule{1-6}
        	Random & \multirow{2}{*}{{Delete}}                     & \multirow{2}{*}{{Vertex}}     &  \multirow{2}{*}{{1}}     & \multirow{2}{*}{{10}}   &	$ 0.234$ \scriptsize $\pm  0.005$ \\
			\imle &   &    &      &  &	$ 0.217$ \scriptsize $\pm  0.003$ \\
        	\cmidrule{1-6}
        	
        	Random & \multirow{2}{*}{{Delete}}                     & \multirow{2}{*}{{Vertex}}     &  \multirow{2}{*}{{3}}     & \multirow{2}{*}{{3}}   &	$ 0.265$ \scriptsize $\pm  0.003$ \\
			\imle &   &    &      &  &	$ 0.184$ \scriptsize $\pm  0.006$ \\
        	\cmidrule{1-6}
        	Random & \multirow{2}{*}{{Delete}}                     & \multirow{2}{*}{{Vertex}}     &  \multirow{2}{*}{{3}}     & \multirow{2}{*}{{10}}   &	$ 0.275$ \scriptsize $\pm  0.010$ \\
			\imle &   &    &      &  &	$ 0.240$ \scriptsize $\pm  0.003$ \\
        	\cmidrule{1-6}
        	
        	Random & \multirow{2}{*}{{Delete}}                     & \multirow{2}{*}{{Vertex}}     &  \multirow{2}{*}{{10}}     & \multirow{2}{*}{{10}}   &	$ 0.210$ \scriptsize $\pm  0.006$ \\
			\imle &   &    &      &  &	$ 0.204$ \scriptsize $\pm  0.004$ \\
        	\cmidrule{1-6}
        	
			Random & \multirow{2}{*}{{Delete}}                     & \multirow{2}{*}{{Edge}}     &  \multirow{2}{*}{{3}}     & \multirow{2}{*}{{1}}   &	$ 0.382$ \scriptsize $\pm  0.004$ \\
			\imle &   &    &      &  &	$ 0.325$ \scriptsize $\pm  0.019$ \\
        	\cmidrule{1-6}
        	Random & \multirow{2}{*}{{Delete}}                     & \multirow{2}{*}{{Edge}}     &  \multirow{2}{*}{{3}}     & \multirow{2}{*}{{3}}   &	$ 0.192$ \scriptsize $\pm  0.002$ \\
			\imle &   &    &      &  &	$ 0.176$ \scriptsize $\pm  0.006$ \\
        	\cmidrule{1-6}
        	Random & \multirow{2}{*}{{Delete}}                     & \multirow{2}{*}{{Edge}}     &  \multirow{2}{*}{{3}}     & \multirow{2}{*}{{10}}   &	$ 0.187$ \scriptsize $\pm  0.002$ \\
			\imle &   &    &      &  &	$ 0.180$ \scriptsize $\pm  0.006$ \\
        	\cmidrule{1-6}
        	Random & \multirow{2}{*}{{Delete}}                     & \multirow{2}{*}{{Edge}}     &  \multirow{2}{*}{{10}}     & \multirow{2}{*}{{3}}   &	$ 0.173$ \scriptsize $\pm  0.007$ \\
			\imle &   &    &      &  &	$ 0.162$ \scriptsize $\pm  0.002$ \\
        	\cmidrule{1-6}
        	Random & \multirow{2}{*}{{Delete}}                     & \multirow{2}{*}{{Edge}}     &  \multirow{2}{*}{{10}}     & \multirow{2}{*}{{10}}   &	$ 0.169$ \scriptsize $\pm  0.013$ \\
			\imle &   &    &      &  &	$ 0.155$ \scriptsize $\pm  0.004$ \\
        	\cmidrule{1-6}
        	
			Random & \multirow{2}{*}{{Select}}                     & \multirow{2}{*}{{Vertex}}     &  \multirow{2}{*}{{20}}     & \multirow{2}{*}{{3}}   &	$ 0.384$ \scriptsize $\pm  0.011$ \\
			\imle &   &    &      &  &	$ 0.313$ \scriptsize $\pm  0.016$ \\
			\cmidrule{1-6}
			Random & \multirow{2}{*}{{Select}}                     & \multirow{2}{*}{{Edge}}     &  \multirow{2}{*}{{20}}     & \multirow{2}{*}{{3}}   &	$ 0.274$ \scriptsize $\pm  0.012$ \\
			\imle &   &    &      &  &	$ 0.261$ \scriptsize $\pm  0.014$ \\
			\cmidrule{1-6}
			
			Random & \multirow{2}{*}{{Delete}}                     & \multirow{2}{*}{{$1$-Ego}}     &  \multirow{2}{*}{{--}}     & \multirow{2}{*}{{3}}   &	$ 0.330$ \scriptsize $\pm  0.002$ \\
			\imle &   &    &      &  &	$ 0.208$ \scriptsize $\pm  0.010$ \\
			\cmidrule{1-6}
			Random & \multirow{2}{*}{{Delete}}                     & \multirow{2}{*}{{$1$-Ego}}     &  \multirow{2}{*}{{--}}     & \multirow{2}{*}{{10}}   &	$ 0.285$ \scriptsize $\pm  0.006$ \\
			\imle &   &    &      &  &	$ 0.260$ \scriptsize $\pm  0.041$ \\
			\cmidrule{1-6}
			
			Random & \multirow{2}{*}{{Select}}                     & \multirow{2}{*}{{$7$-Ego}}     &  \multirow{2}{*}{{--}}     & \multirow{2}{*}{{3}}   &	$ 0.464$ \scriptsize $\pm  0.023$ \\
			\imle &   &    &      &  &	$ 0.257$ \scriptsize $\pm  0.004$ \\
			\bottomrule
		\end{tabular}}
    \label{t_zc}
\end{table}
\begin{table}[!htb]
    \centering
    \caption{Results for the \textsc{Zinc} dataset with PNA model.}
    \resizebox{0.4\textwidth}{!}{ 
\begin{tabular}{@{}l <{\enspace}@{}ccccc@{}}	\toprule
			 \multicolumn{5}{c}{\textbf{Method}} &  \multirow{1}{*}{\textbf{MAE} $\downarrow$}
			 \\
			 \cmidrule{1-6}
			 \multicolumn{5}{c}{PNA ~\citep{Cor+2020}}   &   $0.188$  \scriptsize	$\pm 0.004$ \\
			 \multicolumn{5}{c}{GIN ~\citep{Xu+2018b}}   &   $0.252$  \scriptsize	$\pm 0.017$ \\
			 \multicolumn{5}{c}{DGN ~\citep{beaini2020directional}}   &   $0.168$  \scriptsize	$\pm 0.003$ \\
			\cmidrule{1-6}
			\multicolumn{5}{c}{Baseline}   &   $0.174$  \scriptsize	$\pm 0.003$ \\
			\cmidrule{1-6}
			& \textsc{Operat.} & \textsc{Type} & {\textsc{\#}}         &  {\textsc{\# Subg.}}     &    \\	\toprule
       
			Random & \multirow{2}{*}{{Delete}}                     & \multirow{2}{*}{{Vertex}}     &  \multirow{2}{*}{{1}}     & \multirow{2}{*}{{3}}   &	$ 0.260$ \scriptsize $\pm  0.001$ \\
			\imle &   &    &      &  &	$ 0.168$ \scriptsize $\pm  0.005$ \\
        	\cmidrule{1-6}
        	Random & \multirow{2}{*}{{Delete}}                     & \multirow{2}{*}{{Vertex}}     &  \multirow{2}{*}{{1}}     & \multirow{2}{*}{{10}}   &	$ 0.227$ \scriptsize $\pm  0.004$ \\
			\imle &   &    &      &  &	$ 0.154$ \scriptsize $\pm  0.008$ \\
        	\cmidrule{1-6}
        	Random & \multirow{2}{*}{{Delete}}                     & \multirow{2}{*}{{Vertex}}     &  \multirow{2}{*}{{3}}     & \multirow{2}{*}{{3}}   &	$ 0.226$ \scriptsize $\pm  0.007$ \\
			\imle &   &    &      &  &	$ 0.172$ \scriptsize $\pm  0.001$ \\
        	\cmidrule{1-6}
        	Random & \multirow{2}{*}{{Delete}}                     & \multirow{2}{*}{{Vertex}}     &  \multirow{2}{*}{{3}}     & \multirow{2}{*}{{10}}   &	$ 0.255$ \scriptsize $\pm  0.004$ \\
			\imle &   &    &      &  &	$ 0.164$ \scriptsize $\pm  0.001$ \\
        	\cmidrule{1-6}
        	
        	Random & \multirow{2}{*}{{Delete}}                     & \multirow{2}{*}{{Edge}}     &  \multirow{2}{*}{{3}}     & \multirow{2}{*}{{3}}   &	$ 0.180$ \scriptsize $\pm  0.007$ \\
			\imle &   &    &      &  &	$ 0.159$ \scriptsize $\pm  0.008$ \\
        	\cmidrule{1-6}
        	Random & \multirow{2}{*}{{Delete}}                     & \multirow{2}{*}{{Edge}}     &  \multirow{2}{*}{{10}}     & \multirow{2}{*}{{3}}   &	$ 0.174$ \scriptsize $\pm  0.009$ \\
			\imle &   &    &      &  &	$ 0.161$ \scriptsize $\pm  0.003$ \\
        	\cmidrule{1-6}
			
			Random & \multirow{2}{*}{{Delete}}                     & \multirow{2}{*}{{$1$-Ego}}     &  \multirow{2}{*}{{--}}     & \multirow{2}{*}{{3}}   &	$ 0.325$ \scriptsize $\pm  0.001$ \\
			\imle &   &    &      &  &	$ 0.167$ \scriptsize $\pm  0.005$ \\
			\bottomrule

		\end{tabular}}
    \label{tab:zinc_pna}
\end{table}
See \cref{tab:molbace} for results on \textsc{ogbg-molbace}.
\begin{table}[!htb]
    \centering
    \caption{Results for the \textsc{ogbg-molbace} dataset.}
    
	\resizebox{0.4\textwidth}{!}{ 
\begin{tabular}{@{}l <{\enspace}@{}ccccc@{}}	\toprule
			 \multicolumn{5}{c}{\textbf{Method}} &  \multirow{1}{*}{\textbf{ROCAUC} $\uparrow$}
			 \\
			 \cmidrule{1-6}
			 \multicolumn{5}{c}{Baseline}   &   $0.714$  \scriptsize	$\pm 0.058$ \\
			\cmidrule{1-6}
			& \textsc{Operat.} & \textsc{Type} & {\textsc{\#}}         &  {\textsc{\# Subg.}}     &    \\	\toprule

			Random & \multirow{2}{*}{{Delete}}                     & \multirow{2}{*}{{Vertex}}     &  \multirow{2}{*}{{1}}     & \multirow{2}{*}{{10}}   &	$ 0.719$ \scriptsize $\pm  0.039$ \\
			\imle &   &    &      &  &	$ 0.723$ \scriptsize $\pm  0.066$ \\
        	\cmidrule{1-6}
        	Random & \multirow{2}{*}{{Delete}}                     & \multirow{2}{*}{{Vertex}}     &  \multirow{2}{*}{{3}}     & \multirow{2}{*}{{3}}   &	$ 0.742$ \scriptsize $\pm  0.025$ \\
			\imle &   &    &      &  &	$ 0.771$ \scriptsize $\pm  0.038$ \\
        	\cmidrule{1-6}
        	Random & \multirow{2}{*}{{Delete}}                     & \multirow{2}{*}{{Vertex}}     &  \multirow{2}{*}{{3}}     & \multirow{2}{*}{{5}}   &	$ 0.730$ \scriptsize $\pm  0.026$ \\
			\imle &   &    &      &  &	$ 0.763$ \scriptsize $\pm  0.030$ \\
        	\cmidrule{1-6}
        	Random & \multirow{2}{*}{{Delete}}                     & \multirow{2}{*}{{Vertex}}     &  \multirow{2}{*}{{3}}     & \multirow{2}{*}{{10}}   &	$ 0.716$ \scriptsize $\pm  0.032$ \\
			\imle &   &    &      &  &	$ 0.757$ \scriptsize $\pm  0.019$ \\
        	\cmidrule{1-6}
        	Random & \multirow{2}{*}{{Delete}}                     & \multirow{2}{*}{{Vertex}}     &  \multirow{2}{*}{{10}}     & \multirow{2}{*}{{3}}   &	$ 0.761$ \scriptsize $\pm  0.026$ \\
			\imle &   &    &      &  &	$ 0.791$ \scriptsize $\pm  0.008$ \\
        	\cmidrule{1-6}
        	
        	Random & \multirow{2}{*}{{Delete}}                     & \multirow{2}{*}{{Edge}}     &  \multirow{2}{*}{{1}}     & \multirow{2}{*}{{3}}   &	$ 0.724$ \scriptsize $\pm  0.056$ \\
			\imle &   &    &      &  &	$ 0.735$ \scriptsize $\pm  0.046$ \\
        	\cmidrule{1-6}
        	Random & \multirow{2}{*}{{Delete}}                     & \multirow{2}{*}{{Edge}}     &  \multirow{2}{*}{{5}}     & \multirow{2}{*}{{3}}   &	$ 0.732$ \scriptsize $\pm  0.026$ \\
			\imle &   &    &      &  &	$ 0.756$ \scriptsize $\pm  0.041$ \\
        	\cmidrule{1-6}
        	Random & \multirow{2}{*}{{Delete}}                     & \multirow{2}{*}{{Edge}}     &  \multirow{2}{*}{{10}}     & \multirow{2}{*}{{3}}   &	$ 0.772$ \scriptsize $\pm  0.028$ \\
			\imle &   &    &      &  &	$ 0.777$ \scriptsize $\pm  0.024$ \\
        	\cmidrule{1-6}
        	Random & \multirow{2}{*}{{Delete}}                     & \multirow{2}{*}{{Edge}}     &  \multirow{2}{*}{{10}}     & \multirow{2}{*}{{10}}   &	$ 0.754$ \scriptsize $\pm  0.018$ \\
			\imle &   &    &      &  &	$ 0.784$ \scriptsize $\pm  0.022$ \\
        	\cmidrule{1-6}
        	
			Random & \multirow{2}{*}{{Delete}}                     & \multirow{2}{*}{{$1$-Ego}}     &  \multirow{2}{*}{{--}}     & \multirow{2}{*}{{3}}   &	$ 0.709$ \scriptsize $\pm  0.023$ \\
			\imle &   &    &      &  &	$ 0.757$ \scriptsize $\pm  0.023$ \\
			\cmidrule{1-6}
			Random & \multirow{2}{*}{{Select}}                     & \multirow{2}{*}{{$5$-Ego}}     &  \multirow{2}{*}{{--}}     & \multirow{2}{*}{{3}}   &	$ 0.768$ \scriptsize $\pm  0.039$ \\
			\imle &   &    &      &  &	$ 0.777$ \scriptsize $\pm  0.027$ \\
			\bottomrule
			
		\end{tabular}}
    \label{tab:molbace}
\end{table}
For ease of implementation, we use unordered subgraph aggregation by default. Here, we show results for additional experimental results comparing between ordered and unordered aggregation methods. \cref{tab:molesol_order} and \cref{tab:molbace_order} show results for \textsc{ogbg} datasets, while \cref{tab:zinc_order} and \cref{tab:zinc_pna_order} show results for 
\textsc{Zinc}.
\begin{table}[!htb]
    \centering
    \caption{Results for the \textsc{ogbg-molesol} dataset using ordered and unordered subgraphs.}
    
	\resizebox{0.4\textwidth}{!}{ 
\begin{tabular}{@{}l <{\enspace}@{}ccccc@{}}	\toprule
			 \multicolumn{5}{c}{\textbf{Method}} &  \multirow{1}{*}{\textbf{RSMSE} $\downarrow$}
			 \\
			 \cmidrule{1-6}
			 \multicolumn{5}{c}{Baseline}   &   $1.193$  \scriptsize	$\pm 0.083$ \\
			\cmidrule{1-6}
			& \textsc{Operat.} & \textsc{Type} & {\textsc{\#}}         &  {\textsc{\# Subg.}}     &    \\	\toprule

			Random & \multirow{3}{*}{{Delete}}                     & \multirow{3}{*}{{Vertex}}     &  \multirow{3}{*}{{1}}     & \multirow{3}{*}{{3}}   &	$ 1.215$ \scriptsize $\pm  0.095$ \\
			I-MLE unordered &   &    &      &  &	$ 1.053$ \scriptsize $\pm  0.080$ \\
			I-MLE ordered & & & &   &  $ 0.835$ \scriptsize $\pm  0.079$ \\
        	\cmidrule{1-6}
        	Random & \multirow{3}{*}{{Delete}}                     & \multirow{3}{*}{{Vertex}}     &  \multirow{3}{*}{{2}}     & \multirow{3}{*}{{3}}   &	$ 1.132$ \scriptsize $\pm  0.020$ \\
			I-MLE unordered &   &    &      &  &	$ 1.081$ \scriptsize $\pm  0.021$ \\
			I-MLE ordered & & & &   &  $ 0.850$ \scriptsize $\pm  0.106$ \\
        	\cmidrule{1-6}
        	Random & \multirow{3}{*}{{Delete}}                     & \multirow{3}{*}{{Vertex}}     &  \multirow{3}{*}{{5}}     & \multirow{3}{*}{{3}}   &	$ 0.992$ \scriptsize $\pm  0.115$ \\
			I-MLE unordered &   &    &      &  &	$ 1.115$ \scriptsize $\pm  0.076$ \\
			I-MLE ordered & & & &   &  $ 0.853$ \scriptsize $\pm  0.043$ \\
        	\bottomrule
			
		\end{tabular}}
    \label{tab:molesol_order}
\end{table}
\begin{table}[!htb]
    \centering
    \caption{Results for the \textsc{ogbg-molbace} dataset using ordered and unordered subgraphs.}
    
	\resizebox{0.4\textwidth}{!}{ 
\begin{tabular}{@{}l <{\enspace}@{}ccccc@{}}	\toprule
			 \multicolumn{5}{c}{\textbf{Method}} &  \multirow{1}{*}{\textbf{AUCROC} $\uparrow$}
			 \\
			 \cmidrule{1-6}
			 \multicolumn{5}{c}{Baseline}   &   $0.714$  \scriptsize	$\pm 0.058$ \\
			\cmidrule{1-6}
			& \textsc{Operat.} & \textsc{Type} & {\textsc{\#}}         &  {\textsc{\# Subg.}}     &    \\	\toprule

			Random & \multirow{3}{*}{{Delete}}                     & \multirow{3}{*}{{Vertex}}     &  \multirow{3}{*}{{3}}     & \multirow{3}{*}{{3}}   &	$ 0.742$ \scriptsize $\pm  0.025$ \\
			I-MLE unordered &   &    &      &  &	$ 0.771$ \scriptsize $\pm  0.038$ \\
			I-MLE ordered & & & &   &  $ 0.761$ \scriptsize $\pm  0.011$ \\
        	\cmidrule{1-6}
        	Random & \multirow{3}{*}{{Delete}}                     & \multirow{3}{*}{{Vertex}}     &  \multirow{3}{*}{{3}}     & \multirow{3}{*}{{10}}   &	$ 0.716$ \scriptsize $\pm  0.032$ \\
			I-MLE unordered &   &    &      &  &	$ 0.757$ \scriptsize $\pm  0.019$ \\
			I-MLE ordered & & & &   &  $ 0.776$ \scriptsize $\pm  0.032$ \\
        	\bottomrule
			
		\end{tabular}}
    \label{tab:molbace_order}
\end{table}
\begin{table}[!htb]
    \centering
    \caption{Results for the \textsc{Zinc} dataset using ordered and unordered subgraphs.}
    	\resizebox{0.4\textwidth}{!}{ 
\begin{tabular}{@{}l <{\enspace}@{}ccccc@{}}	\toprule
			 \multicolumn{5}{c}{\textbf{Method}} &  \multirow{1}{*}{\textbf{MAE} $\downarrow$}
			 \\
			 \cmidrule{1-6}
			 \multicolumn{5}{c}{Baseline}   &   $0.207$  \scriptsize	$\pm 0.006$ \\
			\cmidrule{1-6}
			& \textsc{Operat.} & \textsc{Type} & {\textsc{\#}}         &  {\textsc{\# Subg.}}     &    \\	\toprule
       
			Random & \multirow{3}{*}{{Delete}}                     & \multirow{3}{*}{{Vertex}}     &  \multirow{3}{*}{{1}}     & \multirow{3}{*}{{3}}   &	$ 0.283$ \scriptsize $\pm  0.003$ \\
			I-MLE unordered &   &    &      &  &	$ 0.194$ \scriptsize $\pm  0.007$ \\
			I-MLE ordered &   &    &      &  &	$ 0.187$ \scriptsize $\pm  0.004$ \\
        	\bottomrule
		\end{tabular}}
    \label{tab:zinc_order}
\end{table}
\begin{table}[!htb]
    \centering
    \caption{Results for the \textsc{Zinc} dataset using ordered and unordered subgraphs with PNA model.}
    	\resizebox{0.4\textwidth}{!}{ 
\begin{tabular}{@{}l <{\enspace}@{}ccccc@{}}	\toprule
			 \multicolumn{5}{c}{\textbf{Method}} &  \multirow{1}{*}{\textbf{MAE} $\downarrow$}
			 \\
			 \cmidrule{1-6}
			& \textsc{Operat.} & \textsc{Type} & {\textsc{\#}}         &  {\textsc{\# Subg.}}     &    \\	\toprule
       
			Random & \multirow{3}{*}{{Delete}}                     & \multirow{3}{*}{{Vertex}}     &  \multirow{3}{*}{{1}}     & \multirow{3}{*}{{3}}   &	$ 0.260$ \scriptsize $\pm  0.001$ \\
			I-MLE unordered &   &    &      &  &	$ 0.168$ \scriptsize $\pm  0.005$ \\
			I-MLE ordered &   &    &      &  &	$ 0.182$ \scriptsize $\pm  0.005$ \\
        	\cmidrule{1-6}
        	Random & \multirow{3}{*}{{Delete}}                     & \multirow{3}{*}{{Vertex}}     &  \multirow{3}{*}{{1}}     & \multirow{3}{*}{{10}}   &	$ 0.227$ \scriptsize $\pm  0.004$ \\
			I-MLE unordered &   &    &      &  &	$ 0.154$ \scriptsize $\pm  0.008$ \\
			I-MLE ordered &   &    &      &  &	$ 0.181$ \scriptsize $\pm  0.010$ \\
			\cmidrule{1-6}
        	Random & \multirow{3}{*}{{Delete}}                     & \multirow{3}{*}{{Vertex}}     &  \multirow{3}{*}{{3}}     & \multirow{3}{*}{{3}}   &	$ 0.226$ \scriptsize $\pm  0.007$ \\
			I-MLE unordered &   &    &      &  &	$ 0.172$ \scriptsize $\pm  0.008$ \\
			I-MLE ordered &   &    &      &  &	$ 0.186$ \scriptsize $\pm  0.003$ \\
			\cmidrule{1-6}
        	Random & \multirow{3}{*}{{Delete}}                     & \multirow{3}{*}{{Vertex}}     &  \multirow{3}{*}{{3}}     & \multirow{3}{*}{{10}}   &	$ 0.255$ \scriptsize $\pm  0.004$ \\
			I-MLE unordered &   &    &      &  &	$ 0.164$ \scriptsize $\pm  0.001$ \\
			I-MLE ordered &   &    &      &  &	$ 0.175$ \scriptsize $\pm  0.008$ \\
        	\bottomrule
		\end{tabular}}
    \label{tab:zinc_pna_order}
\end{table}
To show the auxiliary loss described in \cref{par: add_upstream} makes a difference, we carry out ablations studies using no auxiliary loss; see \autoref{tab:molesol_aux}, \autoref{tab:zinc_aux} for results.

\begin{table}[!htb]
    \centering
    \caption{Results for the \textsc{ogbg-molesol} dataset, auxiliary loss ablation.}
    
	\resizebox{0.4\textwidth}{!}{ 
\begin{tabular}{@{}l <{\enspace}@{}ccccc@{}}	\toprule
			 \multicolumn{5}{c}{\textbf{Method}} &  \multirow{1}{*}{\textbf{RSMSE} $\downarrow$}
			 \\
			 \cmidrule{1-6}
			 \multicolumn{5}{c}{Baseline}   &   $1.193$  \scriptsize	$\pm 0.083$ \\
			\cmidrule{1-6}
			& \textsc{Operat.} & \textsc{Type} & {\textsc{\#}}         &  {\textsc{\# Subg.}}     &    \\	\toprule

			Random & \multirow{3}{*}{{Delete}}                     & \multirow{3}{*}{{Vertex}}     &  \multirow{3}{*}{{1}}     & \multirow{3}{*}{{3}}   &	$ 1.215$ \scriptsize $\pm  0.095$ \\
			\imle &   &    &      &  &	$ 1.053$ \scriptsize $\pm  0.080$ \\
			I-MLE ablation & & & &   &  $ 1.120$ \scriptsize $\pm  0.092$ \\
        	\cmidrule{1-6}
        	Random & \multirow{3}{*}{{Delete}}                     & \multirow{3}{*}{{Vertex}}     &  \multirow{3}{*}{{2}}     & \multirow{3}{*}{{3}}   &	$ 1.132$ \scriptsize $\pm  0.020$ \\
			\imle &   &    &      &  &	$ 1.081$ \scriptsize $\pm  0.021$ \\
			I-MLE ablation & & & &   &  $ 1.137$ \scriptsize $\pm  0.146$ \\
        	\cmidrule{1-6}
        	Random & \multirow{3}{*}{{Delete}}                     & \multirow{3}{*}{{Vertex}}     &  \multirow{3}{*}{{5}}     & \multirow{3}{*}{{3}}   &	$ 0.992$ \scriptsize $\pm  0.115$ \\
			\imle &   &    &      &  &	$ 1.115$ \scriptsize $\pm  0.076$ \\
			I-MLE ablation & & & &   &  $ 1.247$ \scriptsize $\pm  0.126$ \\
        	\bottomrule
			
		\end{tabular}}
    \label{tab:molesol_aux}
\end{table}

\begin{table}[!htb]
    \centering
    \caption{Results for the \textsc{Zinc} dataset, auxiliary loss ablation.}
    	\resizebox{0.4\textwidth}{!}{ 
\begin{tabular}{@{}l <{\enspace}@{}ccccc@{}}	\toprule
			 \multicolumn{5}{c}{\textbf{Method}} &  \multirow{1}{*}{\textbf{MAE} $\downarrow$}
			 \\
			 \cmidrule{1-6}
			 \multicolumn{5}{c}{Baseline}   &   $0.207$  \scriptsize	$\pm 0.006$ \\
			\cmidrule{1-6}
			& \textsc{Operat.} & \textsc{Type} & {\textsc{\#}}         &  {\textsc{\# Subg.}}     &    \\	\toprule
       
			Random & \multirow{3}{*}{{Delete}}                     & \multirow{3}{*}{{Vertex}}     &  \multirow{3}{*}{{1}}     & \multirow{3}{*}{{3}}   &	$ 0.283$ \scriptsize $\pm  0.003$ \\
			\imle &   &    &      &  &	$ 0.194$ \scriptsize $\pm  0.007$ \\
			I-MLE ablation & & & &   &  $ 0.194$ \scriptsize $\pm  0.004$ \\
        	\cmidrule{1-6}
        	Random & \multirow{3}{*}{{Delete}}                     & \multirow{3}{*}{{Vertex}}     &  \multirow{3}{*}{{3}}     & \multirow{3}{*}{{3}}   &	$ 0.265$ \scriptsize $\pm  0.003$ \\
			\imle &   &    &      &  &	$ 0.184$ \scriptsize $\pm  0.006$ \\
			I-MLE ablation & & & &   &  $ 0.184$ \scriptsize $\pm  0.004$ \\
        	\cmidrule{1-6}

        	Random & \multirow{3}{*}{{Delete}}                     & \multirow{3}{*}{{Edge}}     &  \multirow{3}{*}{{3}}     & \multirow{3}{*}{{3}}   &	$ 0.192$ \scriptsize $\pm  0.002$ \\
			\imle &   &    &      &  &	$ 0.176$ \scriptsize $\pm  0.006$ \\
			I-MLE ablation & & & &   &  $ 0.178$ \scriptsize $\pm  0.008$ \\
        	\cmidrule{1-6}
        	Random & \multirow{3}{*}{{Delete}}                     & \multirow{3}{*}{{Edge}}     &  \multirow{3}{*}{{10}}     & \multirow{3}{*}{{10}}   &	$ 0.169$ \scriptsize $\pm  0.013$ \\
			\imle &   &    &      &  &	$ 0.155$ \scriptsize $\pm  0.004$ \\
			I-MLE ablation & & & &   &  $ 0.162$ \scriptsize $\pm  0.001$ \\
        	\bottomrule
        	
		\end{tabular}}
    \label{tab:zinc_aux}
\end{table}

Finally, we designed more sophisticated subgraph selection methods.
For vertices, we solve an Integer Linear Programming problem (ILP). The objective goal is to select subgraphs, maximizing the sum of the corresponding weights, while the constraints are that each vertex in the original graph much be selected at least once. We compare this combinatorial-optimization-based selection method with unordered \imle. For edges, we grow a Maximum Spanning Tree (MST) on each graph. We repeat this several times to get different subgraph instances. We compare this method to MST-based selection strategy using uniformly sampled  edge weights; see \Cref{tab:molesol_method} and \autoref{tab:molbace_method} for results.

\begin{table}[!htb]
    \centering
    \caption{Results for the \textsc{ogbg-molesol} dataset using  different selection methods.}
    
	\resizebox{0.4\textwidth}{!}{ 
\begin{tabular}{@{}l <{\enspace}@{}ccccc@{}}	\toprule
			 \multicolumn{5}{c}{\textbf{Method}} &  \multirow{1}{*}{\textbf{RSMSE} $\downarrow$}
			 \\
			 \cmidrule{1-6}
			 \multicolumn{5}{c}{Baseline}   &   $1.193$  \scriptsize	$\pm 0.083$ \\
			\cmidrule{1-6}
			& \textsc{Operat.} & \textsc{Type} & {\textsc{\#}}         &  {\textsc{\# Subg.}}     &    \\	\toprule

			Random & \multirow{3}{*}{{Delete}}                     & \multirow{3}{*}{{Vertex}}     &  \multirow{3}{*}{{1}}     & \multirow{3}{*}{{3}}   &	$ 1.215$ \scriptsize $\pm  0.095$ \\
			I-MLE unordered &   &    &      &  &	$ 1.053$ \scriptsize $\pm  0.080$ \\
			I-MLE covered & & & &   &  $ 1.074$ \scriptsize $\pm  0.115$ \\
        	\cmidrule{1-6}
        	Random & \multirow{3}{*}{{Delete}}                     & \multirow{3}{*}{{Vertex}}     &  \multirow{3}{*}{{2}}     & \multirow{3}{*}{{3}}   &	$ 1.132$ \scriptsize $\pm  0.020$ \\
			I-MLE unordered &   &    &      &  &	$ 1.081$ \scriptsize $\pm  0.021$ \\
			I-MLE covered & & & &   &  $ 1.081$ \scriptsize $\pm  0.068$ \\
        	\cmidrule{1-6}
        	Random & \multirow{3}{*}{{Delete}}                     & \multirow{3}{*}{{Vertex}}     &  \multirow{3}{*}{{5}}     & \multirow{3}{*}{{3}}   &	$ 0.992$ \scriptsize $\pm  0.115$ \\
			I-MLE unordered &   &    &      &  &	$ 1.115$ \scriptsize $\pm  0.076$ \\
			I-MLE covered & & & &   &  $ 0.946$ \scriptsize $\pm  0.058$ \\
			\cmidrule{1-6}
			Random & \multirow{2}{*}{{MST}}                     & \multirow{2}{*}{{Edge}}     &  \multirow{2}{*}{{--}}     & \multirow{2}{*}{{3}}   &	$ 1.095$ \scriptsize $\pm  0.021$ \\
			\imle &   &    &      &  &	$ 1.070$ \scriptsize $\pm  0.005$ \\
			\bottomrule
			
		\end{tabular}}
    \label{tab:molesol_method}
\end{table}

\begin{table}[!htb]
    \centering
    \caption{Results for the \textsc{ogbg-molbace} dataset using  different selection methods.}
    
	\resizebox{0.4\textwidth}{!}{ 
\begin{tabular}{@{}l <{\enspace}@{}ccccc@{}}	\toprule
			 \multicolumn{5}{c}{\textbf{Method}} &  \multirow{1}{*}{\textbf{AUCROC} $\uparrow$}
			 \\
			 \cmidrule{1-6}
			 \multicolumn{5}{c}{Baseline}   &   $0.714$  \scriptsize	$\pm 0.058$ \\
			\cmidrule{1-6}
			& \textsc{Operat.} & \textsc{Type} & {\textsc{\#}}         &  {\textsc{\# Subg.}}     &    \\	\toprule

			Random & \multirow{3}{*}{{Delete}}                     & \multirow{3}{*}{{Vertex}}     &  \multirow{3}{*}{{3}}     & \multirow{3}{*}{{3}}   &	$ 0.742$ \scriptsize $\pm  0.025$ \\
			I-MLE unordered &   &    &      &  &	$ 0.771$ \scriptsize $\pm  0.038$ \\
			I-MLE covered & & & &   &  $ 0.765$ \scriptsize $\pm  0.032$ \\
        	\cmidrule{1-6}
        	Random & \multirow{2}{*}{{MST}}                     & \multirow{2}{*}{{Edge}}     &  \multirow{2}{*}{{--}}     & \multirow{2}{*}{{3}}   &	$ 0.740$ \scriptsize $\pm  0.034$ \\
			\imle &   &    &      &  &	$ 0.758$ \scriptsize $\pm  0.025$ \\
			\cmidrule{1-6}
			Random & \multirow{2}{*}{{MST}}                     & \multirow{2}{*}{{Edge}}     &  \multirow{2}{*}{{--}}     & \multirow{2}{*}{{10}}   &	$ 0.741$ \scriptsize $\pm  0.025$ \\
			\imle &   &    &      &  &	$ 0.763$ \scriptsize $\pm  0.027$ \\
			\bottomrule
			
		\end{tabular}}
    \label{tab:molbace_method}
\end{table}

\begin{table}[t!]
	\begin{center}
		\caption{Dataset statistics and properties for graph-level prediction tasks, $^\dagger$---Continuous vertex labels following~\cite{Gil+2017}, the last three components encode 3D coordinates.}
		\resizebox{.8\textwidth}{!}{ 	\renewcommand{\arraystretch}{1.0}
			\begin{tabular}{@{}lcccccc@{}}\toprule
				\multirow{3}{*}{\vspace*{4pt}\textbf{Dataset}}&\multicolumn{6}{c}{\textbf{Properties}}\\
				\cmidrule{2-7}
				                         & Number of  graphs & Number of classes/targets & $\varnothing$ Number of vertices & $\varnothing$ Number of edges & Vertex labels              & Edge labels \\ \midrule
				$\textsc{Alchemy}$       & 202\,579          & 12                        & 10.1                          & 10.4                          & \cmark                   & \cmark      \\
				$\textsc{Qm9}$           & 129\,433          & 12                        & 18.0                          & 18.6                          & \cmark (13+3D)$^\dagger$ & \cmark (4)  \\
			
					$\textsc{Zinc}$       &  249\,456 &1	&23.1 &	24.9     & \cmark  & \cmark           \\
					$\textsc{Exp}$       &   1\,200 &	2 & 44.5& 55.2 &  \cmark & \xmark           \\
					$\textsc{ogbg-molesol}$       &   1\,128 &	1 & 13.3& 13.7 &  \cmark & \cmark           \\
					$\textsc{ogbg-molbace}$       &    1\,513 & 2 & 34.1& 36.9 &  \cmark & \cmark           \\
					$\textsc{Proteins}$       &   1\,113 &	2 & 39.1 & 72.8 &  \cmark & \xmark            \\

				\bottomrule
			\end{tabular}}
		\label{ds}
	\end{center}
\end{table}

\section{$\PWLk{k}$ and \PMPNN{k}: Omitted proofs}
In the following, we outline the proofs from the main paper.

\subsection{Equivalence $\PWLk{k}$ and \PMPNN{k}}

Let $\mathcal{G}$ be the set of all vertex-labeled graphs and $F$ be a set of permutation-invariant functions over $\mathcal{G}$, e.g., the functions expressible by some GNN architecture. Then, following~\cite{Azi+2020}, we define an equivalence relation $\rho$ where for graphs $G$ and $H$ in $\mathcal{G}$
\begin{align*}
	(G,H) \in \rho(F) \iff \text{for all } f \in F, f(G) = f(H)
\end{align*}
holds. When $F$ is replaced by an architecture's name, we mean the set of function expressible with that architecture.

\begin{proposition}[Proposition 1 in the main text] For all $k \geq 1$, it holds that 
	\begin{align*}
		\rho(\text{\PMPNNs{k}}) =  \rho(\PWLk{k}).
	\end{align*}
\end{proposition}
To show the above result, we show the inclusions 
$\rho(\text{\PMPNNs{k}})\subseteq\rho(\PWLk{k})$ and $\rho(\PWLk{k})\subseteq \rho(\text{\PMPNNs{k}})$ in \cref{lem:pmpnn_pwl} and \cref{wlinmp}, respectively.
\begin{lemma}\label{lem:pmpnn_pwl}
For all $k \geq 1$, it holds that  
\begin{linenomath}
 	\postdisplaypenalty=0
\begin{align*}
	 \rho(\PWLk{k}) \subseteq  \rho(\text{\PMPNNs{k}}).
\end{align*}
 \end{linenomath}
\end{lemma}
\begin{proof}
We show that if $\PWLk{k}$ does not distinguish two vertices $v$ and $w$ in a graph $G$, then any \PMPNN{k} will also not distinguish them. That is, any \PMPNN{k} will compute the same feature for the two vertices, which implies the result.

Let us make precise what we will show. 
Let $v$ and $w \in V(G)$ such that $C(v)=C(w)$. We recall that 
$C(v) \coloneqq  \REL \bigl(\oms C_{\infty}(v,\pb) \mid \pb\in G_k \cms \bigr)$
and 
$C(w) \coloneqq  \REL \bigl(\oms C_{\infty}(w,\pb) \mid \pb\in G_k \cms \bigr)$. For $C(v)=C(w)$ to hold, we therefore need that 
 \begin{linenomath}
 	\postdisplaypenalty=0
\begin{align}\label{equalcolor}
	C_{i}(v)\coloneqq\oms C_{i}(v,\pb) \mid \pb\in G_k  \cms=	\oms C_{i}(w,\pb) \mid \pb\in G_k)  \cms =:C_{i}(w)
\end{align}
 \end{linenomath}
for all iterations $i$ of the $\PWLk{k}$.

We next turn to \PMPNNs{k}. Let us denote by $\hb_{v,\pb}^\tup{i}$ the vertex feature of $v$ for $\pb \in G_k$ computed in layer $i$ of a \PMPNN{k}. We define
 \begin{linenomath}
 	\postdisplaypenalty=0
\begin{align*}
	\hb_{v}^\tup{i}\coloneqq  \pAGG\bigl(\oms\hb_{v,\pb}^\tup{i}\mid \pb\in G_k \text{ s.t. } \pmb\pi_{v,\pb}\neq\mathbf{0}\cms\bigr).
\end{align*}
 \end{linenomath}
We now show
 \begin{linenomath}
 	\postdisplaypenalty=0
\begin{align}\label{assum}
	C_{i}(v)=C_{i}(w)\Longrightarrow \hb_{v}^\tup{i}=\hb_{w}^\tup{i},
\end{align}
 \end{linenomath}
 for $i \geq 0$. To do so, we first show the following result. 
 \begin{claim} It holds that\label{claim_equal}
	\begin{equation}\label{imply_ind}
		C_i(v,\pb) = C_i(w,\pb') \Longrightarrow \hb_{v,\pb}^\tup{i}=\hb_{w,\pb'}^\tup{i} \text{ and } \pmb\pi_{v,\pb} =  \pmb\pi_{w,\pb'}
	\end{equation} 
	for all $\pb$ and $\pb' \in G_k$ and $i \geq 0$.
\end{claim}
 \begin{proof}
 We proof the result by induction on the number of iterations or layers $i$.
 The base case, $i = 0$, follows by definition of the initial coloring of the $\PWLk{k}$ and the initial features of \PMPNNs{k}, that is, both are dictated solely by the atomic type. The same holds for $\pmb\pi_{v,\pb}$ and  $\pmb\pi_{w,\pb'}$, which remain unchanged for layers $i > 0$.
 
Assume~\cref{imply_ind} holds for the first $i$ iteration and further assume $C_{i+1}(v,\pb) = C_{i+1}(w,\pb')$ holds. Hence, $C_i(v,\pb) = C_i(w,\pb')$ and $\oms C_{i}(u,\pb)\mid u \in \square  \cms = \oms C_{i}(u,\pb')\mid u \in \square  \cms$. We now define the multi-sets
 \begin{linenomath}
 	\postdisplaypenalty=0
	\begin{align*}
	M_{v,\pb}^\tup{i+1} \coloneqq \oms \hb_{u,\pb}^\tup{i} \mid u \in \square  \cms \quad \text{ and } \quad M_{w,\pb'}^\tup{i+1} \coloneqq \oms \hb_{u,\pb'}^\tup{i} \mid u\in \square \cms.
	\end{align*}
 \end{linenomath}
By the above, we know that $M_{v,\pb}^\tup{i+1} = M_{w,\pb'}^\tup{i+1}$ and $\hb_{v,\pb}^\tup{i} = \hb_{w,\pb'}^\tup{i}$. Therefore, regardless of the concrete choice of $\UPD^{(i+1)}$ and $\AGG^{(i+1)}$, $\hb_{v,\pb}^\tup{i+1} = \hb_{w,\pb'}^\tup{i+1}$.
\end{proof}
We are now ready to show~\cref{assum}. Hence, we assume $C_i(v,\pb) = C_i(v,\pb')$ for $i \geq 0$ holds. Hence, by assumption, the two multisets of \cref{equalcolor} are (element-wise) equal. Hence, there exists a bijection $\theta \colon \{ (v,\pb) \mid \pb \in G_k \} \to \{ (w,\pb) \mid \pb \in G_k \}$ such that $C_0(v,\pb) = C_0(\theta(v,\pb))$. By leveraging~Claim \ref{claim_equal}, we now construct a bijection $\varphi$ with the same domain and co-domain as $\theta$ such that $\hb_{v,\pb}^\tup{i} = \hb_{\varphi(v,\pb)}^\tup{i}$, implying~\cref{assum}. 

We construct the bijection $\varphi$ as follows. Take $(v,\pb) \in V(G) \times G_k$, let $\theta(v,\pb) = (w,\pb')$, and set $\varphi(v,\pb) =  \theta(v,\pb) = (w,\pb')$. Since $C_i(v,\pb) = C_i(\theta(v,\pb))$, Claim~\ref{claim_equal} implies that $\hb_{v,\pb}^\tup{i} = \hb_{\varphi(v,\pb)}^\tup{i}$ and $\pmb\pi_{v,\pb} = \pmb\pi_{\varphi(v,\pb)}$. Hence, by the existence of the bijection $\varphi$, we have that
\begin{align*}
	\oms\hb_{v,\pb}^\tup{i}\mid \pb\in G_k \text{ s.t. } \pmb\pi_{v,\pb}\neq\mathbf{0}\cms = \oms\hb_{w,\pb}^\tup{i}\mid \pb\in G_k \text{ s.t. } \pmb\pi_{w,\pb}\neq\mathbf{0}\cms.
\end{align*}
Hence, the feature vector $\hb_{v}$ is equal to $\hb_{w}$.
\end{proof}

\begin{lemma}\label{wlinmp}
For all $k \geq 1$, it holds that  
\begin{linenomath}
 	\postdisplaypenalty=0
\begin{align*}
 \rho(\text{\PMPNNs{k}})
  \subseteq
  \rho(\PWLk{k}).
\end{align*}
 \end{linenomath}
\end{lemma}
\begin{proof}
We argue that there exists a canonical $k$-PMPMN that can simulate the $\PWLk{k}$. 
By setting $\UPD$ to the identity function and a constant, non-zero function, we can simulate the initial labeling of the $\PWLk{k}$. For the other iterations, we need to show that there exist instances of $\UPD^{(i)}$, $\AGG^{(i)}$ for $i >0$, and $\pAGG$ that are injective, faithfully distinguishing non-equal multisets. The existence of such instances follows directly from the proof of Theorem 2 in \citep{Mor+2019}.
\end{proof}
The above two lemmas directly imply~\cref{equal}.

\subsection{Separation results}
\begin{theorem}[Theorem 3 in the main text]
	For all $k \geq 1$ it holds that 
	\begin{align*}
		\rho(\text{\PMPNNs{k+1}}) \subsetneq \rho(\text{\PMPNNs{k}}).
	\end{align*}
\end{theorem}
\begin{proof}
First, $\rho(\text{\PMPNNs{k+1}}) \subseteq \rho(\text{\PMPNNs{(k)}})$, is a direct consequence of the results in~\citep{geerts2022} showing that
$(k+1)$-\MPNNs are bounded by $\WLk{(k+1)}$ and that \PMPNNs{k} are a restricted class of 
$(k+1)$-\MPNNs. The strictness follows by~\cref{equal} and~\cref{thm:sepnew1}. 
\end{proof}


\paragraph{Construction of F\"urer grid-graphs} 
We restate the following construction due to~\citet{Furer01}. 
Let $h$ and $n$ be fixed positive integers such that $n >\!\!> h + 1$. 
Fix a \emph{global} graph $G^h_n$, defined to be a $h \times n$ grid graph.
Define a graph $X(G^h_n)$ as follows.
\begin{enumerate}
	\item For each vertex $v \in V(G^h_n)$,
			\begin{itemize}
				\item let degree of $v$ be $d$,
				\item let $E_v$ be the set of edges incident to $v$,
				\item replace $v$ by a \emph{vertex cloud} $C_v$ of $2^{d-1}$ vertices of the form $(v,S)$
						such that $S$ is an even subset of $E_v$.
			\end{itemize}
	\item\label{lab:twist} For each edge $e = \{u,v\} \in E(G^h_n)$, 
			\begin{itemize}
				\item for each $(u,S) \in C_u$ and $(v,T) \in C_v$,
					  add an edge between $(u,S)$ and $(v,T)$ if
					  \begin{itemize}
					  	\item both $S$ and $T$ contain $e$, or
					  	\item both $S$ and $T$ do not contain $e$. 
					  \end{itemize}
			\end{itemize}
\end{enumerate} 

The graph $Y(G^h_n)$ is defined exactly as $X(G^h_n)$, 
with the following exception. Fix an edge $\{u^*,v^*\}$ of the global graph $G^h_n$. In the second step above (\cref{lab:twist}),
we use a different rule for this edge $\{u^*,v^*\}$, 
\begin{itemize}
	\item for each $(u^*,S) \in C_{u^*}$ and $(v^*,T) \in C_{v^*}$,
	add an edge between $(u^*,S)$ and $(v^*,T)$ if
	\begin{itemize}
		\item exactly one out of $S$ and $T$ contains $e$
	\end{itemize}
\end{itemize}
The edge $\{u^*,v^*\}$ is said to be \emph{twisted}. Equivalently, $Y(G^h_n)$ is the graph obtained from the graph $X(G^h_n)$ by performing a bipartite-complement operation on the bipartite graph between 
the vertex clouds $C_{u^*}$ and $C_{v^*}$. 

For a vertex $v$ in $X_k$ or $Y_k$, let $\bar{v}$ denote the vertex $x$ in $G^h_n$ 
such that $v \in C_x$ (also called a meta-vertex in \citep{Furer01}). 
We assign a fresh color say $c_v$ to the the vertex cloud $C_v$ for each $v \in V(G^h_n)$, imposing an initial coloring on the graphs $X(G^h_n)$ and $Y(G^h_n)$. It is easy to see that this coloring is stable under Color Refinement. 

\paragraph{Construction and Properties of $X_k$ and $Y_k$}
To ease notation, we set $B = G^{k+1}_n$ as our $\emph{base graph}$ where $n >\!\!> k+1$. For $k \in \mathbb{N}$, we set $X_k = X(B)$ and $Y_k = Y(B)$. 
F\"urer showed that the graphs $X_k$ and $Y_k$ are non-isomorphic
yet $\WLk{k}$-indistinguishable. It was also shown that $\WLk{(k+1)}$ can distinguish these graphs after at least $n$ rounds. The proof technique relies on \emph{trapping the twist} using $k+2$ pebbles in a Spoiler-Duplicator game \citep{Imm+1990}.

Moreover, let $Z$ be a graph obtained by twisting some $\ell$ distinct edges of $B$, similar to how $Y_k$ is obtained from $B$ after a single twist. Then, it was shown that $Z$ is isomorphic to $X_k$ if $\ell$ is even, and $Z$ is isomorphic to $Y_k$ if $\ell$ is odd.  


\paragraph{Twists and Shields}
Let $u \in V(X_k)$. Let $\boldsymbol{u}$ in $V(X_k)^k$.
Let $B \backslash(\bar{\boldsymbol{u}},\bar{u})$ denote the graph obtained by
deleting the vertices in $(\bar{\boldsymbol{u}},\bar{u})$ in the base graph $B$. 
Let $e = (x,y)$ be the edge of $B$ which was twisted to obtain $Y_k$. 
Assume that at least one of its endpoints of $e$ is not in $(\bar{\boldsymbol{u}},\bar{u})$. 
The \emph{twisted component} of $B$ w.r.t $(\bar{\boldsymbol{u}},\bar{u})$, denoted by $\mathsf{TC}(\bar{\boldsymbol{u}},\bar{u})$, is the unique component of $B\backslash(\bar{\boldsymbol{u}},\bar{u})$ which contains the twisted edge. 

Let $N_{\mathsf{TC}}[\bar{\boldsymbol{u}},\bar{u}]$ be the neighborhood of vertices in $(\boldsymbol{u},u)$ into the twisted component, i.e., the set of vertices 
$v \in \mathsf{TC}(\bar{\boldsymbol{u}},\bar{u})$ which are incident to $(\bar{\boldsymbol{u}},\bar{u})$. Then, a twisted component is a \emph{shield} if it satisfies the following two properties: 
\begin{itemize}
    \item the twisted edge is not incident to any of the vertices in $(\bar{\boldsymbol{u}}, \bar{u})$ and $N_{\mathsf{TC}}[\bar{\boldsymbol{u}},\bar{u}]$, and
    \item the twisted edge lies on some cycle in $\mathsf{TC}(\bar{\boldsymbol{u}},\bar{u}) \backslash N_{\mathsf{TC}}[\bar{\boldsymbol{u}},\bar{u}]$.
\end{itemize}
In this case, we also call $(\boldsymbol{u},u)$ to be \emph{shielding} for $Y_k$. 
The motivation behind these conditions is as follows. The first condition ensures that the twist is at distance at least two from the individualized vertices. The second condition ensures that the twist cannot be trapped using just two pebbles. 

\begin{proposition}
\label{prop:shield}
	Suppose that $(\boldsymbol{u},u)$ is shielding for $Y_k$.  
	If we run color refinement on the disjoint union of $(X_k,\boldsymbol{u})$
	and $(Y_k,\boldsymbol{u})$, the stable color of $u$ in $(X_k,\boldsymbol{u})$ is identical to the stable color of $u$ in  $(Y_k,\boldsymbol{u})$.
\end{proposition}
\begin{proof}
Since $(\boldsymbol{u},u)$ is shielding for $Y_k$, the twisted edge lies on a cycle $C$ inside $\mathsf{TC}(\bar{\boldsymbol{u}},\bar{u}) \backslash N_{\mathsf{TC}}[\bar{\boldsymbol{u}},\bar{u}]$. Hence, every vertex of $C$ is at distance at least two from the vertices in $(\boldsymbol{u},u)$. We invoke the usual Spoiler-Duplicator games of Immerman-Lander to argue the desired claim \citep{Cai+1992}. 

We show that a Duplicator can always move around the twist such that it is never caught by the Spoiler. This game uses $k$ pairs of fixed pebbles corresponding to  $\boldsymbol{u}$ in each graph, and two pairs of movable pebbles which are placed on $u$ in $X_k$ and $Y_k$ respectively. Recall that color refinement can be simulated using a $2$-pebble Spoiler-Duplicator game \citep{Imm+1990}. Since the $k$ fixed pebbles are influential only in their neighbourhood, the Duplicator strategy is to move the twist around in the cycle $C$, so that the twist is always at a distance of at least two from the fixed pebbles. This renders the fixed pebbles useless for the Spoiler. Since there are only two movable pebbles, the Duplicator can always move the twist around in the cycle $C$ and hence avoid a situation where the Spoiler can trap the twist with the two movable pebbles. 
\end{proof}

\paragraph{Shielding Twists}

Let $u \in V(X_k)$ and $\boldsymbol{u} \in V(X_k)^k$. Next we show that if $(\boldsymbol{u},u)$ is not shielding for $Y_k$, we can do a series of twisting operations on the graph $Y_k$ to obtain an isomorphic graph $Y_k'$ such that
$(\boldsymbol{u},u)$ is shielding for $Y_k'$. 

\begin{proposition}\label{prop:makeshield}
If $(\boldsymbol{u},u)$ is not shielding for $Y_k$, there exists 
$\boldsymbol{v} \in V(X_k)^k$ such that 
$(\boldsymbol{v},u)$ is shielding for $Y_k$. 
Hence, if we run Color Refinement on the disjoint union of $(X_k,\boldsymbol{u})$
and $(Y_k,\boldsymbol{v})$, the stable color of $u$ in $(X_k,\boldsymbol{u})$ is identical to the stable color of $u$ in  $(Y_k,\boldsymbol{v})$.
\end{proposition} 

\begin{proof}
Since $n >\!\!> k$, there exists at least one component in 
$ B\backslash(\bar{\boldsymbol{u}},\bar{u})$ 
such that it contains a grid $G_{3\times 3}$ of dimension $3 \times 3$ as an induced subgraph, where $G_{3\times 3}$ does not have any edges to $(\boldsymbol{u},u)$.
Let $C^*$ be the lexicographically least such component in 
$B \backslash(\bar{\boldsymbol{u}},\bar{u})$. 
Our goal is to use an automorphism $\theta$ of $Y_k$ to transfer the twist to this grid $G_{3 \times 3}$ inside the component $C^*$ such that \emph{$\theta$ fixes $u$}, i.e. $\theta(u) = u$. This would mean that $(\boldsymbol{u}, u)$ is shielding for $Y_k^\theta$ with $C^*$ as the shield. Hence, we set
$\boldsymbol{v} = \boldsymbol{u}^{\theta^{-1}}$ so that $(\boldsymbol{v},u)$ is shielding for $Y_k$. 

To achieve this transformation, for every $\bar{\boldsymbol{u}} \in V(B)^k$, 
we fix a shortest path $P^{\bar{\boldsymbol{u}}}$ from one of the ends of the twisted edge to the central vertex of the grid $G_{3,3}$ such that $P$ \emph{avoids} $u$. We twist all the edges on the path $P^{\bar{\boldsymbol{u}}}$. If the length of the path $P$ is odd, we twist one more edge in $G_{3,3}$ so as to ensure that $P^{\bar{\boldsymbol{u}}}$ has even length. The resulting graph $Y_k'$ is isomorphic to $Y_k$ via a unique isomorphism $\theta$. 
Since the path $P$ avoids $u$, the isomorphism $\theta$ fixes $u$. 
Hence, $(\boldsymbol{u},u)$ is shielding for $Y_k'$, and therefore $(\boldsymbol{u}^{\theta^{-1}},u)$ is shielding for $Y_k$. Hence, proved. 
\end{proof}

Observe that the association $\boldsymbol{u} \mapsto \boldsymbol{v}$ in the proof of the above claim is bijective, as follows. Suppose there exists $\boldsymbol{w} \mapsto \boldsymbol{v}$ such that $\boldsymbol{u} \neq \boldsymbol{w}$. Now, $\boldsymbol{u}$ and $\boldsymbol{w}$ must have same initial color type, since the used isomorphisms preserve vertex clouds, i.e. $\bar{\boldsymbol{u}} = \bar{\boldsymbol{w}}$.
Hence, the same path $P^{\boldsymbol{u}}$ is used for both $\boldsymbol{u}$ and $\boldsymbol{v}$ in the base graph $B$. For a fixed path $P^{\boldsymbol{u}}$ of even length, there is a unique isomorphism $\theta$ which twists all the edges in $P$ to yield the graph $Y_k'$. Hence, it must be the case that $\boldsymbol{u} = \boldsymbol{w} = \boldsymbol{v}^{\theta^{-1}}$.\\

\begin{lemma}\label{lem:sepnew1} 
   For $k \in \mathbb{N}$, $\PWLk{k}$ cannot distinguish graphs $X_k$ and $Y_k$.
\end{lemma}
\begin{proof}
Let $\mathcal{X}$ denote the disjoint union of graphs $(X_k,\boldsymbol{u})$, $\boldsymbol{u} \in V(X_k)^k$. 
Let $\mathcal{Y}$ denote the disjoint union of graphs $(Y_k,\boldsymbol{v})$, $\boldsymbol{v} \in V(Y_k)^k$. 
It suffices to show the equality of the following nested multisets
\begin{align*} \label{eqn:card}
\oms\oms \mathsf{CR}(\mathcal{X},u^{\boldsymbol{u}}) \,\vert\, \boldsymbol{u} \in V(X_k)^k \cms\,\vert\, u \in V(X_k)\cms = 
\oms \oms \mathsf{CR}(\mathcal{Y},v^{\boldsymbol{v}}) \,\vert\, \boldsymbol{v} \in V(Y_k)^k \cms\,\vert\, v \in V(Y_k)\cms,
\end{align*}
where $u^{\boldsymbol{u}}$ denotes the vertex $u$ in the constituent $(X_k,\boldsymbol{u})$ of $\mathcal{X}$. Similarly, $v^{\boldsymbol{v}}$ denotes the vertex $v$ in the constituent $(X_k,\boldsymbol{v})$ of $\mathcal{Y}$. 

Observe that the graphs $X_k$ and $Y_k$ have the same vertex set. 
We claim that for every $u \in V(X_k)$, the corresponding vertex $u \in V(Y_k)$ 
satisfies 
\[
\oms \mathsf{CR}(\mathcal{X},u^{\boldsymbol{u}}) \,\vert\, \boldsymbol{u} \in V(X_k)^k \cms = 
\oms \mathsf{CR}(\mathcal{Y},u^{\boldsymbol{v}}) \,\vert\, \boldsymbol{v} \in V(Y_k)^k \cms.
\]
Indeed, this follows immediately from \cref{prop:shield} and \cref{prop:makeshield} along with the fact that the association in \cref{prop:makeshield} is bijective (see the discussion subsequent to \cref{prop:makeshield}). Hence, proved. 
\end{proof}

\begin{theorem}\label{thm:sepnew1}
For $k \in \mathbb{N}$, there exist graphs $X_k$ and $Y_k$ such that
they are distinguishable by $\WLk{(k+1)}$ but not distinguishable by $\PWLk{k}$.
\end{theorem}
\begin{proof}
Immediate from \cref{lem:sepnew1}.
\end{proof}

Next we compare the expressive power of $\WLk{k}$ and $\PWLk{k}$.

\begin{lemma}\label{lem:sepnew2}
For $k \in \mathbb{N}$, there exist graphs $X_k$ and $Y_k$ such that
they are distinguishable by $\PWLk{k}$ but not distinguishable by $\WLk{k}$.
\end{lemma}
\begin{proof}
We set $X_k$ and $Y_k$ to be CFI-gadgets $G_{k+1}$ and $H_{k+1}$ which are 
known to be indistinguishable by $\WLk{k}$ (see Section~\ref{subsec:CFI} for the definition of these gadgets). It remains to show that 
they can be distinguished by $\PWLk{k}$. Recall that $X_k$ contains a colorful distance-two-clique $Q$ of size $k+2$ while $Y_k$ does not contain such an object. We place $k$ fixed pebbles on some $k$ vertices of $Q$, and let $x,y$ be the remaining two vertices in $Q$. It is clear that upon two rounds of color refinement, the vertices $x$ and $y$ see all individualized colors corresponding to the fixed pebbles. Moreover, the individualized pebbles also see all the individualized colors of other pebbles. 

On the other hand, doing such an operation on $Y_k$ will never yield such colors, since this would otherwise ensure a colorful distance-two-clique in $Y_k$. Hence, there does not exist any $x' \in V(Y_k)$ and $\boldsymbol{v} \in V(Y_k)^k$ such that color refinement on the disjoint union of $(X_k, \boldsymbol{u})$ and $(Y_k,\boldsymbol{v}$) yields the same colors for $x$ and $x'$. Therefore
for any choice of $x' \in V(Y_k)$ it holds that 
the following multisets for vertices $x \in V(X_k)$ and $x'$, obtained by aggregation over all ordered subgraphs, satisfy
\[
\oms \mathsf{CR}(\mathcal{X},x^{\boldsymbol{u}}) \,\vert\, \boldsymbol{u} \in V(X_k)^k \cms \neq 
\oms \mathsf{CR}(\mathcal{Y},{(x')}^{\boldsymbol{v}}) \,\vert\, \boldsymbol{v} \in V(Y_k)^k \cms.
\]
which implies that the aggregated multisets over all vertices
\begin{align*} \label{eqn:card}
\oms\oms \mathsf{CR}(\mathcal{X},u^{\boldsymbol{u}}) \,\vert\, \boldsymbol{u} \in V(X_k)^k \cms\,\vert\, u \in V(X_k)\cms \neq
\oms \oms \mathsf{CR}(\mathcal{Y},v^{\boldsymbol{v}}) \,\vert\, \boldsymbol{v} \in V(Y_k)^k \cms\,\vert\, v \in V(Y_k)\cms.
\end{align*}
Hence, $\PWLk{k}$ distinguishes $X_k$ and $Y_k$. 
\end{proof}


The following theorem shows that the algorithms $\PWLk{k}$, $k \in \mathbb{N}$, form a hierarchy of increasingly powerful isomorphism tests. 

\begin{theorem}\label{strictly}
For $k \in \mathbb{N}$, $\PWLk{k}$ has strictly less expressive power than 
$\PWLk{(k+1)}$.
\end{theorem}
\begin{proof}
The proof follows immediately from \cref{thm:sepnew1} and \cref{lem:sepnew2}; see below.
\end{proof}


\section{Vertex-subgraph $\PWLk{k}$ and \PMPNN{k}: Omitted Proofs}\label{sec:vertexpebble}

In this section we consider a variant of $\PWLk{k}$, denoted vertex-subgraph $\PWLk{k}$, in which the construction of the multi-sets used to define the color of graph is defined differently. As before, we define $C_i(v,\pb)$ and $C_\infty(v,\pb)$ for $v\in V(G)$ and $\pb\in G_k$. Then, instead of computing a single color for a vertex $v$, we compute a single color for $\pb\in G_k$. We do this by aggregating over all vertex in $G$, i.e, we compute
\begin{align*}
	C(\pb) \coloneqq  \REL \bigl(\oms C_{\infty}(v,\pb) \mid v\in V(G) \cms \bigr).
\end{align*}
Finally, we use 
\begin{align*}
	 \REL\bigl(\oms C(\pb)\mid \pb\in G_k\cms\bigr)
\end{align*}
to obtain the color $C(G)$ of $G$. The neural counterpart, vertex-subgraph \PMPNNs{k}, are defined in a similar way. That is, $\hb_{v,\pb}^\tup{i}$ is defined as for \PMPNNs{k} but we now define
\allowdisplaybreaks
\begin{linenomath}
	\postdisplaypenalty=0
	\begin{align*}
		\hb_{\pb}^\tup{T}&\coloneqq\AGG\bigl(\oms \hb_{v,\pb}^\tup{T}\mid v\in V(G)\cms\bigr)\\
		\hb_G&\coloneqq\pAGG\bigl(\oms\hb_{\pb}^\tup{T}\mid \pb\in V(G)^k, v\in V(G)\, \pmb\pi_{v,\pb}\neq\mathbf{0}\cms\bigr).
	\end{align*}	
\end{linenomath}
Again, $\AGG$ and $\pAGG$ are differentiable, parameterized functions, e.g., neural networks.

\subsection{Equivalence of vertex-subgraph $\PWLk{k}$ and vertex-subgraph \PMPNN{k}}

\begin{proposition} For all $k \geq 1$, vertex-subgraph \PMPNNs{k} and vertex-subgraph $\PWLk{k}$ have the same distinguishing power. 
\end{proposition}

The proof consists in showing that (i) vertex-subgraph \PMPNNs{k} cannot distinguish more graphs than vertex-subgraph $\PWLk{k}$ (\cref{lem:pmpnn_pwl_new}); and (ii) vertex-subgraph $\PWLk{k}$ cannot distinguish more graphs than vertex-subgraph \PMPNNs{k} (\cref{lem:wlinmp_new}).
\begin{lemma}\label{lem:pmpnn_pwl_new}
For all $k \geq 1$, it holds that  
$\rho(\PWLk{k})\subseteq \rho(\text{\PMPNNs{k}}) $.
\end{lemma}
\begin{proof}
Consider graphs $G$ and $H$ in $\rho(\PWLk{k})$. By definition, this implies that
\begin{equation}
\REL\bigl(\oms C(\pb)\mid \pb\in G_k\cms\bigr)=
\REL\bigl(\oms C(\qb)\mid \qb\in H_k\cms\bigr)\label{eq:pebbl}
\end{equation}
holds.
We next show 
$C(\pb)=C(\qb)$ for $\pb\in G_k$ and $\qb\in H_k$ implies that any  \PMPNN{k} computes the same features for $k$-ordered subgraphs $\pb$ and $\qb$. Combined with 
\cref{eq:pebbl} this implies that any \PMPNN{k} assigns the same feature to $G$ and $H$.
Assume $C(\pb)=C(\qb)$, hence
$C(\pb):=\REL\bigl(\oms C_\infty(v,\pb)\mid v\in V(G)\cms\bigr)$ and $C(\qb):=\REL\bigl(\oms C_\infty(w,\qb)\mid w\in V(H)\cms\bigr)$. Hence, for $C(\pb)=C(\qb)$ to hold, we need that
\begin{equation}
C_i(\pb):=\oms C_i(v,\pb)\mid v\in V(G)\cms=\oms C_i(w,\qb)\mid w\in V(H)\cms=:C_i(\qb) \label{eq:mspebbles}
\end{equation}
for all iterations $i$ of the $\PWLk{k}$. On the \PMPNN{k} side we define
\[
\hb_{\pb}^\tup{i}:=\AGG\bigl(\oms \hb_{v,\pb}^\tup{i}\mid v\in V(G)\bigr)
\]
and similarly for $\hb_{\qb}^\tup{i}$. We now show
\begin{equation}
C_i(\pb)=C_i(\qb)\Longrightarrow \hb_{\pb}^\tup{i}=\hb_{\qb}^\tup{i}.
\end{equation}
Indeed, $C_i(\pb)=C_i(\qb)$ and \cref{eq:mspebbles} imply that there exists a bijection
$\theta:V(G)\to V(H)$ such that $C_i(v,\pb)=C_i(\theta(v),\qb)$. Claim \ref{claim_equal}
implies that $\hb_{v,\pb}^\tup{i}=\hb_{\theta(w),\qb}^\tup{i}$ and thus $\theta$ can be used to define a bijection between the multisets defining  $\hb_{\pb}^\tup{i}$ and $\hb_{\qb}^\tup{i}$. Hence, $\hb_{\pb}^\tup{i}=\hb_{\qb}^\tup{i}$ as desired.
\end{proof}

\begin{lemma}\label{lem:wlinmp_new}
For all $k \geq 1$, it holds that $ 
\rho(\text{\PMPNNs{k}})\subseteq   \rho(\PWLk{k})$.
\end{lemma}
This is shown in precisely the same way as \cref{wlinmp}.

\subsection{CFI-Gadgets}\label{subsec:CFI}
The comparison with $\WLk{k}$ and separation results are derived from a graph  construction, also outlined in \citet[Appendix C.1.1]{Morris2020b}. They provide an infinite family of graphs $(G_k, H_k)$, $k \in \mathbb{N}$, such that (a) $\WLk{(k-1)}$ does not distinguish $G_k$ and $H_k$ but (b) $\WLk{k}$ distinguishes $G_k$ and $H_k$. In the following, we recall some relevant results from their paper. 

\paragraph{Construction of $G_k$ and $H_k$.} Let $K_{k+1}$ denote the complete graph on $k+1$ nodes
(there are no self-loops). We index the nodes of $K_{k+1}$ from $0$ to $k$. Let $E(v)$ denote the set of edges incident to $v$ in $K_{k+1}$. Clearly, $|E(v)| = k$ for all $v \in V(K_{k+1})$. 
We define the graph $G_k$ as follows.
\begin{enumerate}
	\item For the node set $V(G_k)$, we add   
	\begin{enumerate}
		\item[(a)]\label{vc} $(v,S)$ for each $v$ in $V(K_{k+1})$ and for each \emph{even} subset $S$ of $E(v)$.
		\item[(b)]\label{ec} two nodes $e^1$ and $e^0$ for each edge $e$ in $E(K_{k+1})$.  
	\end{enumerate} 
	\item For the edge set $E(G_k)$, we add 
	\begin{enumerate}
		\item[(a)] an edge $(e^0,e^1)$ for each $e$ in $ E(K_{k+1})$, 
		\item[(b)] an edge between $(v,S)$ and $e^1$ if $v$ in $e$ and $e$ in $S$,  
		\item[(c)] an edge between $(v,S)$ and $e^0$ if $v$ in $ e$ and $e$ not in $S$.
	\end{enumerate} 
\end{enumerate} 

For $v \in V(K_{k+1})$, the set of vertices of the form $(v,S)$ defined in \cref{vc} are assigned a common color $C_v$. They form what we call a \emph{vertex-cloud} corresponding to the vertex $v$. Similarly, for $e \in E(K_{k+1})$, the two vertices $e^0$ defined in \cref{ec} are assigned a common color $C_e$. They form what we call an \emph{edge-cloud} corresponding to the edge $e$. A \emph{vertex-cloud vertex} is a vertex of the form $(v,S)$ as defined above. An \emph{edge-cloud vertex} is a vertex of the form $e^0$ or $e^1$ as defined above. 

We define the graph $H_k$, in a similar manner to $G_k$, with the following exception. In step 1(a), for the node $0$ in $V(K_{k+1})$, we choose all \emph{odd} subsets of $E(0)$. Clearly, both graphs have $(k)\cdot 2^{k} + \binom{k+2}{2} \cdot 2$ nodes. 
The above construction of graphs $(G_k,H_k)$ is essentially the application of the classic Cai-F{\"u}rer-Immerman construction to a $(k+1)$-clique: we refer to these graphs as \emph{CFI-gadgets} henceforth.

\paragraph{Distance-two cliques.}
We say that a set $S$ of nodes form a \emph{distance-two-clique} if the distance between any two nodes in $S$ is exactly two. 
A distance-two-clique $S$ is \emph{colorful} if (a) every vertex of $S$ is of vertex-cloud kind, and (b) no two vertices in $S$ belong to the same vertex cloud. Clearly, each vertex in a colorful distance-two-clique has a unique initial color. The following lemma is a mild strengthening of a lemma from \citet{Morris2020b}: the proof is a straightforward derivation from the proof of their lemma.

\begin{lemma}[\citep{Morris2020b}]\label{lem:dtc}
	The following holds for the graphs $G_k$ and $H_k$ defined above. 
	\begin{itemize}
		\item There exists a set of $k+1$ vertex-cloud vertices in $G_k$ such that they form a colorful distance-two-clique of size $(k+1)$.
		\item There does not exist a set of $k+1$ vertex-cloud vertices in $H_k$ such that they form a colorful distance-two-clique of size $(k+1)$.
	\end{itemize}
	Hence, $G_k$ and $H_k$ are non-isomorphic. 
\end{lemma}

Further, they showed the following results regarding the power and limitations of Weisfeiler-Leman vis-a-vis such graphs.  

\begin{lemma}[\citep{Morris2020b}]\label{lem:neurips}
	The $\WLk{(k-1)}$ does not distinguish $G_k$ and $H_k$. 
\end{lemma}

\begin{lemma}[\citep{Morris2020b}]\label{lem:neurips2}
	The $\WLk{k}$ does distinguish $G_k$ and $H_k$. 
\end{lemma}

\subsection{Separation results: Comparison of vertex-subgraph $\PWLk{k}$ and $\WLk{k}$.}
We now compare the relative expressive power of the $k$-ordered subgraph Weisfeiler-Leman and the standard Weisfeiler-Leman. 
We remark that, by definition, $\PWLk{0}=\WLk{1}$, so in the remainder of this section we consider $\PWLk{k}$ for $k>0$.

We show that $\PWLk{k}$ is bounded in distinguishing power by $\WLk{k+1}$ (\cref{lem:pwlinkplus1wl}), yet there are are graphs that can be distinguished by $\WLk{k+1}$ but not by $\PWLk{k}$ (\cref{prop:pwlweakerthankplus1wl}). Moreover, $\PWLk{k}$ can distinguish graphs which cannot be distinguished by $\WLk{k}$ (\cref{lem:pwlmorepowerkwl}).
As a consequence,
As a consequence, the algorithms $\PWLk{k}$, $k \in \mathbb{N}$, form a strict hierarchy of vertex-refinement algorithms.  
\begin{lemma}
Let $k \in \mathbb{N}$. Then $\PWLk{k}$ is strictly less expressive than $\PWLk{(k+1)}$. 
\end{lemma}
\begin{proof}
We have shown that (a) $\PWLk{k}$ is strictly less expressive than $\WLk{(k+1)}$, 
(b) there exist $\WLk{(k+1)}$-indistinguishable graphs which are distinguished by $\PWLk{(k+1)}$.  Hence, we obtain the desired claim.
\end{proof}


\begin{lemma}\label{lem:pwlinkplus1wl}
Let $k\in\mathbb{N}$. Then $\PWLk{k}$ is strictly less expressive than $\WLk{(k+1)}$.
\end{lemma}

We start by describing a construction of a new family of graphs ($X_k,Y_k)$, $k\in\mathbb{N}$, based on the CFI gadgets $G_k$ and $H_k$.

\paragraph{Construction of the graphs $X_k$ and $Y_k$}
The graph $X_k$ is defined as follows. Let $C$ be a \emph{backbone} cycle of length four with vertices $u_1,u_2,u_3,u_4$, each colored ``red''. We attach the CFI gadgets $G_{k},H_{k}$ to each of these four vertices as follows. By ``attaching a gadget $F$ to a vertex $u$'', we mean that every vertex of the gadget $F$ is made adjacent to the backbone vertex $u$. Conversely, a backbone vertex $u$ \emph{points} to a gadget $F$ if $F$ is attached to $u$.  
Going back to our construction of $X_k$, we attach a copy of $G_{k}$ each to $u_1$ and $u_3$. 
We also attach a copy of $H_{k}$ each to $u_2$ and $u_4$. All the gadget vertices retain their original colors.   

The graph $Y_k$ is defined similarly to $X_k$ with the following exception. We attach a copy of $G_{k}$ each to consecutive vertices $u_1$ and $u_2$, while we attach a copy of $H_{k}$ each to consecutive vertices $u_3$ and $u_4$. Hence, $X_k$ and $Y_k$ only differ in the cyclic ordering of the attached gadgets. 

Observe that the backbone vertices in $X_k$ and $Y_k$ are colored `red' initially. The gadget vertices inherit their colors from the construction of graphs $G_k$ and $H_k$. These colors are either vertex cloud colors, say $\{C_i : i \in [k+1] \}$,
or the edge clouds colors $\{ C_{ij} : \{i,j\} \in {k+1 \choose 2} \}$. We call these two kinds of colors along with the red color as the \emph{basic} colors. Let the basic color of a vertex $u$ be denoted by $\beta(u)$. 

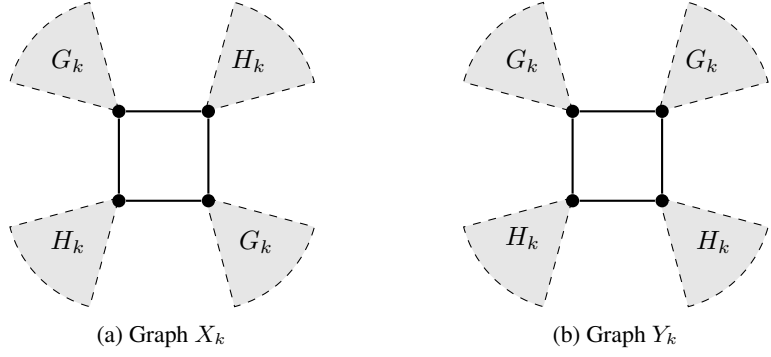
\begin{figure}[t]
\centering
\captionsetup[subfigure]{justification=centering}
\tikzset{
vertex/.style = {fill,circle,inner sep=0pt,minimum size=5pt},
edge/.style = {-,thick},
link/.style = {-, semithick},
lbl/.style={color=lightgray},
node/.style = {inner sep=0.7pt,circle,draw,fill}
}
\begin{subfigure}[t]{0.35 \textwidth}
\centering
\begin{tikzpicture}
\node[vertex] (a) [label = {}] {} ;
\node[vertex] (b) [label = {}] [right = of a] {}; 
\node[vertex] (c) [label = {}] [below = of b] {}; 
\node[vertex] (d) [label = {}] [left = of c] {}; 

\draw[edge] (a) -- (b) {}; 
\draw[edge] (b) -- (c) {}; 
\draw[edge] (c) -- (d) {}; 
\draw[edge] (d) -- (a) {}; 

\begin{scope}[xshift = -0.8cm, yshift=0.45cm]
\node (aa) at (60:0.25cm) {$G_k$};
\end{scope}

\begin{scope}[xshift = 1.60cm, yshift=0.45cm]
\node (aa) at (60:0.25cm) {$H_k$};
\end{scope}

\begin{scope}[xshift = 1.70cm, yshift=-1.95cm]
\node (aa) at (60:0.25cm) {$G_k$};
\end{scope}

\begin{scope}[xshift = -0.8cm, yshift=-1.95cm]
\node (aa) at (60:0.25cm) {$H_k$};
\end{scope}

\begin{scope}[on background layer]
\draw[dashed, thin ,fill = black!10] (0,0) -- (105:15mm)
arc [start angle=105, end angle=165, radius=15mm] -- (0,0); 
\draw[dashed, thin ,fill = black!10, xshift = 1.15cm] (0,0) -- (15:15mm)
arc [start angle=15, end angle=75, radius=15mm] -- (0,0); 
\draw[dashed, thin ,fill = black!10, yshift = -1.15cm] (0,0) -- (195:15mm)
arc [start angle=195, end angle=255, radius=15mm] -- (0,0); 
\draw[dashed, thin ,fill = black!10, xshift = 1.15cm, yshift= -1.15cm] 
(0,0) -- (285:15mm)
arc [start angle=285, end angle=345, radius=15mm] -- (0,0); 
\end{scope}
\end{tikzpicture}
\caption{Graph $X_k$}
\label{fig:a}
\end{subfigure}
\quad\quad\quad
\begin{subfigure}[t]{0.35 \textwidth}
	\centering
	\begin{tikzpicture}
	\node[vertex] (a) [label = {}] {} ;
	\node[vertex] (b) [label = {}] [right = of a] {}; 
	\node[vertex] (c) [label = {}] [below = of b] {}; 
	\node[vertex] (d) [label = {}] [left = of c] {}; 
	
	\draw[edge] (a) -- (b) {}; 
	\draw[edge] (b) -- (c) {}; 
	\draw[edge] (c) -- (d) {}; 
	\draw[edge] (d) -- (a) {}; 
	
	\begin{scope}[xshift = -0.79cm, yshift=0.45cm]
	\node (aa) at (60:0.25cm) {$G_k$};
	\end{scope}
	
	\begin{scope}[xshift = 1.60cm, yshift=0.45cm]
	\node (aa) at (60:0.25cm) {$G_k$};
	\end{scope}
	
	\begin{scope}[xshift = 1.75cm, yshift=-1.95cm]
	\node (aa) at (60:0.25cm) {$H_k$};
	\end{scope}
	
	\begin{scope}[xshift = -0.79cm, yshift=-1.9cm]
	\node (aa) at (60:0.25cm) {$H_k$};
	\end{scope}
	
	\begin{scope}[on background layer]
	\draw[dashed, thin ,fill = black!10] (0,0) -- (105:15mm)
	arc [start angle=105, end angle=165, radius=15mm] -- (0,0); 
	\draw[dashed, thin ,fill = black!10, xshift = 1.15cm] (0,0) -- (15:15mm)
	arc [start angle=15, end angle=75, radius=15mm] -- (0,0); 
	\draw[dashed, thin ,fill = black!10, yshift = -1.15cm] (0,0) -- (195:15mm)
	arc [start angle=195, end angle=255, radius=15mm] -- (0,0); 
	\draw[dashed, thin ,fill = black!10, xshift = 1.15cm, yshift= -1.15cm] 
	(0,0) -- (285:15mm)
	arc [start angle=285, end angle=345, radius=15mm] -- (0,0); 
	\end{scope}
	\end{tikzpicture}
	\caption{Graph $Y_k$}
	\label{fig:a}
\end{subfigure}
\caption{Pair of graphs which are $\WLk{(k+1)}$ distinguishable but $\PWLk{k}$ indistinguishable.
	The graphs $G_k$ and $H_k$ are CFI gadgets. Shaded sector represents uniform adjacency to the backbone vertex.} 
\label{fig:cntex}
\end{figure}

\newcommand{\FOC}{\mathsf{FOC}}

\begin{proposition} \label{prop:foc}
For $k \in \mathbb{N}$, $\WLk{(k+1)}$ distinguishes the graphs $X_k$ and $Y_k$.
\end{proposition}
\begin{proof}
It suffices to define a $(k+2)$-variable sentence $\varphi$ in first-order logic with counting quantifiers 
($\FOC$) 
such that $\varphi(X_k) \neq \varphi(Y_k)$ (Indeed, \citet{Cai+1992} establishes a precice  correspondence between $\WLk{k}$ and $\FOC$ sentences using $(k+1)$ variables). Intuitively, the sentence $\varphi$ expresses that there is a backbone vertex which has two different backbone vertices, each of which pointing to a CFI-gadget containing a colorful distance-two clique of size $(k+1)$.

We first note that a distance-two-clique of size $(k+1)$ over vertex-cloud vertices is definable as a $\FOC$-formula on $(k+2)$ variables. Indeed, let $\mathsf{C}_i$, $i\in[k+1]$ be unary color predicates for vertex clouds, and $\mathsf{C
}_{ij}$, $\{i,j\} \in {k+1 \choose 2}$ unary predicates for edge clouds. Then,
\[
\mathsf{DTC}(x_1,\dots,x_{k+1}) := \bigwedge_{\substack{i,j \in [k+1]\\i\neq j}} \mathsf{C}_i(x_i) \land \mathsf{C}_j(x_j)\land \exists\, x_{k+2} \, (\mathsf{C}_{ij}(x_{k+2}) \land (E(x_{k+2},x_i) \land E(x_{k+2},x_j)))
\]
is a formula that is satisfied by a $(k+1)$-tuple of vertices in a graph when they form a colorful distance-two clique with colors based on vertex and edge clouds.


We proceed to the description of $\varphi$. Let $\exists\,X$ denote the chain of $k+1$ quantifiers $\exists\,x_1 \cdots \exists\,x_{k+1} $. 
Since the backbone vertices in $X_k$ and $Y_k$ receive a distinct color (red), let $\mathsf{BB}(x)$ denote the unary predicate encoding this condition. 
The $(k+2)$-variable formula $\mathsf{POINT}_{G_k}(x)$ encodes whether 
a backbone vertex $x$ points to a $G_k$-gadget, by requiring the existence of a distance-two-clique of the kind stated in \cref{lem:dtc}. 
\begin{align*}
\mathsf{POINT}_{G_k}(x_{k+2}) := \mathsf{BB}(x_{k+2}) \land \exists X \,\,(\,\, \mathsf{DTC}(x_1,\dots,x_{k+1}) \land \bigwedge_{i \in [k+1]} E(x_{k+2},x_i) \,\,)
\end{align*}
Then the desired sentence 
\begin{align*}
\varphi &:= \exists\,x_{k+1} \left( \, \mathsf{BB}(x_{k+1}) \,\land \, \exists^{=2}x_{k+2} \, \left( \mathsf{BB}(x_{k+2}) \land \, E(x_{k+2},x_{k+1}) \land  \mathsf{POINT}_{G_k}(x_{k+2}) \right) \right)
\end{align*}

It is know clear that $X_k$ satisfies $\varphi$: it has a backbone vertex $u_3$ which has exactly two backbone neighbours $u_2$ and $u_4$, each of which point to a $G_k$-gadget. Since $G_k$ contains a distance-two-clique of size $(k+1)$ while $H_k$ does not contain a distance-two-clique of size $(k+1)$, $X_k$ satisfies $\varphi$. On the other hand, $Y_k$ does not satisfy $\varphi$ because it does not have any backbone vertex with two such backbone neighbours. 
\end{proof}

\paragraph{Cyclic Types} Given a vertex tuple $\boldsymbol{z}=(z_1,\dots,z_k) \in V(X_k)^k$, define its \emph{cyclic type} as follows. For $i \in [4]$, let $S_i$ denote the set of all $j \in [k]$ such that the vertex $z_j$ is either equal to or is attached to the backbone vertex $u_i$. Call $|S_i|$ to be the \emph{weight} of the backbone vertex $u_i$. This associates a cyclic sequence $S_{\boldsymbol{z}} = (S_1,S_2,S_3,S_4)$ with $\boldsymbol{z}$: by cyclic sequence, we mean that only the cyclic ordering of the sets matters, e.g., the cyclic sequence $(S_2,S_3,S_4,S_1)$ is equal to the cyclic sequence $(S_1,S_2,S_3,S_4)$. The cyclic type of $\boldsymbol{z}=(z_1,\dots,z_k)$ is then defined by the tuple $(\beta(z_1),\dots,\beta(z_k))$ of basic colors and the cyclic sequence $S_{\boldsymbol{z}}$. The same procedure can be used to define the cyclic type of a vertex tuple $\boldsymbol{z} \in V(Y_k)^k$. Further, a cyclic type is said to be \emph{simple} if the weight of every backbone vertex is at most $k-2$. If the cyclic type is not simple, there is a unique backbone vertex of weight at least $k-1$. We call such a vertex as a \emph{leader}. If the weight of the leader is exactly $k-1$, there exists a \emph{follower} vertex of weight one. A backbone vertex of weight zero is called \emph{weightless}.

Instead of usual color refinement (CR), we define a \emph{skewed} color refinement (SCR) on the graphs $(X_k,\boldsymbol{u})$ and $(Y_k,\boldsymbol{v})$, as follows: in the first stage, we exhaustively and exclusively refine the class of backbone vertices. This refinement uses color information from both backbone and non-backbone vertices. In the second stage, we do the usual Color Refinement on the resulting graph. Using standard arguments, it is easy to show that both CR and SCR both converge to the same stable coloring. 


\begin{proposition}\label{prop:pwlweakerthankplus1wl}
For $k \in \mathbb{N}$, $\PWLk{k}$ cannot distinguish the graphs $X_k$ and $Y_k$.
\end{proposition}
\begin{proof}
It suffices to show a partition $\mathcal{P}_X$ of $V(X_k)^k$ into $m$ classes $P_1,\dots,P_m$ and a partition $\mathcal{P}_Y$ of $V(Y_k)^k$ into $m$ classes $Q_1,\dots,Q_m$ such that 
for each $i \in [m]$ it holds that (a) $|P_i| = |Q_i|$ , and (b) graphs $(X_k,\boldsymbol{u})$ and $(Y_k,\boldsymbol{v})$ are indistinguishable under color refinement, for all $\boldsymbol{u} \in P_i$ and for all $\boldsymbol{v} \in Q_i$. Here $(X_k,\boldsymbol{u})$ stands for a copy of the graph $X_k$ in which the vertex $\boldsymbol{u}_i$ receives a distinct initial color $i$, for each $i \in [k]$. Similarly, $(Y_k,\boldsymbol{v})$ stands for a copy of the graph $Y_k$ in which the vertex $\boldsymbol{u}_i$ receives a distinct initial color $i$, for each $i \in [k]$.

For the partition $\mathcal{P}_X$, we first partition the tuples in $V(X_k)^k$ into sets $P_\tau$ according to their cyclic type $\tau$. Next, if a cyclic type $\tau$ is not simple, we further partition the set $P_\tau$ depending on whether the leader vertex points to a $G_k$-gadget or a $H_k$-gadget. We obtain a corresponding partition $\mathcal{P}_Y$ following the same process for $Y_k$. 

Instead of usual color refinement (CR), we do a \emph{skewed} color refinement (SCR) on the graphs $(X_k,\boldsymbol{u})$ and $(Y_k,\boldsymbol{v})$: in first stage, we exhaustively and exclusively refine the class of backbone vertices, and in second stage we do the usual Color Refinement on the resulting graph. Using standard arguments, it is easy to show that both CR and SCR both converge to the same stable coloring. 

Let $\tau$ be a simple cyclic type. It is easy to verify that the number of tuples of type 
$\tau$ in $V(X_k)^k$ and $V(Y_k)^k$ are equal. Suppose that $\boldsymbol{u} \in V(X_k)^k$ and $\boldsymbol{v} \in V(Y_k)^k$ have the cyclic type $\tau$.  After the first stage of SCR, the two graphs are indistinguishable because of the cyclic types being equal. After the second stage of SCR, the backbone vertices do not get refined any further: since each gadget has at most $k-2$ individualized vertices, it is not possible to identify whether it is a $G_k$ gadget or a $H_k$ gadget with color refinement (otherwise $\WLk{k}$ would also distinguish $G_k$ and $H_k$). Hence, Color Refinement does not distinguish $(X_k,\boldsymbol{u})$ and $(Y_k,\boldsymbol{v})$. 

Otherwise, let $\tau$ be a non-simple type. Let $\boldsymbol{u} \in V(X_k)^k$ and 
$\boldsymbol{v} \in V(Y_k)^k$ of type $\tau$ such that their leader vertices point to a $G_k$-gadget. Again, it is easy to verify that the number of such tuples in $V(X_k)^k$ and $V(Y_k)^k$ are equal. Again, we do a skewed color refinement on the graphs $(X_k,\boldsymbol{u})$ and $(Y_k,\boldsymbol{v})$. After the first stage of SCR, the two graphs are again indistinguishable because of the cyclic types being equal. After the second stage of SCR, the backbone vertices again do not get refined any further for the following reason. The leader and the follower vertices are already in singleton color classes. If the weightless vertices are not already distinguished after stage one, they will not get distinguished any further because the gadgets attached to them do not have any individualized vertices and hence they cannot be distinguished by Color Refinement (i.e. $\WLk{1}$). Hence, CR does not distinguish $(X_k,\boldsymbol{u})$ and $(Y_k,\boldsymbol{v})$. A similar argument works when both leader vertices point to a $H_k$-gadget. This finishes our case analysis.  
\end{proof} 


We conclude by comparing $\WLk{k}$ and $\PWLk{k}$.
\begin{lemma}\label{lem:pwlmorepowerkwl}
For each $k \in \mathbb{N}$ there exist $\WLk{k}$-indistinguishable graphs which are distinguished by $\PWLk{k}$. 
\end{lemma}

\begin{proof}
We show that the CFI-gadget graphs $G_{k+1}$ and $H_{k+1}$ can be distinguished by $\PWLk{k}$. Since \cref{lem:neurips} implies that these graphs cannot be distinguished by $\WLk{k}$ this suffices.

More specifically, we will identify a $k$-ordered subgraph $\pb$ in $V(G_{k+1})^k$ for which $C(\pb)$ is different from any $C(\qb)$ with
$\qb$ in $V(H_{k+1})^k$.

Let $\{v_1,\dots,v_k,v_{k+1},v_{k+2}\}$ be a colorful distance-two-clique in $G_{k+1}$. Recall that each $v_i$ lies in a distinct vertex cloud. We set 
$\pb=(v_1,\ldots,v_k)$.  For each pair $v_i,v_j$ with $i,j\in[k]$, there exist a vertex $v_{ij}$ in an edge cloud, such that $(v_i,v_{ij})$ and $(v_{ij},v_j)$ are edges. This information is captured by $C(v_{ij},\pb)$ and hence also by 
$C(\pb)$. In other words, $C(\pb)$ will reflect that the vertices in $\pb$ form a colorful distance-two clique of size $k$. We now argue that $C(\pb)$ also reflects that there is a distance-two clique of size $k+2$ in $G_{k+1}$.

Indeed, observe that $C(v_{k+1},\pb)$
contains information that $v_{k+1}$ is connected to all vertices in $\pb$ at distance two, and similarly for $C(v_{k+2},\pb)$. Moreover,
since $C(v_{k+1},\pb)$ reflects that $v_{k+2}$ is at distance two from $v_{k+1}$. In other words, $C(\pb)$ indeed reflects that there is a colorful distance-two cliques of size $k+2$ in $G_{k+1}$. By Lemma~\ref{lem:dtc}, $C(\qb)$ cannot reflect this since $H_{k+1}$ does not contain a colorful distance-two cliques of size $k+2$.\end{proof}

\section{Subgraph-enhanced GNNs as \PMPNNs{k}, Proofs of~\cref{mark,identity,esan}.}\label{sec:sgnns}

\subsection{Unordered vs. ordered subgraphs}\label{unordered}

We specified \PMPNNs{k} using ordered $k$-vertex subgraphs $G[\vec{v}]$ with $\vec{v}\in V(G)^k$. The order information is
encoded in $G[\vec{v}]$ by means of the vertex labels in $[k]$ of
the vertices in $\vec{v}$. In the unordered case, we would simply consider $G[\vec{v}]$ without any labels. That is, \PMPNNs{k} using ordered $k$-vertex subgraphs can simulate any \PMPNN{k} using unordered $k$-vertex subgraphs.

As an example of how ordered $k$-vertex subgraphs can be used, consider $\vec{v}=(v_1,v_2,v_2,v_3)\in V(G)^4$ and assume that the vertices $v_1,v_2$ and $v_3$ form a $3$-clique in $G$. In the ordered case
$G[\vec{v}]$ is simply the $3$-clique, in the ordered case, $G[\vec{v}]$ is the $3$-clique in which $v_1$ is labeled with $1$,
$v_2$ with $2$ and $3$, and $v_3$ with $3$. Suppose we want to use the selected subgraph to simulate edge deletions. Then, in the unordered case one cannot distinguish between the different edges in the $3$-clique, and hence they will be all treated as deleted. In contrast, in the ordered case we can, e.g., only delete edges with end points labeled $1$ and $2$, and $2$ and $3$, leaving the edge labeled $1$ and $3$ intact.

\subsection{Proofs}

In the following, we show that \PMPNNs{k} capture  most recently proposed subgraph-enhanced GNNs, implying~\cref{mark,identity,esan}. 


\paragraph{Marked GNNs, dropout GNNs and reconstruction GNNs}
Dropout GNNs \citep{Pap+2021} 
generate vertex embeddings by running
classical \MPNNs on $k$-vertex deleted subgraphs and then aggregating the obtained embeddings. Dropout GNNs were generalized to \textit{Marked GNNs} (\mgnns{k})~\citep{Pap+2022}  in which the $k$ vertices are just marked, in contrast to always being deleted. 
For efficiency reasons,  a random strategy is used to select the $k$ vertices to be marked or deleted~\citep{Pap+2021,Pap+2022}. 

Here, we consider  the deterministic variant of  \mgnns{k} in which all possible sets of $k$ vertices are considered to be marked (as this provides the maximum distinguishing power) as is also used in~\cite{Cot+2021} in the context of reconstruction GNNs. 
The marking process in \mgnns{k} naturally relates to the 
selection of unordered $k$-vertex subgraphs, as we will illustrate.

Let $G$ be a graph and let $M\subseteq V(G)$, with $|M|=k$, be a set of $k$ marked vertices.
Let $N_G^{M}(v)\coloneqq N_G(v)\cap M$ be the set of marked  neighbors of $v$.
As described in \citet{Pap+2022}, when running an MPNN on a graph with $k$ marked vertices $M$, features are computed in layer $i\geq 0$ as
\begin{equation}
\hb_{v}^\tup{i+1}\coloneqq 
\AGG_{\mathsf{marked}}^\tup{i+1}\bigl(\oms \hb_u^\tup{i}\mid u\in N_G^M(v)\cms\bigr)+
\AGG_{\mathsf{unmarked}}^\tup{i+1}\bigl(
\oms \hb_u^\tup{i}\mid u\in N_G(v)\setminus N_G^{M}(v)\cms
\bigr).\label{eq:markedgnn}
\end{equation}
In other words, during neighbor aggregation, MPNNs can distinguish between marked and unmarked neighbors. Furthermore, for \mgnns{k} one  first obtains vertex features for all markings $M$, which are subsequently aggregated into a single vertex feature. Finally, these vertex features are aggregated to obtain a single graph feature.

Inspecting~\cref{eq:markedgnn}, we see that we can replace the two aggregation functions by one aggregation function provided that $\hb_u^\tup{i}$ contains information indicating whether or not $u$ is marked. In other words, we can replace~\cref{eq:markedgnn} by
$$
\hb_v^\tup{i+1}\coloneqq \AGG^\tup{i+1}\bigl(
\oms (\hb_u^\tup{i},\ones_{u\in M}) \mid u\in N_G(v)\cms
\bigr),
$$
for a given set $M$ of markings. We use this observation for casting \mgnns{k} as \PMPNNs{k}.
Indeed, each marking $M$ corresponds to an ordered $k$-vertex subgraph $\pb\in G_k$. Furthermore, we set the update function  in $\hb_{v,\pb}^\tup{0}\coloneqq\UPD(\mathsf{atp}(v,\st(\pb)))$ such that it returns the label $l(v)$ of $v$ and the indicator function $\ones_{v\in \pb}$. 
We ensure that all update functions propagate the indicator function to the next layers such that aggregation functions  have this information at their disposal in every layer. As mentioned, this suffices to perform the aggregation carried out by \mgnns{k}.
Moreover, all possible markings are considered in \mgnns{k}. Hence, the \PMPNN{k} will select all possible $k$-vertex graphs as well. We will capture this by setting $\pmb\pi_{v,\pb}=1$ below.



More precisely, the following \PMPNN{k} correspond to \mgnns{k}:
\allowdisplaybreaks
\begin{linenomath}\postdisplaypenalty=0
\begin{align*}
    \hb_{v,\pb}^\tup{0}&=(l(v),\ones_{v\in \pb})\\
	\pmb\pi_{v,\pb}&= 
1 \\
	\hb_{v,\pb}^\tup{i+1}&=\UPD^\tup{i+1}\Bigl(\hb_{v,\pb}^\tup{i},\AGG^\tup{i+1}\bigl(\oms\hb_{u,\pb}^\tup{i}\mid u\in N_G(v)\cms\bigr)\Bigr)\\
	\hb_v^\tup{T}&=\pAGG\bigl(\oms\hb_{v,\pb}^\tup{T}\mid \pb\in G_k \text{ s.t. } \pmb\pi_{v,\pb}\neq 0\cms\bigr)=\pAGG\bigl(\oms\hb_{v,\pb}^\tup{T}\mid \pb\in G_k\cms\bigr)\\	
    \hb_G&=\AGG\bigl(\oms\hb_{v}^\tup{T}\mid v\in V(G)\cms\bigr),
 \end{align*}\end{linenomath}
where the aggregation functions $\AGG^\tup{i+1}$ are those from the marked GNN under consideration, and the update functions are such that they propagate the indicator function $\ones_{v\in\pb}$ to the next iteration, as explained before.
Finally, $\hb_G$ is obtained by aggregating over vertex embeddings, which in turn are defined in terms of aggregating over vertex embedding $\hb_{v,\pb}^\tup{T}$ for $\pb\in G_k$. This is in accordance with how marked GNNs operate.

We also note that \mgnns{k} in~\citet{Pap+2022} are guaranteed to be stronger than \MPNNs because they run a classical \MPNN alongside. This is not shown in the \PMPNN{k} description given above as any architecture can be made at least as strong as \MPNNs in this way.


We next consider \textit{$k$-reconstruction GNNs} (\recon{k}) \citep{Cot+2021} which for each set $S$ of $k$ vertices, compute a vectorial representation (using an MPNN) of $G[S]$, then concatenate all the obtained representations (for all $S$), followed by  the application of a permutation invariant update function to obtain a graph representation. The difference with marked GNN is thus that the order of aggregation is different. And indeed,
\recon{k} are captured by vertex-subgraph \PMPNNs{k}, as we will see shortly.

Clearly, the $S$ of $k$ vertices and, more specifically, $G[S]$ corresponds to a vertex-ordered subgraph $\pb\in G_k$. Then, in order to compute a representation of $G[S]$ using $\pb$, we proceed as follows: We run an MPNN on $G[S]$ by ensuring that the update and aggregation functions in  the vertex-subgraph \PMPNNs{k} know which vertices belong to $\pb$ (i.e., $G[S]$). This is done in the same way as for \mgnns{k} by including this information in the initial features. In contrast to \mgnns{k}, we perform vertex aggregation for each $\pb$ to obtain a representation of $\pb$ ($G[S]$). Then, we aggregate over all $\pb$ (i.e., all $S$ and thus $G[S]$) using concatenation as an aggregation function, and finally apply an update function, following how \recon{k} operate. We have thus shown the following.
\begin{proposition}[\Cref{mark} in the main text]
For $k \geq 1$, \PMPNNs{k} capture \mgnns{k} and vertex-subgraph \PMPNNs{k} capture \recon{k}. 
\end{proposition}
Our results on expressive power of \PMPNNs{k} now imply that these architectures are bounded by (and are strictly weaker than) $\WLk{(k+1)}$.

\paragraph{Identity-aware GNNs}
We next consider identity-aware GNNs (\idgnn{1})~\citep{You+2021}, an extension of MPNNs in which message functions can differentiate whether the vertices visited are equal or different from a given center vertex and vertex exploration only happens inside the $h$-hop egonet of the center vertices. 

More specifically, let us
denote by $N_G^h(v)$ the set of $h$-hop neighbors of the ``center'' vertex $v$. Then, for each $v\in V(G)$, \idgnn{1} compute vertex features of $u\in N_G^h(v)$ in layer $i>0$ as
\[
\hb_{u,v}^\tup{i+1}\coloneqq \UPD^\tup{i+1}\bigl(\hb_{u,v}^\tup{i},\AGG^\tup{i+1}(\oms(\hb_{w,v},\ones_{w=v})\mid w\in N_G(u)\cap N_G^{h}(v)\cms)
\bigr)
\]
and then after layer $h$, one lets $\hb_v\coloneqq \hb_{v,v}^\tup{h}$ and $\hb_G\coloneqq\AGG(\oms \hb_v\mid v\in V(G)\cms)$. In other words, vertex features are computed by means of a local message passing neural network, centered around each vertex, in which the aggregation functions can distinguish between the center and other vertices.

We next show how to model \idgnn{1} as \PMPNNs{1}.
We first observe that   \PMPNNs{1} can extract information related to $h$-hop distance neighbors. More precisely, let $g\in G_1=V(G)$ be a single-vertex subgraph. We first compute the function $\mu_{u,g}^\tup{i}$ for $0 \leq i \leq h$, indicating if the vertex $u \in N_G^i(g)$. We can compute $\mu_{u,g}^\tup{i}$ using $i$ \PMPNNs{1} layers as follows:
\allowdisplaybreaks
\begin{linenomath}\postdisplaypenalty=0
\begin{align*}
 \mu_{u,g}^\tup{0}&=\UPD^\tup{0}\Bigl(\mathsf{atp}(u,g)\Bigr)=\ones_{u=g}\\
\mu_{u,g}^\tup{i+1}&=\UPD^\tup{i+1}\Bigl(\mu_{u,g}^\tup{i},
\AGG^\tup{i+1}\bigl(\oms\mu_{w,g}^\tup{i}\mid w \in N_G(u)\cms\bigr)\Bigr),
 \end{align*}
 \end{linenomath}
where the update and aggregation functions are such that  $\mu_{u,g}^\tup{i+1}=1$ if and only if $\mu_{u,g}^\tup{i}=1$ or there exists a $w \in N_G(u)$ with $\mu_{w,g}^\tup{i}=1$. We will use these layers for computing the indicator function $\ones_{u\in N_G^h(g)}$ in other architectures below.

We can now model \idgnn{1} as \PMPNNs{1}, as follows. We let the center vertices correspond to $1$-vertex subgraphs $g\in V(G)$, and ensure that the initial features $\hb_{v,g}^\tup{0}$ carry around $\ones_{v=g}$ and $\mu_{v,g}^\tup{h}$ (i.e., $\ones_{v\in N_G^h(g)}$).
As before, we assume that all update functions propagate this information to the next layer such that aggregation functions can take into account whether or not a vertex is equal $g$ or lies in $N_G^h(g)$.

In contrast to \mgnns{1} and \recon{1}, where $\pmb\pi_{v,g}$ did not restrict the subgraphs, \idgnn{1} obtain vertex features for $v$ only using the subgraph $g=v$ (recall $\hb_v\coloneqq\hb_{v,v}^\tup{h})$). Hence, we will impose that $
\pmb\pi_{v,g}=1$ iff $v=g$. More  specifically, \idgnn{1} correspond to \PMPNNs{1} of the form:
\allowdisplaybreaks
\begin{linenomath}
\postdisplaypenalty=0
\begin{align*}
 \hb_{v,g}^\tup{0}&=(l(v),\ones_{v=g},\mu_{v,g}^\tup{h})\\
 \pmb\pi_{v,g}&= \UPD(\mathsf{atp}(v,g))=\ones_{v=g}
\\
 \hb_{v,g}^\tup{i+1}&=\UPD^\tup{i+1}\Bigl(\hb_{v,g}^\tup{i},\AGG^\tup{i+1}\bigl(\oms \hb_{u,g}^\tup{i} \mid u\in N_G(v)\cms\bigr)\Bigr)\\
\hb_v&=\pAGG\bigl(\oms\hb_{v,g}^\tup{h}\mid g\in G_1 \text{ s.t. }\pmb\pi_{v,g}\neq 0\cms)=\pAGG\bigl(\oms\hb_{v,v}^\tup{h}\cms\bigr)\\
 \hb_G&=\AGG\bigl(\hb_{v}\mid v\in V(G)\bigr),
 \end{align*}
 \end{linenomath}
where $\AGG^\tup{i+1}$ only takes into account those $u\in N_G(v)\cap N_G^h(g)$ (using $\mu_{u,g}^\tup{h}=1$) and also uses the availability of $\ones_{u=g}$ to simulate the aggregation function used in  \idgnn{1}. The definitions of $\hb_v$ and $\hb_G$ are as in the description of \idgnn{1} given earlier.

\citet{You+2021} showed that \idgnn{1} can distinguish more graphs than $\WLk{1}$ based on the counting of cycles. By viewing \idgnn{1} as \PMPNNs{1} our results provide an upper bound by $\WLk{2}$ for  \idgnn{1}. This is consistent with their ability to count cycles, as this can be done in $\WLk{2}$.

Although not considered in~\citet{You+2021}, one could extend \idgnn{1} to \idgnnk{k} by allowing checking for identity with vertices on a previously explored path of length $k-1$, as follows:
\allowdisplaybreaks
\begin{linenomath}\postdisplaypenalty=0
\begin{align*}
 \hb_{v,\pb}^\tup{0}&=(l(v),\ones_{v\in\pb},\mu_{v,g_1}^\tup{h})\\
\pmb\pi_{v,\pb}&= \UPD(\mathsf{atp}(v,\pb))=\begin{cases}
1 & \text{if $v=g_1,g_2,\ldots,g_k$ from a path in $G$}\\
0 & \text{otherwise}
\end{cases} \\
\hb_{v,\pb}^\tup{i+1}&=\UPD^\tup{i+1}\Bigl(\hb_{v,\pb}^\tup{i},\AGG^\tup{i+1}\bigl(\oms \hb_{u,\pb}^\tup{i})\mid u\in N_G(v)\cms\bigr)\Bigr)\\
 \hb_v&=\pAGG\bigl(\oms \hb_{v,\pb}^\tup{h}\mid \pb\in G_k \text{ s.t. }  \pmb\pi_{v,\pb}\neq 0\cms \bigr)\\
 &=\pAGG\bigl(\oms\hb_{v,\pb}\mid \text{ $\pb$ is a walk in $G$ of length $k$ starting from $v$}\cms\bigr)\\
 \hb_G&=\AGG\bigl(\oms\hb_{v}\mid v\in V(G)\cms\bigr)
 \end{align*}\end{linenomath}
 Our results show that such \idgnnk{k} are bounded by $\WLk{(k+1)}$ in expressive power. It is also readily verified that \idgnnk{k} for $k>1$ can detect more complex substructures than cycles.


\comm{
\floris{There is more to say here: there are graphs known to be distinguishable by identity-aware \GNNs but not by \MPNNs. Can we give such graphs for the extension of identity-aware \GNNs given above? Is there a characterization in terms of hom counts of special trees with cycles? Perhaps also interesting is the connection to (anonymous) walk-based \GNN architectures?} 
\cm{It should be possible to also use the the graph $G_k$ and $H_k$ to show a hierarchy, see below, as using a path of $2(k+1)$ pebbles should be enough such that $\WLk{1}$ detects the distance-two cliques in $G_k$.   }
\cm{Is there a reason we restrict to walks?}
\floris{My intuition was that since an ID-aware GNN can check equality with the root of unravelling, a $k$ id-aware GNN would be able to test equality with paths originating from the root, i.e., walks. }
}

\paragraph{Nested GNNs}
We next consider Nested GNNs (\nested)~\citep{Zha+2021} that obtain vertex embeddings based on the aggregation over vertex embeddings in the $h$-hop egonets. In their notation, $G_{w}^h$ denotes the subgraph in $G$, rooted at $w$ of ``height'' $h$. Then, for any
vertex $v\in V(G_w^h)$ they compute:
\[
\hb_{v,G_{w}^h}^\tup{i+1}=\UPD^\tup{i+1}_1\Bigl(\hb_{v,G_w^h}^\tup{i},\sum_{u\in N(v \mid G_w^h)} \UPD_2^\tup{i+1}\bigl(\hb_{u,G_w^h}^\tup{i}\bigr)\Bigr)
\]
where $N(v \mid G_w^h)$ is the set of neighbors of $v$ within $G_w^h$.
Pooling happens after layer $T$:
\[
\hb_w=\AGG(\oms\hb_{v,G_w^h}^\tup{T}\mid v\in V(G_w^h)\cms)
\]
and then $\hb_G=\AGG(\oms \hb_v\mid v\in V(G)\cms)$. 

We can formulate \nested as \PMPNNs{1} as follows. Similarly as for \idgnn{1}, the center vertices correspond to  $1$-vertex subgraphs $g$ and we again include the information $\mu_{v,g}^\tup{h}$ as initial feature in order to aggregate over vertices in $N_G^h(g)$, i.e., those in $G_g^h$. We assume that the update function propagate this information to higher layers, as before. As aggregation functions $\AGG^\tup{i+1}$, we use summation but only over those features for which the $\mu_{u,g}^\tup{h}$ component is $1$ and in this way simulate the aggregation step used in \nested. More specifically, we have:
\allowdisplaybreaks
\begin{linenomath}\postdisplaypenalty=0
\begin{align}
    \hb_{v,g}^\tup{0}&=(l(v),\mu_{v,g}^\tup{h})\label{eq:nestedini}\\
  \pmb\pi_{v,g}&= \UPD(\mathsf{atp}(v,g))=\ones_{v=g}\\ 
	\hb_{v,g}^\tup{i+1}&=\UPD^\tup{i}\Bigl(\hb_{v,g}^\tup{i},\AGG^\tup{i+1}\bigl(
\oms \hb_{u,g}^\tup{i}\mid u \in N_G(v)\cms
\bigr)\Bigr),\label{eq:mespasnested}
    \end{align}
 \end{linenomath}
 which similar as to how \idgnn{1}, viewed as \PMPNNs{1}, operate. The main difference with \idgnn{1} is
 how vertex features are computed. Indeed, \nested assign to vertex $v$ the representation of $G_v^h$. We use  $\AGG^\tup{T+1}$ to aggregate over all vertices $u$ in $G_g^h$ by leveraging $\mu_{u,g}^\tup{h}$. More precisely,
 instead of aggregating over neighbors we aggregate over the entire vertex set but ensure that $\AGG^\tup{T+1}$ only takes into account those vertices in $N_G^h(g)$ using $\mu_{u,g}^\tup{h}$: \allowdisplaybreaks
\begin{linenomath}\postdisplaypenalty=0
\begin{align}
\hb_{v,g}&=\AGG^\tup{T+1}\bigl(\oms\hb_{u,g}^\tup{T}\mid w\in V(G)\cms\bigr)\notag\\
\hb_v&=\pAGG\bigl(\oms\hb_{v,g}\mid g \in G_1\text{ s.t. } \pmb\pi_{v,g}\neq 0\cms\bigr)=\pAGG\bigl(\oms\hb_{v,v}\cms\bigr)\label{eq:nestedv}\\
\hb_G&=\AGG\bigl(\oms\hb_{v}\mid v\in V(G)\cms\bigr)\label{eq:nestedg},
     \end{align}
 \end{linenomath}
where the final steps are in place to create a graph representation in accordance with how \nested operate.

\citet{Zha+2021} observe that \nested are more powerful than $\WLk{1}$ and raise the question whether \nested can be more powerful than $\WLk{2}$.
By viewing \nested as \PMPNNs{1}, our general results about expressive power, show that \nested  are bounded by $\WLk{2}$ in expressive power.
Moreover, \citet{Zha+2021} allude to deeper nested GNNs in their paper. It seems natural to conjecture that these can be cast as \PMPNNs{k} when $k$ levels of nesting are used. We leave the verification of this conjecture for future work.


\paragraph{GNN As Kernel} 

Very related to \nested are GNN as kernel (\kernel)~\citep{Zha+2021b}. Indeed, the only difference is that once the $\hb_{v}$ are defined for $v\in V(G)$ in \cref{eq:nestedv}, \kernel restart the message passing over egonets (\cref{eq:nestedini}-\cref{eq:mespasnested}), but this time with the initial features $\hb_{v,g}^\tup{0}$ replaced by
$(\hb_{v},\mu_{v,g}^\tup{h})$. This is then repeated a number of times, after which a graph representation is obtained, just as for \nested (\cref{eq:nestedg}).
It is now readily verified that we can express \kernel as \PMPNNs{1} in the same as we showed for \nested.



\begin{proposition}[\Cref{identity} in the main text]
The \PMPNNs{1} capture \idgnn{1}, \kernel, and \nested.
\end{proposition}
\begin{proof}
The result follows from the above equations. 
\end{proof}

We note again that our results show that \idgnn{1}, \kernel and \nested are all bounded in expressive power by $\WLk{2}$, yet are less expressive than $\WLk{2}$.

\paragraph{DS-GNN with the $k$-vertex-deleted policy}
In the following, we define an instance of a vertex-subgraph \PMPNN{k} which captures \dsgnn with the $k$-vertex-deleted policy \citep{Bev+2021}. In a nutshell, \dsgnn generate MPNN-based representations of a collection of subgraphs and then aggregate those to obtain a representation of the original graph. In general, a policy is in place in \dsgnn to select the subgraphs. 
Here, we consider the $k$-vertex-deleted policy 
in which all $k$-vertex deleted subgraphs $S$ are considered. The deletion of $k$-vertices used to obtain a subgraph $S$ will be simulated by considering $k$-vertex subgraphs $\pb$ and by treating the vertices in $\pb$ to be marked.
In other words, \dsgnn act like a \mgnns{k} except that graph representations are obtained by aggregating subgraph representations. More specifically:
\allowdisplaybreaks
\begin{linenomath}\postdisplaypenalty=0
\begin{align*}
    \hb_{v,\pb}^\tup{0}&=(l(v),\ones_{v\in \pb})\\
	\pmb\pi_{v,\pb}&= 
1 \\
	\hb_{v,\pb}^\tup{i+1}&=\UPD^\tup{i+1}\Bigl(\hb_{v,\pb}^\tup{i},\AGG^\tup{i+1}\bigl(\oms\hb_{u,\pb}^\tup{i}\mid u\in N_G(v)\cms\bigr)\Bigr)\\
	\hb_\pb&=\AGG\bigl(\oms\hb_{v,\pb}^\tup{T}\mid v\in V(G)\cms\bigr)\\	
    \hb_G&=\pAGG\bigl(\oms\hb_{\pb}\mid \pb\in G_k, \exists v\in V(G)\,  \pmb\pi_{v,\pb}\neq\mathbf{0}\cms\bigr)=\pAGG\bigl(\oms\hb_{\pb}\mid \pb\in G_k\cms\bigr),
 \end{align*}\end{linenomath}
where update functions propagate $\ones_{v\in\pb}$ and aggregation functions treat vertices in $\pb$ as marked (or to be deleted). 
 

\begin{proposition}[\Cref{esan} in the main text]
Vertex-subgraph \PMPNNs{k} capture \dsgnn with the $k$-vertex-deleted policy.
\end{proposition}
\begin{proof}[Proof sketch]
We argue that the above \PMPNN{k} instance can simulate \dswl~\citep{Bev+2021} which upper bounds any possible \dswl in terms of distinguishing non-isomorphic graphs.

As noted by~\cite{Pap+2022}, see also above paragraph on \mgnns{k}, marking vertices is at least as powerful as removing them. The markings enable the aggregation function to distinguish between deleted and non-deleted vertices. By choosing injective instances of $\UPD$ and $\AGG$, we can simulate the coloring function $c_{v,\pb}^\tup{i}$, for $i \geq 0$, of the \dswl. That is, if $\hb_{v,\pb}^\tup{i} = \hb_{w,\pb}^\tup{i}$ holds, it follows that $c_{v,\pb}^\tup{i}= c_{w,\pb}^\tup{i}$ for all vertices $v$ and $w$ of a given graph $G$ and $\pb \in G_k$ holds. The existence of such instances follows directly from the proof of Theorem 2 in \citep{Mor+2019}. Similarly, by choosing injective instances of $\pAGG$ and $\AGG$ for computing the single graph feature, the resulting architecture has at least the same expressive power as the \dswl in distinguishing non-isomorphic graphs. The reasoning is analogous to the proof~\cref{lem:wlinmp_new}. Hence, the resulting architecture has at least the expressive power of \dswl, implying the result.
\end{proof}

\citet{Bev+2021} also consider the $1$-edge-deleted policy in which the subgraphs $S$ considered are those obtained by deleting a single edge.
The deletion of an edge can be simulated by marking two vertices, which can be simulated using message and update functions having access to $2$-vertex subgraphs $\pb\in G_2$. Hence, \dsgnn with the $1$-edge-deleted policy can be captured by vertex-subgraph \PMPNNs{2}. As a consequence, such \dsgnn are bounded in expressive power by $\WLk{3}$. Combined with the discussion in~\cref{unordered} it should be clear that \dsgnn with $k$-edge-deleted policy can be captured by vertex-subgraph \PMPNNs{2k}. As argued in~\cref{unordered} the use of ordered graphs is crucial to simulate multiple edge deletions. Finally, \citet{Bev+2021} also consider two variants of  $k$-hop ego-net policies. In the first, the subgraphs $S$ consist of all $k$-hop ego-net subgraphs, one for each vertex in the graph. In the second variant, equality with the center vertex in each ego-net can be checked. It should be clear from our treatment of \nested and \kernel that the ego-net extraction can be simulated in vertex-subgraph \PMPNNs{1} and that the distinction between two variants pours down to include $\ones_{v=g}$ as an initial feature (just as for \idgnn{k}).
Hence  \dsgnn with ego-net policies are bounded by $\WLk{2}$.

\end{document}